\documentclass{article}

\usepackage{microtype}
\usepackage{graphicx}
\usepackage{subcaption}
\usepackage{booktabs} 

\usepackage{hyperref}

\usepackage[accepted]{icml2026}

\usepackage{amsmath}
\usepackage{amssymb}
\usepackage{mathtools}
\usepackage{amsthm}
\usepackage{bm}
\usepackage{multirow}
\usepackage[table]{xcolor}
\usepackage{array}
\usepackage{enumitem}
\usepackage{colortbl}
\usepackage{float}

\definecolor{oursblue}{RGB}{225,235,245}
\definecolor{rowgray}{gray}{0.95}

\usepackage[capitalize,noabbrev]{cleveref}

\theoremstyle{plain}
\newtheorem{theorem}{Theorem}[section]
\newtheorem{proposition}[theorem]{Proposition}
\newtheorem{lemma}[theorem]{Lemma}

\theoremstyle{definition}

\newtheorem{assumption}[theorem]{Assumption}
\theoremstyle{remark}
\newtheorem{remark}[theorem]{Remark}

\usepackage[textsize=tiny]{todonotes}

\newcommand{\best}[1]{\textcolor{red}{#1}}
\newcommand{\secondbest}[1]{\textcolor{blue}{#1}}

\definecolor{morandiL}{RGB}{235, 240, 245}
\definecolor{morandiM}{RGB}{205, 218, 228}
\definecolor{morandiD}{RGB}{175, 195, 210}
\icmltitlerunning{Is Task-Specific Training Necessary for Anomaly Detection?}

\begin{document}

\twocolumn[
  \icmltitle{Is Task-Specific Training Necessary for Anomaly Detection?}



  \icmlsetsymbol{equal}{*}
  \icmlsetsymbol{corr}{\dag}

  \begin{icmlauthorlist}
    \icmlauthor{Xingwu Zhang}{yyy}
    \icmlauthor{Guanxuan Li}{yyy}
    \icmlauthor{Paul Henderson}{sch}
    \icmlauthor{Gerardo Aragon-Camarasa}{sch}
    \icmlauthor{Zijun Long}{yyy,corr}
  \end{icmlauthorlist}

  \icmlaffiliation{yyy}{College of Electrical and Information Engineering, Hunan University, Changsha, China}
  \icmlaffiliation{sch}{School of Computing Science, University of Glasgow, Glasgow, United Kingdom}

  \icmlcorrespondingauthor{Zijun Long}{longzijun@hnu.edu.cn}

  \icmlkeywords{Machine Learning, ICML}

  \vskip 0.3in
]



\printAffiliationsAndNotice{}  

\begin{abstract}

Current state-of-the-art multi-class unsupervised anomaly detection (MUAD) methods rely on training encoder--decoder models to reconstruct anomaly-free features. However, we argue that such task-specific training is costly under distribution shifts, and that reconstruction-based residual scoring further faces a fidelity--stability dilemma. Existing training-free alternatives, in turn, remain prone to cross-category and cross-region mismatches in MUAD. Motivated by these limitations, we propose Retrieval-based Anomaly Detection (RAD), a task-specific training-free framework that stores anomaly-free features in a memory and detects anomalies through multi-level retrieval, matching test patches against the memory. Experiments demonstrate that RAD achieves state-of-the-art performance across four established benchmarks (MVTec-AD, VisA, Real-IAD, 3D-ADAM) under both standard and few-shot settings. On MVTec-AD, RAD reaches 96.7\% Pixel AUROC with just a single anomaly-free image compared to 98.5\% of RAD's full-data performance. Collectively, these findings overturn the assumption that MUAD requires task-specific training, showing that state-of-the-art anomaly detection is feasible with training-free memory-based retrieval. Our code is available at \url{https://github.com/longkukuhi/RAD}.

\end{abstract}

\setlength{\parskip}{0.4em}
\section{Introduction}
Multi-class Unsupervised Anomaly Detection (MUAD) aims to detect abnormal patterns and localize anomalous regions using only anomaly-free samples, without access to class labels. Due to the diversity and scarcity of potential anomalies, the task is typically formulated as measuring how much a test sample deviates from the anomaly-free data pattern.
Existing state-of-the-art methods rely on task-specific training to capture anomaly-free patterns, especially when the anomaly-free distribution changes across categories. Prior work spans several training paradigms, including student--teacher distillation~\cite{wei2025uninet}, diffusion-based generative modeling~\cite{he2024diffusion,fengomiad}, encoder--decoder reconstruction~\cite{guo2025dinomaly,he2024mambaad,zhang2023exploring,you2022unified}, and memory-based techniques~\cite{lee2024continuous,yao2024resad,tan2024unsupervised,liang2025look,liang2025lightweight} (see Appendix Sec.~\ref{sec:rw} for details). The empirical success of these approaches has reinforced the belief that competitive task-specific training-free MUAD is unattainable.


\begin{figure}
    \centering
    \includegraphics[width=1\linewidth]{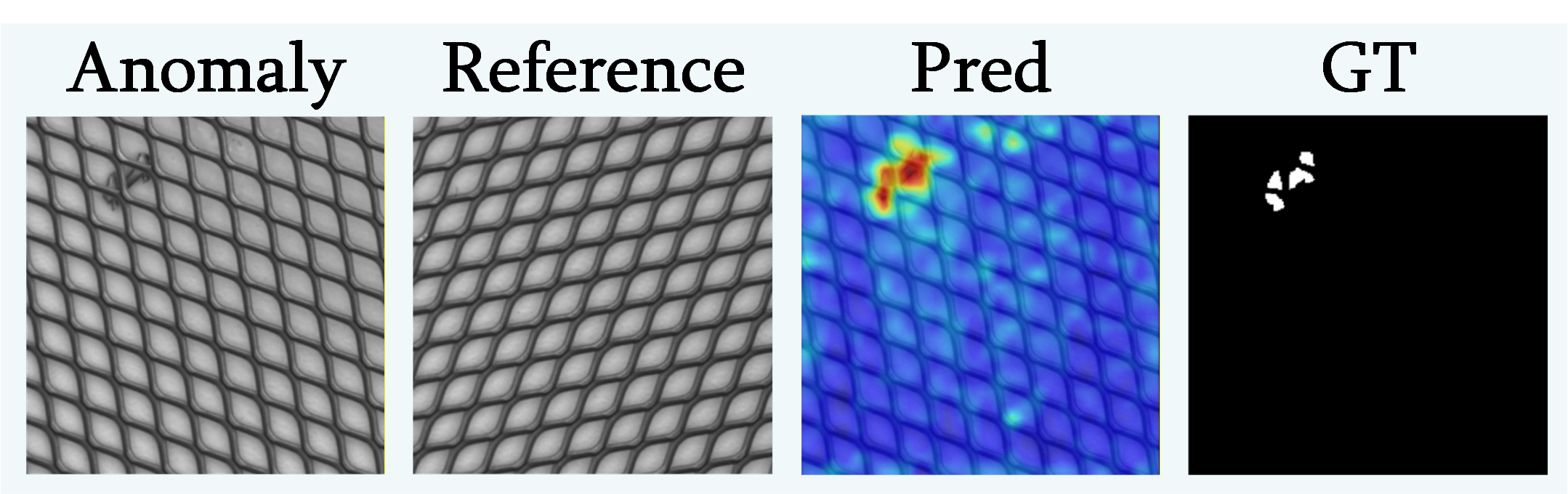}
    \caption{Visualizations of MUAD with only one reference by RAD.}
    \label{fig:Figure_RAD}
\end{figure}
Among the above methods, encoder--decoder reconstruction methods have been particularly successful. A pre-trained encoder is typically frozen, a lightweight decoder is trained on anomaly-free data to reconstruct encoder features, and anomalies are detected by thresholding the discrepancies between original and reconstructed features.
In this work, we analyze the behavior of reconstruction-based MUAD in representation space, and show for the first time that these methods suffer from an inherent \textit{fidelity--stability dilemma}.
They must simultaneously achieve high-fidelity reconstruction of anomaly-free features (towards identity mapping) and stable residual scores under benign variations (a clear boundary between normal and abnormal). We show that these two training objectives are naturally contradictory and eventually lead to the \emph{dilemma} as discussed in Sec. \ref{sec:theory}. When high-fidelity reconstruction is enforced, the decoder must amplify small benign variations in encoder features, making the decision boundary highly sensitive. Conversely, relaxing reconstruction pressure improves stability but introduces a systematic approximation gap on the anomaly-free distribution, diluting residual-based anomaly evidence. 

The cost of task-specific training under distribution shifts and the fidelity--stability dilemma of reconstruction-based scoring together call for task-specific training-free MUAD. However, existing task-specific training-free methods, especially deep k-Nearest-Neighbor-based methods~\cite{roth2022towards,bergman2020deep,cohen2020sub}, typically score anomalies by directly matching test features against stored anomaly-free features. In MUAD, such naive retrieval can suffer from cross-category and cross-region mismatches.

To realize robust task-specific training-free MUAD, we propose RAD (Retrieval-based Anomaly Detection), a framework that replaces decoder training with context-conditioned multi-level retrieval over anomaly-free feature memories to mitigate MUAD-specific mismatches. For each test image, RAD first retrieves a small set of globally compatible anomaly-free references, restricting comparisons to features that match its global context and spatial layout. It then conducts spatially constrained patch matching within these references across multiple encoder layers, and fuses the resulting layer-wise scores to combine low-level local sensitivity with high-level semantic robustness. RAD can adapt to new data simply by updating the memory bank rather than re-training, naturally accommodating new categories and distribution shifts; it also directly benefits from stronger frozen encoders as they become available, as confirmed by our experiments in Sec. \ref{sec:exp&analysis}.

We further provide a theoretical analysis showing that, when sharing the same frozen encoder and anomaly-free training set, retrieval-based anomaly scores are maximal within a restricted class of stable feature-space scores, while avoiding the \emph{fidelity--stability dilemma} induced by information loss in decoder-based reconstruction. Empirically, RAD is, to our knowledge, the first task-specific training-free MUAD framework that consistently surpasses task-specific training-based methods, achieving state-of-the-art performance on four benchmarks~\cite{bergmann2019mvtec,zou2022spot,wang2024real,mchard20253d} under both standard multi-class and few-shot MUAD settings. We visualize how RAD accurately localizes anomalies on MVTec-AD (see Fig.~\ref{fig:Figure_RAD}).
It remains effective even with a single anomaly-free reference whose appearance differs from the test sample, demonstrating that competitive anomaly detection is possible without task-specific training. 

Specifically, our main contributions are: (i) Analyzing the fidelity--stability dilemma in encoder--decoder MUAD.
(ii) Proposing RAD, the first task-specific training-free MUAD framework that mitigates cross-category and cross-region mismatches via multi-level retrieval.
(iii) Demonstrating that RAD achieves state-of-the-art results across four benchmarks under both standard and few-shot MUAD settings compared with both trained and training-free methods.



\section{Fidelity-Stability Dilemma}
\label{sec:theory}

In this section, we analyze encoder--decoder methods from the view of representation space and show they inherently face a fidelity--stability dilemma. We first show that, even on anomaly-free patches, the optimal reconstruction map generally deviates from the identity, so residuals remain non-zero due to unavoidable reconstruction errors. We then study how enforcing high-fidelity reconstruction of benign feature variations interacts with the bottleneck, and derive a lower bound that ties this requirement to the decoder’s local gain, revealing a fundamental fidelity--stability trade-off in residual scores.

\paragraph{Problem formulation.}
Let $\mathcal{D}_{\mathrm{norm}}^{\mathrm{img}} = \{x_i\}_{i=1}^N$ be a dataset of $N$ anomaly-free training images covering a variety of classes. Each image is decomposed into patches on a fixed grid, and we denote by $\mathcal{D}_{\mathrm{norm}}^{\mathrm{patch}}$ the resulting set of anomaly-free training patches.
In contrast, test images may contain unseen defects, with some patches anomalous and others normal.
Given a test image $x$, a detector produces an anomaly score $S(x,t)$ for each patch $t$.
For the theoretical analysis below, we fix a test image and a patch and write $S(t)$ for the corresponding patch-wise anomaly score.

\paragraph{Reconstruction-based scoring.}

Current state-of-the-art reconstruction methods~\cite{you2022unified,guo2025dinomaly,zhang2023exploring} use a training paradigm that freezes a pre-trained encoder $\Phi$ and learns a decoder $\Psi_\theta$ that approximately inverts an information-losing bottleneck $B$, aiming to reconstruct the encoder features from their bottlenecked versions.

To prevent the decoder from trivially copying $\Phi(n)$, the bottleneck $B$ stochastically removes or perturbs part of the representation, so that the decoder must infer missing information from the learned anomaly-free patterns rather than simply act as an exact identity mapping.
The reconstruction path can then be written as $R_\theta(n) = \Psi_\theta\big(B(\Phi(n))\big)$.
The decoder is trained on anomaly-free patches by: 
\begin{equation}
\theta^* = \arg\min_\theta
    \sum_{n \in \mathcal{D}_{\mathrm{norm}}^{\mathrm{patch}}}
    \ell_{\mathrm{rec}}\big(\Phi(n),\, R_\theta(n)).
\end{equation}

Here $\theta^*$ represents the (ideal) best decoder parameters achievable under this objective, and $R_{\theta^*}$ is the corresponding optimal reconstruction decoder for anomaly-free patches. At test time, we use $R_{\theta^*}$ to reconstruct each patch $t$, defining the anomaly score as the feature residual
$
    S_{\mathrm{rec}}(t) = \big\|\Phi(t) - R_{\theta^*}(t)\big\|.
$
\label{eqa:residual_score}

To analyze equation \ref{eqa:residual_score}, we study the reconstruction map $R_{\theta^*}$
in feature space: (i) its nominal deviation from the identity on anomaly-free features, and (ii) how benign feature variations are propagated through the bottleneck--decoder pipeline. We begin by decomposing $R_{\theta^*}$
on anomaly-free patches into an identity term plus a reconstruction error.

\begin{figure*}
    \centering
    \includegraphics[width=1\linewidth]{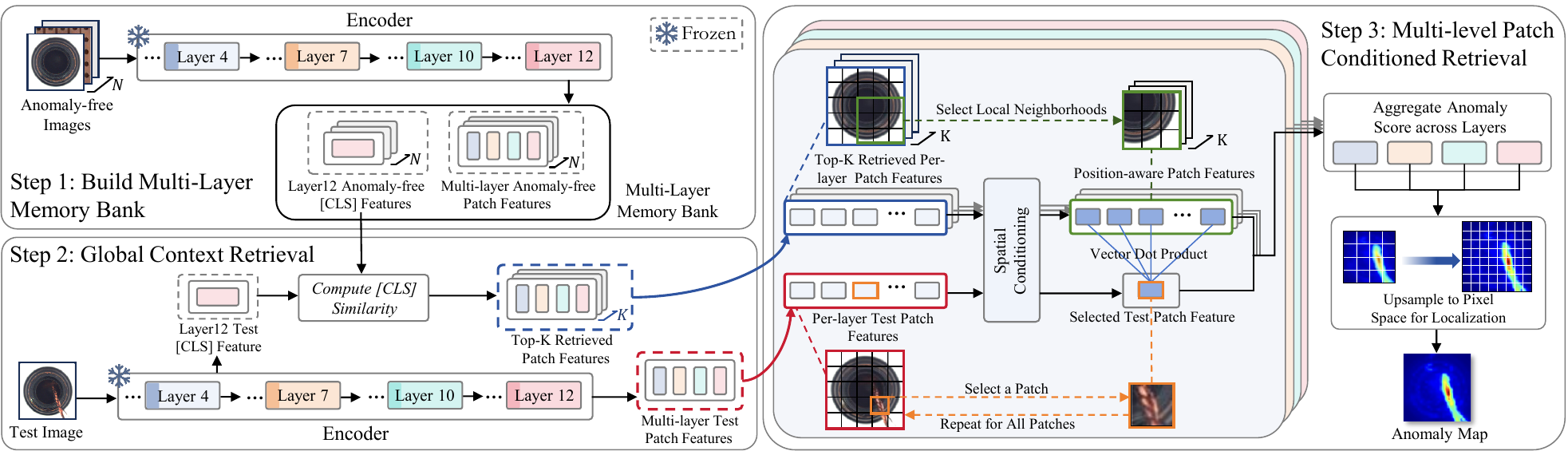}
    \caption{Overview of the proposed RAD framework.}
    \label{fig:Figure_fw}
\end{figure*}

\paragraph{Reconstruction as amortized inverse.}
For anomaly-free patches $n \in \mathcal{D}^{patch}_{\mathrm{norm}}$, the ideal behavior would be $R_{\theta^*}(n) = \Phi(n)$.
In practice, the information loss in $B$, finite capacity, and imperfect optimization jointly induce a systematic deviation from the identity, which we write as
\begin{equation}
    R_{\theta^*}(n) = \Phi(n) + \varepsilon_{\mathrm{approx}}(n),
    \qquad n \in \mathcal{D}_{\mathrm{norm}}^{\mathrm{patch}},
\end{equation}



We refer to $\varepsilon_{\mathrm{approx}}(n)$ as the per-patch \emph{reconstruction error} in feature space.
Its magnitude $\alpha(n) \triangleq \|\varepsilon_{\mathrm{approx}}(n)\|$ is the \emph{reconstruction gap} on patch $n$.
Notably, for anomaly-free patches, the residual score satisfies
\[
S_{\mathrm{rec}}(n)
= \|\Phi(n)-R_{\theta^*}(n)\|
= \|\varepsilon_{\mathrm{approx}}(n)\|
= \alpha(n).
\]

Thus the reconstruction path acts as an amortized, biased inverse of the
bottlenecked encoder on the anomaly-free feature manifold rather than an exact
identity.
Moreover, both the \emph{magnitude} of $\varepsilon_{\mathrm{approx}}$ (captured by $\alpha(n)$) and its \emph{geometry} across $n$
(e.g., which feature directions are systematically distorted)
govern how reconstruction fidelity trades off against the stability of the residual score $S_{\mathrm{rec}}(t)$.


\paragraph{Decoder amplification under information loss.}
Let $R = \Psi \circ B$ denote the feature reconstruction map and $\mathcal{V}_n \subset \mathbb{R}^D$ be a subspace of benign feature
variations around $\Phi(n)$ that should not be treated as anomalies, e.g.,
feature differences between $\Phi(n)$ and $\Phi(n)+\delta$ for nearby
anomaly-free patches.
We write $J_R(u)$ for the Jacobian of the map $R$ evaluated at $u$ (and similarly
$J_\Psi(u)$ and $J_B(u)$).
By the chain rule,
\begin{equation}
    J_R(\Phi(n)) = J_\Psi\big(B(\Phi(n))\big)\, J_B\big(\Phi(n)\big).
\label{eq:chainrule}
\end{equation}

This decomposition allows us to convert the fidelity requirement on $R$
(Assumption~\ref{ass:fidelity}) and the contraction induced by $B$
into a lower bound on the decoder gain $J_\Psi$.

\begin{assumption}[anomaly-free feature fidelity]
\label{ass:fidelity}
Fix a fidelity parameter $\eta \in (0,1)$, which upper-bounds the relative reconstruction error on benign feature variations, with the high-fidelity regime corresponding to $\eta \ll 1$.
For each anomaly-free training patch $n \in \mathcal{D}_{\mathrm{norm}}^{\mathrm{patch}}$, there exists a radius $r>0$ such that for all benign feature perturbations $\delta \in \mathcal{V}_n$ with $\|\delta\|\le r$,
\begin{equation}
    \big\| R(\Phi(n)+\delta) - R(\Phi(n)) - \delta \big\| \le \eta \|\delta\|.
\end{equation}
Equivalently, the Jacobian of $R_{\theta^*}$ restricted to $\mathcal{V}_n$ satisfies
\begin{equation}
    \sigma_{\min}\big(J_{R}(\Phi(n))\big|_{\mathcal{V}_n}\big) \ge 1 - \eta,
\end{equation}
where $\sigma_{\min}(\cdot)$ is the smallest singular value.
\end{assumption}

Assumption~\ref{ass:fidelity} enforces a local near-identity regime: the reconstruction map preserves benign feature variations in $\mathcal{V}_n$ up to relative error $\eta$, which is the formal notion of high-fidelity reconstruction used in our amplification analysis.

\begin{lemma}[Decoder amplification]
\label{lem:amplification}
Under Assumption~\ref{ass:fidelity}, the decoder Jacobian satisfies
\begin{equation}
    \sigma_{\max}\big(J_\Psi(B(\Phi(n)))\big)
    \;\ge\;
    \frac{1-\eta}{
    \sigma_{\min}\big(J_B(\Phi(n))\big|_{\mathcal{V}_n}\big)},
\end{equation}
where $\sigma_{\max}(\cdot)$ is the largest singular value.
\end{lemma}

\emph{Proof sketch.}
Restrict all operators to the subspace $\mathcal{V}_n$.
For matrices $A,B$, one has $\sigma_{\min}(AB) \le \sigma_{\max}(A)\,\sigma_{\min}(B)$.
Applying this to $A = J_\Psi\big(B(\Phi(n))\big)$ and 
$B = J_B\big(\Phi(n)\big)\big|_{\mathcal{V}_n}$ and using
Eq.~\ref{eq:chainrule}
yields
\begin{equation}
\resizebox{0.9\linewidth}{!}{$
\sigma_{\min}\big(J_R(\Phi(n))\big|_{\mathcal{V}_n}\big)
\le
\sigma_{\max}\big(J_\Psi(B(\Phi(n)))\big)\,
\sigma_{\min}\big(J_B(\Phi(n))\big|_{\mathcal{V}_n}\big)
$}
\end{equation}

Rearranging and invoking Assumption~\ref{ass:fidelity} finishes the proof (see Appendix Sec.~\ref{sec:proof} for the full proof). Thus, for a strong bottleneck, the restricted smallest singular value
$\sigma_{\min}(J_B(\Phi(n))|_{\mathcal{V}_n})$ can be very small since $B$ removes or compresses information along many benign feature variations.
Lemma~\ref{lem:amplification} then implies that the decoder must have a large $\sigma_{\max}(J_\Psi)$ in order to maintain high-fidelity reconstruction on anomaly-free patches, which makes it behave as a high-gain amplifier.


In feature space, reconstruction acts as a biased, amortized inverse of the bottlenecked encoder on normal features, so residuals on anomaly-free patches are driven by reconstruction error; enforcing variation preservation through an information-losing bottleneck forces the decoder into a high-gain amplifier, making residuals unstable to benign perturbations.
Consequently, either many anomaly-free patches acquire large residuals and are incorrectly flagged as anomalous, or strong regularization makes the decoder overly smooth so that small defects are reconstructed together with benign variations and go undetected. This stability--fidelity dilemma echoes the stability--accuracy trade-off for deep inverse problems~\cite{colbrook2022difficulty}, but here it appears directly in patch-wise reconstruction scores for anomaly detection. In the next section we show that retrieval-based scoring, which replaces learning-based approximation with an explicit memory of anomaly-free features, sidesteps this dilemma.

\section{RAD: Towards Training-free MUAD}
\label{sec:evoad}
In this section, we formalize the retrieval-based view of RAD and derive its concrete multi-level design. We begin by constructing an explicit multi-layer memory of anomaly-free features that serves as the reference bank for all subsequent retrieval operations (Sec.~\ref{sec:memory}). Building on this, a global context retrieval stage selects a small, semantically compatible reference set for each test image (Sec.~\ref{sec:cls-retrieval-subsec}). We then describe multi-level conditional patch retrieval and anomaly scoring, where spatially conditioned patch matching is performed within local neighborhoods and scores are fused across layers into dense anomaly maps (Sec.~\ref{sec:retrieval}).




\subsection{Multi-layer Anomaly-free Feature Memory Bank}
\label{sec:memory}
Features from different encoder layers provide complementary similarity signals~\cite{park2023self,amir2022effectiveness}. Shallow and intermediate layers retain high-frequency, local information such as edges and small geometric deviations, such that similarity is sensitive to subtle surface changes. Deeper layers aggregate information over larger receptive fields and become more invariant to low-level appearance, so similarity instead reflects semantic and structural consistency rather than precise texture.

For multi-class anomaly detection (MUAD), this complementarity is crucial. Fine-grained defects such as scratches, or small dents are best captured by shallow and intermediate features, while larger structural changes or misalignments are more naturally expressed in deeper semantic features. By retrieving and scoring patches jointly across multiple layers, RAD combines the sensitivity of shallow features with the robustness and semantic context of deeper ones, leading to more reliable localization of both subtle surface anomalies and larger structural defects.

Concretely, we instantiate this idea by constructing a multi-layer memory bank of anomaly-free features.
Let $\Phi$ be a frozen encoder, and let $\mathcal{L}$ be a small set of encoder layers chosen for multi-layer feature extraction. For each anomaly-free image $x_i$, we pass $x_i$ through $\Phi$ and take the \texttt{[CLS]} token of the final hidden state, apply $\ell_2$ normalization, and  store it as a global descriptor $g(x_i)$, representing the entire input as a compact, global summary. These descriptors are later used for global context retrieval in Sec.~\ref{sec:cls-retrieval-subsec}.
In addition, we store all $\ell_2$-normalized patch embeddings $z^{(\ell)}(x_i) \in \mathbb{R}^{d_\ell}$ at layers $\ell \in \mathcal{L}$, which are later used for conditioned retrieval and scoring in Sec.~\ref{sec:retrieval}.
These stored features form a multi-layer memory bank of anomaly-free representations that enables the task-specific training-free advantage of RAD.

\subsection{Global Context Conditioned Retrieval}
\label{sec:cls-retrieval-subsec}
A naive patch-retrieval scheme would compare each test patch against all normal patches in the memory bank, regardless of their semantic class. In the MUAD setting, this leads to two problems: (i) patches from semantically unrelated classes can become nearest neighbors, introducing severe semantic mismatch, and (ii) dense patch matching over the entire memory bank becomes increasingly expensive as the number of stored anomaly-free images grows. To address these issues, RAD first performs a coarse global retrieval step that filters the memory to a small set of semantically compatible reference images.

As shown in Figure~\ref{fig:Figure_fw}, for each anomaly-free global descriptor $g(x_i)$ in the memory bank, given a test image $x$, we compute the cosine similarity between them as $\label{eq:cls-retrieval}
\mathrm{sim}_i(x)
:=
\big\langle g(x),\, g(x_i) \big\rangle.$
We then define $\mathcal{N}_K(x)
:=
\operatorname{TopK}\big(\{\mathrm{sim}_i(x)\}_{i=1}^{N}\big),$
as the indices of the $K$ most similar anomaly-free images to $x$.
All subsequent patch comparisons at any layer $\ell \in \mathcal{L}$ are restricted to images in $\mathcal{N}_K(x)$, enforcing semantic consistency and reducing the cost of dense patch matching.

We use only the last-layer \texttt{[CLS]} token as $g(x_i)$ because it provides a semantically rich global summary of the image, optimized during pre-training to capture high-level category and scene information. In MUAD, this semantic descriptor lets us retrieve a few globally similar references from a heterogeneous memory bank, suppressing features from unrelated classes before patch-level matching. Low- and mid-level cues are already handled by the multi-layer patch memory, so adding additional global tokens from earlier layers would increase complexity without clear benefit.

\subsection{Patch-level Conditioned Retrieval and Scoring}
\label{sec:retrieval}

On the global context retrieved subset $\mathcal{N}_K(x)$, RAD performs
patch-level conditioned retrieval and scoring independently based on multi-layer features.
For clarity, we first describe the procedure for a layer
$\ell \in \mathcal{L}$; the same steps are then applied in parallel to all
layers in $\mathcal{L}$, and the resulting scores are fused across layers.

\paragraph{Spatial conditioning via local neighborhoods.}

In anomaly detection, spatial location often provides a useful cue: different regions of an object or scene tend to correspond to different functional parts (e.g., screw vs.\ background), while small shifts, pose changes, and viewpoint variations may still occur.
Matching a query patch against anomaly-free patches from arbitrary locations, even within the global-context retrieved set $\mathcal{N}_K(x)$, can therefore pair visually plausible but functionally unrelated regions, which may distort anomaly scores.
Thus, we use patch location as a coarse locality prior.

To implement this inductive bias, RAD applies spatial conditioning through local neighborhoods.
Instead of requiring rigid alignment, this design follows the intuition of local receptive fields: nearby grid locations are more likely to share similar functional semantics, while a neighborhood radius $\rho$ allows tolerance to small misalignments and viewpoint changes.
This preserves coarse spatial semantics without forcing each patch to match only the exact same grid position.

Formally, we view a test image $x$ as an $H \times W$ grid of patches and let
$t \leftrightarrow (h_t, w_t)$ index patch $t$.
For each layer $\ell$, the memory bank stores anomaly-free patch embeddings
$\{z^{(\ell)}_n\}$, where each $z^{(\ell)}_n$ comes from a patch at spatial
coordinates $(h_n, w_n)$ in an anomaly-free image.
Given a neighborhood radius $\rho \in \mathbb{Z}_{\ge 0}$, define $\mathcal{P}_\rho(t)
:=
\Big\{ n : \big\|(h_n, w_n) - (h_t, w_t)\big\|_\infty \le \rho \Big\}$ as the set of anomaly-free patches within an $\ell_\infty$-ball of radius $\rho$
around $t$.
For layer $\ell$ and retrieved image indices $\mathcal{N}_K(x)$, we then form
the position-aware candidate set of $\ell_2$-normalized anomaly-free patch embeddings
\begin{equation}
\label{eq:cand-set}
\mathcal{M}_\ell(x,t)
:=
\Big\{ z^{(\ell)}_{n} : n \in \mathcal{P}_\rho(t),\ i(n) \in \mathcal{N}_K(x) \Big\},
\end{equation}
where $i(n)$ is the index of the image from which patch $n$ is extracted.
Thus $\mathcal{M}_\ell(x,t)$ restricts patch matching to candidates that are globally compatible and locally plausible under this coarse spatial prior.

\paragraph{Multi-layer Anomaly Scoring.}
Given the globally and spatially conditioned candidate set $\mathcal{M}_\ell(x,t)$, RAD assigns each test patch an anomaly score by measuring how well it is supported by retrieved features.
For a layer $\ell$ test patch embedding $z^{(\ell)}_{t}(x)$ at location $t$ in image $x$, we define the layer-wise anomaly score via 1-NN cosine dissimilarity:
\begin{equation}
\label{eq:patch-energy-cos}
S_\ell(x,t)
:= 1 - \max_{z \in \mathcal{M}_\ell(x,t)}
\left\langle z^{(\ell)}_{t}(x), z \right\rangle.
\end{equation}
If a test patch is well supported by retrieved features, it admits a close neighbor in $\mathcal{M}_\ell(x,t)$ and thus a small $S_\ell(x,t)$, whereas patches that deviate from the local anomaly-free distribution yield larger $S_\ell(x,t)$.

To aggregate evidence across layers, we adopt a simple weighted fusion and combine the layer-wise scores as
\begin{equation}
S(x,t)
:=
\sum_{\ell \in \mathcal{L}} w_\ell \, S_\ell(x,t),
\quad
w_\ell \ge 0,\ \sum_{\ell \in \mathcal{L}} w_\ell = 1.
\end{equation}
Here, $w_\ell$ denotes the layer-fusion weight.

Each layer thus contributes as an expert in its own feature space, and the fused patch scores $S(x,t)$ are finally upsampled to pixel space for localization, while image-level anomaly scores are obtained by pooling.

\subsection{Computational Efficiency Analysis}
\label{sec:efficiency}
Let $P$ denote the number of patch tokens per image. A naive dense patch-retrieval scheme compares $P$ test patches against the $NP$ stored patch features, leading to $O(NP^2)$ patch comparisons. RAD instead first performs global retrieval over $N\times$ \texttt{[CLS]} descriptors with cost $O(N)$, and then conducts dense patch matching only within the top-$K$ retrieved images, reducing the dominant patch-level cost to $O(KP^2)$. Thus, the overall retrieval cost becomes $O(N+KP^2)$ rather than $O(NP^2)$, with $K \ll N$ in practice.
\section{Why is Retrieval Better?}
\label{sec:why-retrieval}
We now explain why, for the same encoder and training data, RAD avoids the fidelity--stability dilemma analyzed in Sec.~\ref{sec:theory}, and why its retrieval score is maximal within a class of stable feature-space scores.

\paragraph{Fidelity.}
Encoder--decoder methods attempt to reconstruct anomaly-free patch embeddings
through a parametric, bottlenecked mapping, which inevitably incurs
function-approximation error on the training embeddings due to information loss
(Sec.~\ref{sec:theory}).
In contrast, RAD is task-specific training-free as it stores all anomaly-free embeddings in a
memory bank and defines the retrieval score as the distance to this empirical support.


Let $\gamma$ denote the set of stored anomaly-free embeddings, and define
\begin{equation}
S_{\mathrm{ret}}(z) := d_\gamma(z) := \min_{u \in \gamma} \|z-u\|,
\end{equation}
where $\|\cdot\|$ is the feature-space norm used for retrieval.
Then for every stored normal embedding $z_i \in \gamma$ we have
$S_{\mathrm{ret}}(z_i) = 0$, since its nearest neighbor is itself.
Thus retrieval attains zero empirical distance on the stored anomaly-free embeddings.

More importantly, $S_{\mathrm{ret}}$ is the canonical distance-to-set score.
Let $\mathcal{F}_\gamma$ be the class of non-negative, $1$-Lipschitz feature-space
scores that vanish on $\gamma$. For any $S\in\mathcal{F}_\gamma$, we have
$S(z)\le d_\gamma(z)=S_{\mathrm{ret}}(z)$, so retrieval is pointwise maximal within
this stable score class. Moreover, for anomalous--normal feature pairs $(A,N)$
drawn from a distribution $\pi$ with
$\mathbb{E}_{\pi}[d_\gamma(A)+d_\gamma(N)]<\infty$, and the expected score gap
\begin{equation}
J_\pi(S)=\mathbb{E}_\pi[S(A)-S(N)],
\end{equation}
we show in Appendix Sec.~\ref{sec:proof-retrieval} that
\begin{equation}
\sup_{S\in\mathcal{F}_\gamma} J_\pi(S)-J_\pi(S_{\mathrm{ret}})
\le
\mathbb{E}_\pi[d_\gamma(N)].
\end{equation}
Thus, when unseen normal features are close to the anomaly-free memory,
retrieval is near-optimal within this stable score class for expected
anomaly--normal separation. Empirically, on MVTec-AD, the mean distance
from an unseen normal test patch to its nearest anomaly-free memory patch is only
\textbf{0.024}, supporting the small-$\mathbb{E}_\pi[d_\gamma(N)]$ regime.
We further report image-level unseen-normal scores and false-positive rates in
Appendix Sec.~\ref{sec:unseen-normal}.

\begin{table*}[t]
\centering
\caption{Standard MUAD results.
Best results are highlighted in boldface, and the runner-up is underlined.}
\label{tab:muad_results}

\begingroup
\small                        
\setlength{\tabcolsep}{3.8pt}   
\renewcommand{\arraystretch}{1.15}
\setlength{\extrarowheight}{0.3ex}

\begin{tabular}{l c|cc|cc|cc}
\specialrule{1.2pt}{0pt}{2pt}

\multicolumn{2}{c|}{\textbf{Dataset} $\rightarrow$} &
\multicolumn{2}{c|}{MVTec-AD} &
\multicolumn{2}{c|}{VisA} &
\multicolumn{2}{c}{Real-IAD} \\
\specialrule{0.6pt}{2pt}{2pt}

\multicolumn{2}{c|}{\textbf{Metric} $\rightarrow$} &
\multicolumn{3}{c}{Image-level (I-AUROC/I-AP/I-$F_1$-max)} &
\multicolumn{3}{c}{Pixel-level (P-AUROC/P-AP/P-$F_1$-max/AUPRO)} \\
\specialrule{0.6pt}{2pt}{2pt}

\multicolumn{2}{c|}{\textbf{Method} $\downarrow$}  &
\multicolumn{1}{c}{Image-level} & \multicolumn{1}{c|}{Pixel-level} &
\multicolumn{1}{c}{Image-level} & \multicolumn{1}{c|}{Pixel-level} &
\multicolumn{1}{c}{Image-level} & \multicolumn{1}{c}{Pixel-level} \\
\specialrule{0.6pt}{2pt}{2pt}

\multicolumn{2}{c|}{RD4AD  \tiny \textit{CVPR'22}} &
94.6/96.5/95.2 & 96.1/48.6/53.8/91.1 &
92.4/92.4/89.6 & 98.1/38.0/42.6/91.8 &
82.4/79.0/73.9 & 97.3/25.0/32.7/89.6 \\

\multicolumn{2}{c|}{UniAD \tiny \textit{NIPS'22}}&
96.5/98.8/96.2 & 96.8/43.4/49.5/90.7 &
88.8/90.8/85.8 & 98.3/33.7/39.0/85.5 &
83.0/80.9/74.3 & 97.3/21.1/29.2/86.7 \\

\multicolumn{2}{c|}{SimpleNet \tiny \textit{CVPR'23}}&
95.3/98.4/95.8 & 96.9/45.9/49.7/86.5 &
87.2/87.0/81.8 & 96.8/34.7/37.8/81.4 &
57.2/53.4/61.5 & 75.7/2.8/6.5/39.0 \\

\multicolumn{2}{c|}{DeSTSeg \tiny \textit{CVPR'23}}&
89.2/95.5/91.6 & 93.1/54.3/50.9/64.8 &
88.9/89.0/85.2 & 96.1/39.6/43.4/67.4 &
82.3/79.2/73.2 & 94.6/37.9/41.7/40.6 \\

\multicolumn{2}{c|}{GeneralAD \tiny \textit{ECCV'24}}&
95.4/98.5/95.9 & 97.1/46.0/50.6/87.0 &
87.4/87.3/82.0 & 96.5/35.3/38.7/81.9 &
57.5/53.6/61.6 & 76.5/4.8/7.9/40.8 \\

\multicolumn{2}{c|}{DiAD \tiny \textit{AAAI'24}}&
97.2/99.0/96.5 & 96.8/52.6/55.5/90.7 &
86.8/88.3/85.1 & 96.0/26.1/33.0/75.2 &
75.6/66.4/69.9 & 88.0/2.9/7.1/58.1 \\

\multicolumn{2}{c|}{MambaAD \tiny \textit{NIPS'24}}&
98.6/99.6/97.8 & 97.7/56.3/59.2/93.1 &
94.3/94.5/89.4 & 98.5/39.4/44.0/91.0 &
86.3/84.6/77.0 & 98.5/33.0/38.7/90.5 \\

\multicolumn{2}{c|}{OmiAD \tiny \textit{ICML'25}}&
98.8/99.7/98.5 & 97.7/52.6/56.7/93.2 &
95.3/96.0/91.2 & 98.9/40.4/44.1/89.2 &
\underline{90.1}/\underline{88.6}/\underline{82.8} & \underline{98.9}/37.7/42.6/93.1 \\

\multicolumn{2}{c|}{Dinomaly \tiny \textit{CVPR'25}}&
\textbf{99.6}/\textbf{99.8}/\textbf{99.0} &
\underline{98.4}/\underline{69.3}/\underline{69.2}/\underline{94.8} &
\textbf{98.7}/\textbf{98.9}/\textbf{96.2} &
\underline{98.7}/\underline{53.2}/\underline{55.7}/\underline{94.5} &
89.3/86.8/80.2 &
98.8/\underline{42.8}/\underline{47.1}/\underline{93.9} \\

\specialrule{0.6pt}{2pt}{3pt}

\rowcolor{blue!10}
\multicolumn{2}{c|}{\textbf{RAD (Ours)}} &
\textbf{99.6}/\textbf{99.8}/\textbf{99.0} &
\textbf{98.5}/\textbf{75.6}/\textbf{71.3}/\textbf{94.9} &
\underline{97.5}/\underline{97.3}/\underline{95.2} &
\textbf{99.1}/\textbf{55.3}/\textbf{57.5}/\textbf{94.6} &
\textbf{92.0}/\textbf{89.9}/\textbf{83.0} &
\textbf{99.1}/\textbf{53.7}/\textbf{54.2}/\textbf{96.1} \\

\specialrule{1.2pt}{2pt}{2pt}
\end{tabular}
\endgroup
\end{table*}

\paragraph{Inherited Stability.}
Retrieval also inherits a strong stability guarantee in feature space.
Because the encoder is frozen, the memory set $\gamma$ is fixed, and the
retrieval score is computed solely as a distance to this fixed set.
Hence, we obtain the following property.
\begin{lemma}[Non-expansiveness]
\label{lem:nonexpansive}
For any features $u,v \in \mathbb{R}^D$,
\begin{equation}
    \big|S_{\mathrm{ret}}(u) - S_{\mathrm{ret}}(v)\big|
    \le \|u - v\|.
\end{equation}
\end{lemma}

\emph{Proof.}
Let $a_v \in \arg\min_{a\in\gamma} \|v-a\|$.
By the triangle inequality,
$S_{\mathrm{ret}}(u) \le \|u-a_v\|
\le \|u-v\| + \|v-a_v\|
= \|u-v\| + S_{\mathrm{ret}}(v)$,
which yields $S_{\mathrm{ret}}(u) - S_{\mathrm{ret}}(v) \le \|u-v\|$.
Exchanging $u$ and $v$ gives the claim.

Lemma~\ref{lem:nonexpansive} shows that retrieval does not amplify benign perturbations in the encoder features as small changes in $z$ translate to at most equally small changes in $S_{\mathrm{ret}}(z)$; i.e. the variation in anomaly score between any two features is upper bounded by their feature-space distance.
This stands in sharp contrast to reconstruction-based anomaly scoring in
Sec.~\ref{sec:theory}, where information loss in the bottleneck forces the
decoder to have large Jacobian norm (Lemma~\ref{lem:amplification}), so the
residual score $S_{\mathrm{rec}}$ inevitably exhibits decoder amplification. The full proofs are presented in Appendix Sec.~\ref{sec:proof-retrieval}.

\paragraph{Discussion.}
Together, these properties explain why RAD avoids the fidelity--stability
dilemma faced by encoder--decoder architectures.
On the one hand, it attains \emph{fidelity} by using the canonical distance-to-set score on the empirical set of anomaly-free embeddings, which assigns zero score to stored normal embeddings and is pointwise maximal within the stable score class $\mathcal{F}_\gamma$.
Moreover, when unseen normal features remain close to the anomaly-free memory, this score is near-optimal within $\mathcal{F}_\gamma$ for expected anomaly--normal separation.
On the other hand, it maintains \emph{stability} by operating with this non-expansive nearest-neighbor score in frozen feature space, rather than through a high-gain decoder.
In this sense, retrieval-based scoring preserves empirical fidelity while controlling score stability, providing a principled way to avoid the decoder-induced fidelity--stability dilemma.

\section{Experiments and Analysis}\label{sec:exp&analysis}
\subsection{Experimental Setup}
\paragraph{Datasets.} Our empirical analysis is performed on four widely adopted multi-class anomaly detection (MUAD) datasets: \textbf{MVTec-AD}~\cite{bergmann2019mvtec}, \textbf{VisA}~\cite{zou2022spot}, \textbf{Real-IAD}~\cite{wang2024real} and \textbf{3D-ADAM}~\cite{mchard20253d}. 
Specifically, \textbf{MVTec-AD} comprises 15 texture and object categories with high-resolution, lab-style imagery and localized visually salient defects. \textbf{VisA} covers 12 everyday product categories with more cluttered backgrounds and subtle fine-grained anomalies, stressing generalization beyond controlled setups. \textbf{Real-IAD} further scales to a large, real-production benchmark (151,230 samples) with substantial intra-class variation and naturally occurring defects. Finally, \textbf{3D-ADAM} targets 3D anomaly detection in advanced manufacturing, providing RGB+3D multi-camera scans of machined parts in a factory environment with diverse surface defects.

\paragraph{Metrics.} Consistent with previous studies~\cite{he2024mambaad,he2024diffusion,fengomiad,guo2025dinomaly}, we use seven metrics for evaluation and follow their metric computation protocols. Image-level anomaly detection performance is assessed using the Area Under the Receiver Operating Curve (AUROC), Average Precision (AP), and the $F_1$-score at the optimal threshold ($F_1$-max). For pixel-level anomaly localization, we utilize AUROC, AP, $F_1$-max, and the Area Under the Per-Region-Overlap (AUPRO).
\paragraph{Implementation Details.} Our proposed RAD adopts a frozen DINOv3~\cite{simeoni2025dinov3} ViT-B/16 encoder. We extract patch tokens from layers $4, 7, 10$, and $12$ to perform multi-level feature matching. The neighborhood radius $\rho$ is set to $1$ for MVTec-AD, $2$ for VisA, $0$ for Real-IAD and $1$ for 3D-ADAM to accommodate different levels of viewpoint changes. We use uniform layer fusion weight $w_\ell = 0.25$ in all experiments. Detailed settings are in Appendix Sec.~\ref{sec:id}. Appendix Section~\ref{sec:sensitivity} further shows that RAD remains robust across different hyperparameter settings.

\subsection{Standard Multi-class Anomaly Detection.}
We benchmark RAD against a broad set of state-of-the-art multi-class anomaly detection (MUAD) baselines under the standard MUAD setting~\cite{guo2025dinomaly}, including reconstruction-based methods RD4AD~\cite{deng2022anomaly}, UniAD~\cite{you2022unified}, DiAD~\cite{he2024diffusion}, MambaAD~\cite{he2024mambaad}, Dinomaly~\cite{guo2025dinomaly}, OmiAD~\cite{fengomiad}, as well as embedding-based approaches SimpleNet~\cite{liu2023simplenet},  DeSTSeg~\cite{zhang2023destseg} and GeneralAD~\cite{strater2024generalad}. As summarized in Table~\ref{tab:muad_results}, RAD achieves state-of-the-art performance across all three benchmarks, with particularly strong gains on pixel-level localization.
On the most widely used MVTec-AD, RAD delivers its main improvements on localization, outperforming the previous best (Dinomaly) by \textbf{6.3$\uparrow$/2.1$\uparrow$} on \textbf{P-AP/P-$F_1$-max}, while remaining on par with the strongest baselines at the image level.
On VisA, RAD achieves the best pixel-level performance, outperforming the strongest baseline (Dinomaly) by \textbf{2.1$\uparrow$/1.8$\uparrow$} on \textbf{P-AP/P-$F_1$-max}, while remaining competitive at the image level.
On the most challenging multi-view Real-IAD, RAD further extends these gains, achieving \textbf{1.9$\uparrow$/1.3$\uparrow$/0.2$\uparrow$} improvements on image-level metrics and \textbf{0.2$\uparrow$/10.9$\uparrow$/7.1$\uparrow$/2.2$\uparrow$} on pixel-level metrics. We provide a same-backbone comparison with Dinomaly in Appendix Sec.~\ref{sec:same-backbone} to further separate the contribution of retrieval from backbone strength. Per-class performances and qualitative visualizations are presented in Appendix Sec.~\ref{sec:pre-category} and Sec.~\ref{sec:visualization}.

Although our study primarily targets conventional 2D MUAD, we also report a preliminary evaluation on the 3D benchmark 3D-ADAM using only its RGB modality; this dataset exhibits substantially larger geometric and appearance variation than standard 2D benchmarks.
In this challenging regime, RAD achieves the best performance on all seven metrics, outperforming upgraded Dinomaly by \textbf{6.5$\uparrow$/3.3$\uparrow$/2.8$\uparrow$} on \textbf{I-AUROC/I-AP/I-$F_1$-max} and by \textbf{0.1$\uparrow$/2.3$\uparrow$/2.4$\uparrow$/0.5$\uparrow$} on \textbf{P-AUROC/P-AP/P-$F_1$-max/P-AUPRO}.
More details are presented in Appendix Sec.~\ref{sec:adam}.




\begin{figure}[t]
    \centering
    \includegraphics[width=1\linewidth]{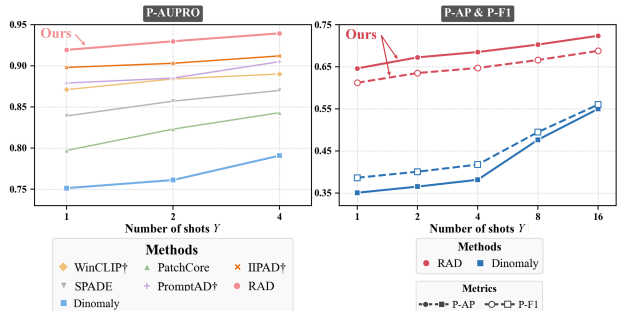}
    \caption{Few-shot results on MVTec-AD.  $\dag$ indicates few-shot-specific method.}
    \label{fig:fewshot_scaling}
\end{figure}

\subsection{Few-Shot Multi-class Anomaly Detection}
To study whether a task-specific training-free method can compete with or even surpass dedicated few-shot MUAD approaches, we evaluate our RAD under the multi-class MVTec-AD protocol of IIPAD~\cite{lvone}, where each category is limited to at most $Y$ anomaly-free references.

As shown in Fig.~\ref{fig:fewshot_scaling}, RAD is highly effective in this setting. On Pixel AUPRO at $Y \in \{1,2,4\}$, RAD consistently outperforms the unified baseline Dinomaly by \textbf{16.8$\uparrow$/16.9$\uparrow$/14.9$\uparrow$}, and also surpasses all few-shot-specific baselines. In particular, it improves over the strongest few-shot method IIPAD~\cite{lvone} by \textbf{2.1$\uparrow$/2.7$\uparrow$/2.7$\uparrow$} at $Y \in \{1,2,4\}$, with larger margins over PromptAD~\cite{li2024promptad}, WinCLIP~\cite{jeong2023winclip}, SPADE~\cite{cohen2020sub}, and PatchCore~\cite{roth2022towards}. More comparison results are presented in Appendix Sec.~\ref{sec:few-shot}.

To probe localization quality along the full few-shot trajectory, we additionally track P-AP and P-$F_1$ for $Y{=}1\rightarrow16$, metrics omitted by existing few-shot-specific baselines, under the same few-shot protocol using Dinomaly as the MUAD baseline. RAD maintains large margins throughout scaling, with \textbf{29.5$\uparrow$/25.7$\uparrow$} P-AP/P-$F_1$ at $Y{=}1$ and consistently elevated curves as $Y$ increases.


Overall, these results show that RAD is a plug-and-play few-shot detector. Even when each class has only a few anomaly-free samples and no few-shot training, it exceeds specialized few-shot methods while providing strong, stable localization across the entire few-shot range.


\subsection{Data Scaling in Cold-Start Anomaly Detection}
\label{sec:data_sacling}

In this section, we analyze how scaling anomaly-free data affects MUAD under cold-start conditions on MVTec-AD, using Dinomaly~\cite{guo2025dinomaly} as the state-of-the-art baseline. We focus on practically relevant scenarios with limited initial data and incomplete class coverage. In contrast, most existing evaluations assume a fixed, relatively clean anomaly-free set, leaving progressive scaling of anomaly-free data largely underexplored. We consider three complementary scaling settings: \emph{single-class scaling}, which follows the standard per-category protocol and isolates per-class sample efficiency; \emph{multi-class scaling}, which pools anomaly-free images across categories to assess robustness under mixed deployments; and \emph{incremental-class scaling}, which mimics deployment-time onboarding of a new class as its anomaly-free evidence grows. Together, these regimes disentangle (i) per-class sample efficiency, (ii) robustness under multi-class mixture, and (iii) behavior under continual class expansion.  Table~\ref{tab:scaling} summarizes results.

\paragraph{Single-class Scaling.}
This setting tests whether RAD’s gains are merely artifacts of cross-class mixture or persist in the standard single-class setting. Each category is evaluated independently using only its anomaly-free in-class data, and results are averaged across categories, removing cross-category interference and isolating the effect of data scarcity.

RAD remains robust under severe data scarcity. When the available anomaly-free data ratio is $\tau{=}0.05$, image-level detection is already nearly saturated; increasing $\tau$ from $0.05$ to $1.00$ raises Image AUROC by only \textbf{2.5$\uparrow$}, whereas pixel-level localization continues to benefit from additional anomaly-free data, with P-AP gaining \textbf{5.1$\uparrow$}. RAD consistently outperforms Dinomaly, and the gains are dominated by localization: at $\tau{=}0.05$, P-AP, P-$F_1$, and AUPRO improve by \textbf{15.9$\uparrow$}, \textbf{11.2$\uparrow$}, and \textbf{7.1$\uparrow$}, respectively.

These results indicate that, even when category mixture is removed and each class is evaluated independently, the same pattern holds: image-level metrics are nearly saturated, but localization continues to benefit from additional anomaly-free evidence, and RAD provides robust, persistent localization gains across the full single-class scaling range.

\begin{table}[!t]
\centering
\caption{Data scaling on MVTec-AD with three settings mentioned in Section \ref{sec:data_sacling}. $\tau$ denotes the ratio of anomaly-free data usage.}
\scriptsize
\setlength{\tabcolsep}{3.0pt}
\resizebox{\linewidth}{!}{
\begin{tabular}{l |c| c c c c c c c}
\toprule
\multirow{2}{*}{Method} & \multirow{2}{*}{$\tau\uparrow$} &
\multicolumn{3}{c}{Image-level} & \multicolumn{4}{c}{Pixel-level} \\
\cmidrule(lr){3-5} \cmidrule(lr){6-9}
& & AUROC & AP & $F_1$-max & AUROC & AP & $F_1$-max & AUPRO \\
\bottomrule

\specialrule{0em}{1pt}{1pt}
\rowcolor{morandiL}
\multicolumn{9}{c}{\textbf{Single-class Scaling}} \\ 
\midrule
\multirow{5}{*}{Dinomaly}
& 0.05  & 92.3 & 95.7 & 93.7 & 92.6 & 54.7 & 56.2 & 87.4 \\
& 0.20  & 94.0 & 96.9 & 95.5 & 94.8 & 58.0 & 58.9 & 89.7 \\
& 0.40& 97.0 & 98.7 & 97.4 & 97.0 & 64.5 & 64.5 & 92.8 \\
& 0.80& 99.1 & 99.6 & 98.4 & 97.8 & 66.7 & 66.7 & 93.8 \\
& 1.00 & 99.3 & 99.7 & 98.8 & 98.0 & 67.4 & 67.3 & 94.2 \\
\midrule
\multirow{5}{*}{RAD}
& 0.05  & 97.1 & 98.5 & 97.2 & 98.0 & 70.6 & 67.4 & 94.5 \\
& 0.20  & 98.8 & 99.5 & 98.3 & 98.3 & 74.1 & 70.1 & 95.4 \\
& 0.40& 99.2 & 99.7 & 98.7 & 98.4 & 74.9 & 70.8 & 95.7 \\
& 0.80& 99.5 & 99.8 & 98.8 & 98.5 & 75.6 & 71.4 & 95.9 \\
& 1.00 & 99.6 & 99.8 & 98.9 & 98.5 & 75.7 & 71.4 & 95.8 \\
\bottomrule

\specialrule{0em}{1pt}{1pt}
\rowcolor{morandiM}
\multicolumn{9}{c}{\textbf{Multi-class Scaling}} \\ 
\midrule
    \multirow{6}{*}{Dinomaly} & 0.05 & 90.5 & 94.8 & 93.4 & 93.0 & 51.1 & 52.1 & 87.0 \\
     & 0.10 & 95.3 & 97.9 & 96.2 & 96.3 & 61.7 & 61.5 & 92.0 \\
     & 0.20 & 98.6 & 99.5 & 98.3 & 97.8 & 67.4 & 66.8 & 94.0 \\
     & 0.40 & 99.1 & 99.6 & 98.7 & 98.1 & 69.1 & 68.5 & 94.4 \\
     & 0.80 & 99.5 & 99.8 & 99.1 & 98.3 & 69.2 & 69.0 & 94.5 \\
     & 1.00 & 99.6 & 99.8 & 99.0 & 98.3 & 69.0 & 68.9 & 94.5 \\
    \midrule
    \multirow{6}{*}{RAD}    &  0.05 & 95.4 & 97.5 & 96.6 & 97.7 & 67.9 & 64.7 & 93.9 \\
     & 0.10 & 97.7 & 98.8 & 97.5 & 98.2 & 71.7 & 68.1 & 95.1 \\
     & 0.20 & 98.6 & 99.4 & 98.3 & 98.3 & 73.7 & 69.9 & 95.5 \\
     & 0.40 & 99.3 & 99.7 & 98.6 & 98.4 & 74.7 & 70.7 & 95.7 \\
     & 0.80 & 99.6 & 99.8 & 98.9 & 98.5 & 75.5 & 71.3 & 95.9 \\
     & 1.00 & 99.6 & 99.8 & 99.0 & 98.5 & 75.6 & 71.3 & 95.9 \\
\bottomrule

\specialrule{0em}{1pt}{1pt}
\rowcolor{morandiD}
\multicolumn{9}{c}{\textbf{Incremental-class Scaling}} \\ 
\midrule
\multirow{6}{*}{Dinomaly}
& 0.01 & 79.8 & 76.4 & 71.0 & 78.7 & 33.1 & 35.0 & 57.7 \\
& 0.02 & 78.8 & 74.8 & 68.7 & 81.1 & 35.5 & 36.4 & 60.8 \\
& 0.10 & 95.9 & 92.7 & 90.0 & 88.6 & 48.7 & 45.9 & 69.5 \\
& 0.20 & 98.1 & 96.2 & 94.0 & 90.4 & 53.3 & 50.6 & 71.6 \\
& 0.80 & 99.0 & 98.0 & 96.3 & 93.2 & 60.7 & 58.2 & 76.1 \\
& 1.00 & 98.8 & 97.6 & 96.3 & 93.5 & 61.6 & 59.1 & 76.2 \\
\midrule
\multirow{6}{*}{RAD}
& 0.01 & 82.3 & 79.2 & 70.9 & 93.0 & 53.7 & 51.9 & 74.9 \\
& 0.02 & 96.3 & 96.1 & 92.3 & 95.3 & 75.1 & 68.6 & 88.0 \\
& 0.10 & 99.0 & 98.7 & 94.0 & 96.9 & 79.3 & 71.5 & 90.3 \\
& 0.20 & 98.8 & 98.2 & 94.0 & 96.9 & 78.1 & 70.7 & 90.9 \\
& 0.80 & 99.5 & 99.2 & 96.3 & 97.4 & 80.6 & 72.4 & 92.6 \\
& 1.00 & 99.5 & 99.3 & 96.3 & 97.4 & 80.7 & 72.4 & 92.5 \\
\bottomrule
\end{tabular}
}
\label{tab:scaling}
\end{table}

\paragraph{Multi-class Scaling.}
We next pool anomaly-free images across categories and study how a multi-class detector behaves as the total anomaly-free samples increases. This setting evaluates how performance changes as normal coverage increases uniformly across all categories.

Fraction scaling mirrors the metric-scaling trend observed in the single-class setting. This indicates that the dominant challenge is refining a precise and stable \emph{localization} boundary rather than achieving coarse image-level detection. In the most data-constrained regime with $\tau{=}0.05$, RAD exhibits strong plug-and-play behavior and achieves clear gains over the strongest unified baseline Dinomaly, improving I-AUROC by \textbf{4.9$\uparrow$} and P-AP by \textbf{16.8$\uparrow$}. Even at $\tau{=}1.00$, RAD still retains advantages in localization, indicating consistent benefits across the full cold-start scaling range.


\paragraph{Incremental-class Scaling.}
How well can a deployed multi-class detector onboard a new class over time without re-training? To answer this, we fix a training set of four relatively challenging categories (\textit{carpet}, \textit{metal\_nut}, \textit{toothbrush}, \textit{leather}) and introduce a new class \textit{transistor} with a gradually increasing ratio $\tau$ of target-class anomaly-free images. 

At $\tau{=}0.01$, where only two \textit{transistor} anomaly-free images are available, RAD already attains $93.0$ P-AUROC (95.4\% of its full performance). Relative to Dinomaly, this corresponds to gains of \textbf{20.6$\uparrow$} in P-AP, and \textbf{17.2$\uparrow$} in AUPRO. These localization margins remain large throughout the scaling process. While image-level metrics for both methods become nearly saturated, RAD consistently lifts the entire pixel-level scaling curve for \textit{transistor}.


Overall, these results indicate that RAD (i) is highly sample-efficient for onboarding new classes, recovering most of its eventual performance from only a handful of anomaly-free examples, and (ii) continues to exploit additional anomaly-free data more effectively than a trained MUAD model, suggesting that task-specific training-free, retrieval-based anomaly detection remains advantageous than training-based approaches even as class-specific anomaly-free sample accumulates. We further validate this incremental-class scaling behavior on VisA in Appendix Sec.~\ref{sec:visa_incremental_scaling}, where RAD remains stronger than Dinomaly across all memory sizes under a challenging target-class scaling protocol.

\subsection{Impact of Encoder Quality and Resolution on RAD}
Since RAD is entirely task-specific training-free, its performance is driven by the quality of its frozen encoder and the fidelity of its inputs. A natural question is therefore: beyond increasing the amount of available anomaly-free data, can RAD benefit from stronger foundation encoders and higher input resolution?

We evaluate RAD with a range of frozen ViT foundation encoders in the MUAD setting, using pixel-level AUPRO (P-AUPRO) as the primary localization metric. With image resolution fixed at $224{\times}224$, the encoders exhibit a clear performance hierarchy. DINOv3~\cite{simeoni2025dinov3} yields the highest P-AUPRO among all foundations, followed by other contrastive and hybrid encoders such as DINOv2~\cite{oquab2023dinov2}, iBOT~\cite{zhou2021ibot}, DINO~\cite{caron2021emerging}, D-iGPT~\cite{ren2023rejuvenating}, and MoCov3~\cite{chen2021empirical}, and then by supervised DeiT~\cite{touvron2021training}. Encoders pre-trained with masked image modeling (MIM), including MAE~\cite{he2022masked}, BEiT~\cite{bao2021beit}, and BEiTv2~\cite{peng2022beit}, obtain the lowest P-AUPRO. This ordering closely mirrors their reported ImageNet~\cite{deng2009imagenet} linear probing strength, indicating that  RAD can directly inherit advances in generic visual representation learning: simply swapping to a stronger foundation encoder yields better anomaly localization, making RAD naturally compatible with future backbone improvements.


We next increase the image resolution to $448{\times}448$. P-AUPRO improves for almost all encoders, while the relative ranking across foundation encoders remains essentially unchanged.  Overall, these trends show that RAD’s anomaly localization ability is affected by the encoder’s ImageNet representation quality, while higher input resolution provides a mostly uniform boost without altering this dependence. Detailed results are presented in Appendix Sec.~\ref{sec:encoder}.

\subsection{Efficiency and Memory Compression}
\label{sec:efficiency}

We further examine the practical cost of RAD in terms of inference latency and memory footprint. On a single NVIDIA RTX 5090 GPU, RAD runs at \textbf{21.6 ms/image} on 3D-ADAM, corresponding to \textbf{46.3 FPS}. On MVTec-AD, the global \texttt{[CLS]} retrieval stage reduces inference latency from \textbf{293.7 ms/image} to \textbf{44.7 ms/image} by restricting dense patch matching to the top-$K$ retrieved references. The memory footprint requires \textbf{12.1/32.6/40.7/60.8 GB} on 3D-ADAM, MVTec-AD, VisA, and Real-IAD, respectively. Although this storage cost is higher than purely feed-forward models, it can be effectively reduced by memory compression.

On MVTec-AD, we evaluate a coreset-compressed variant of RAD. As shown in Table~\ref{tab:coreset}, coreset sampling reduces the memory bank from \textbf{32.6 GB} to \textbf{0.37 GB}, while causing only minor degradation in localization performance. Specifically, P-AUROC/P-AUPRO drop by only \textbf{0.44/0.15}, and P-AP/P-$F_1$-max remain close to the full-memory version. This suggests that coreset compression is a practical way to scale RAD to larger datasets with limited memory overhead.

\begin{table}[t]
  \centering
  \caption{Effect of coreset compression on RAD on MVTec-AD.}
  \label{tab:coreset}

  \begingroup
  \footnotesize
  \setlength{\tabcolsep}{4pt}
  \renewcommand{\arraystretch}{0.95}

  \resizebox{0.98\linewidth}{!}{%
  \begin{tabular}{lccccc}
    \toprule
    Method & Memory & P-AUROC & P-AP & P-$F_1$ & P-AUPRO \\
    \midrule
    RAD & 32.6 GB & 98.50 & 75.60 & 71.30 & 94.90 \\
    RAD + coreset & 0.37 GB & 98.06 & 73.86 & 70.24 & 94.75 \\
    \bottomrule
  \end{tabular}%
  }

  \endgroup
  \vspace{-4mm}
\end{table}

\section{Conclusion and Limitations}
In this work, we revisited the question of whether task-specific training is necessary for strong multi-class unsupervised anomaly detection (MUAD). We further introduced RAD, a task-specific training-free, retrieval-based MUAD framework that surpasses previous state-of-the-art MUAD methods. We hope that retrieval-based anomaly scoring will open up new opportunities for task-specific training-free anomaly detection research. However, RAD has several limitations. First, keeping a memory of patch features incurs an additional storage cost. Second, retrieval introduces extra latency compared to a single feed-forward pass at inference time. We view these costs as an invitation to explore more efficient storage and retrieval methods.



\section*{Impact Statement}
This paper presents work whose goal is to advance multi-class unsupervised anomaly detection (MUAD) with a task-specific training-free retrieval-based scoring framework that replaces domain-specific training with multi-level retrieval over a memory of normal features. Looking ahead, we hope that this retrieval-based perspective will pave the way for new task-specific training-free MUAD research and motivate the community to re-examine whether domain-specific training is truly necessary for strong anomaly detection. There are many potential societal consequences of our work, none of which we feel must be specifically highlighted here.

\bibliographystyle{icml2026}
\bibliography{main}

\newpage
\appendix
\onecolumn

\setcounter{figure}{0}
\setcounter{table}{0}
\renewcommand{\thefigure}{A\arabic{figure}}
\renewcommand{\thetable}{A\arabic{table}}

\section*{Appendix Overview}
This appendix provides supplementary details supporting the main manuscript,
organized as follows:

\begin{itemize}[leftmargin=*,label={}]
  \item \textemdash~\textbf{Sec.~A} provides full implementation details.
  \item \textemdash~\textbf{Sec.~B} provides more details about related works.
  \item \textemdash~\textbf{Sec.~C} provides proofs of our fidelity--stability dilemma analysis.
  \item \textemdash~\textbf{Sec.~D} provides proofs of our retrieval-based anomaly scoring analysis.
  \item \textemdash~\textbf{Sec.~E} reports additional results under multi-class few-shot settings.
  \item \textemdash~\textbf{Sec.~F} shows more details about the impact of encoder quality and resolution.
  \item \textemdash~\textbf{Sec.~G} reports and analyzes the results on the 3D-ADAM dataset.
  \item \textemdash~\textbf{Sec.~H} reports additional ablation results.
  \item \textemdash~\textbf{Sec.~I} provides a same-backbone comparison with Dinomaly.
  \item \textemdash~\textbf{Sec.~J} provides hyperparameter sensitivity analysis on $K$, $\rho$, and layer-fusion weights.
  \item \textemdash~\textbf{Sec.~K} analyzes unseen-normal generalization on the MVTec-AD dataset.
  \item \textemdash~\textbf{Sec.~L} reports additional incremental-class scaling results on the VisA dataset.
  \item \textemdash~\textbf{Sec.~M} reports per-category quantitative results on the MVTec-AD dataset, VisA dataset and Real-IAD dataset.
  \item \textemdash~\textbf{Sec.~N} shows qualitative visualizations on the MVTec-AD dataset, VisA dataset and Real-IAD dataset.
\end{itemize}

\section{Full Implementation Details}
\label{sec:id}
Our proposed RAD adopts a frozen DINOv3~\cite{simeoni2025dinov3} ViT-B/16 encoder. We extract patch tokens from layers $4, 7, 10$, and $12$ to perform multi-level feature matching. Input images are resized to $512^2$ followed by a $448^2$ center crop. The neighborhood radius $\rho$ is set to $1$ for MVTec-AD, $2$ for VisA, $0$ for Real-IAD and $1$ for 3D-ADAM to accommodate different levels of viewpoint changes. The number of nearest neighbors $K$ is set to $150$ for MVTec-AD, $900$ for VisA, $900$ for Real-IAD and $48$ for 3D-ADAM to accommodate varying dataset complexities. The mean of the top 1\% pixels in an anomaly map is used
as the image anomaly score. All experiments are conducted
with random seed=1 with cuda deterministic for invariable
weight initialization and batch order. Codes are implemented with Python 3.11 and PyTorch 2.7.0 cuda 12.8, and
run on an NVIDIA GeForce RTX5090 GPU (32GB).

\section{Related Work}
\label{sec:rw}

\paragraph{General Methods for Unsupervised Anomaly Detection} Here, we review representative approaches for unsupervised anomaly detection (UAD). Epistemic methods assume that networks react differently at inference time to seen and unseen inputs. Within this paradigm, pixel reconstruction methods train networks on anomaly-free images so that anomaly-free regions can be accurately reconstructed, whereas anomalous regions are expected to be poorly restored. Auto-encoders (AE)~\cite{zavrtanik2021reconstruction,hou2021divide}, variational auto-encoders (VAE)~\cite{liu2020towards,lu2023hierarchical}, and generative adversarial networks (GAN)~\cite{akcay2018ganomaly,yan2021learning} are commonly used for this purpose. However, pixel-level reconstruction may also restore previously unseen anomalies when they are visually similar to normal regions or only subtly deviate from them~\cite{deng2022anomaly}. To address this issue, feature reconstruction methods reconstruct features from pre-trained encoders instead of raw pixels~\cite{deng2022anomaly,yang2020dfr,you2022unified}, while freezing encoder parameters to avoid trivial solutions. In feature distillation~\cite{bergmann2020uninformed,salehi2021multiresolution}, a student network is trained from scratch to mimic the output features of a pre-trained teacher network given the same normal input images, under the same assumption that a student trained only on normal samples will specialize in normal-region features.

Beyond classical auto-encoder--type models, recent works exploit diffusion probabilistic models as stronger generative backbones for reconstruction-based UAD. Early diffusion-based approaches, such as AnoDDPM~\cite{wyatt2022anoddpm}, and DDAD~\cite{mousakhan2024anomaly}, employ denoising diffusion models either to synthesize anomalous samples or to generate pseudo-normal reconstructions from noisy inputs, and derive anomaly maps by contrasting input and denoised outputs. DiffAD~\cite{zhang2023unsupervised} further replaces the AE decoder with a latent diffusion model and introduces noisy condition embedding and interpolated channels to prevent direct copying of anomalous regions and to alleviate reconstruction ambiguity.

Pseudo-anomaly methods synthesize artificial defects on normal images to emulate anomalies, thereby converting UAD into supervised classification~\cite{li2021cutpaste} or segmentation problems~\cite{zavrtanik2021draem}. For example, CutPaste~\cite{li2021cutpaste} creates anomalous regions by randomly pasting cropped patches of normal images. DRAEM~\cite{zavrtanik2021draem} generates abnormal regions using noise as a mask and another image as the additive anomaly. DeTSeg~\cite{zhang2023destseg} adopts a similar anomaly-generation strategy and combines it with feature reconstruction. SimpleNet~\cite{liu2023simplenet} injects Gaussian noise into the pre-trained feature space to introduce anomalies. These approaches strongly depend on how well the synthesized anomalies approximate real ones, which limits their ability to generalize across datasets.

Finally, memory-based methods~\cite{defard2021padim,roth2022towards,bae2023pni,hyun2024reconpatch} store a large set of normal features (or their modeled distributions) extracted by networks pre-trained on large-scale datasets and compare them with test features at inference through nearest-neighbor search or density estimation in feature space. Representative approaches model patch-wise Gaussian statistics with Mahalanobis scoring~\cite{defard2021padim}, build compact prototype memories via coreset sampling~\cite{roth2022towards}, or refine the memory with more expressive probabilistic models and reconstruction-guided updates~\cite{bae2023pni,hyun2024reconpatch}. Together, these methods highlight the effectiveness of explicitly modeling feature-space memories for UAD and establish memory-based modeling as a strong foundation for accurate anomaly detection and localization.

\paragraph{Multi-Class Unsupervised Anomaly Detection}
While the above pipelines are often instantiated in a single-class regime, benchmarks with multiple object categories make it undesirable to train and maintain one detector per class. Multi-class UAD (MUAD) therefore considers learning a \emph{single} anomaly detector from the pooled anomaly-free data of all categories and reusing it for every class, typically without exploiting class labels during training. Because modeling a unified multi-class anomaly-free distribution is substantially harder than training separate one-class models, state-of-the-art MUAD methods rely on dedicated training pipelines that can be broadly grouped into encoder--decoder reconstruction, teacher--student distillation, memory-based matching, and diffusion-based generation.

Encoder--decoder reconstruction methods freeze a generic encoder and learn a shared decoder over all categories, detecting anomalies from reconstruction residuals in feature space. UniAD~\cite{you2022unified} is an early instantiation of this paradigm, and HVQ-Trans~\cite{lu2023hierarchical} further inserts a vector-quantized transformer bottleneck to suppress shortcut reconstructions while keeping a unified architecture. Subsequent works strengthen this reconstruction backbone: MambaAD~\cite{he2024mambaad} replaces the decoder with a Mamba-based state space model to better capture long-range dependencies across categories, and Dinomaly~\cite{guo2025dinomaly} shows that a carefully regularized reconstruction head on top of strong DINO features already yields state-of-the-art multi-class performance. Teacher--student MUAD methods adapt reverse-distillation pipelines such as Uninet~\cite{wei2025uninet} to the multi-class setting by training a single student network on mixed-category anomaly-free data to mimic a frozen teacher, using feature discrepancies for detection.

In parallel, memory-based MUAD techniques explicitly construct shared normal-feature memories instead of learning powerful decoders. CRAD~\cite{lee2024continuous} encodes anomaly-free features into a continuous grid-based memory that supports dense retrieval for all categories, ResAD~\cite{yao2024resad} models residual features within a shared hypersphere to reduce inter-class variation and improve class generalization, and OMAC~\cite{tan2024unsupervised} maintains dual memory banks with representative samples and a query mechanism so that a single model can handle all categories. Diffusion-based MUAD approaches, including DiAD and OmiAD~\cite{he2024diffusion,fengomiad}, leverage denoising diffusion models to synthesize pseudo-normal images or features conditioned on the test input, and sometimes distill the diffusion process into compact student networks for efficient deployment. Despite their architectural diversity, these MUAD pipelines share a common design philosophy: they all depend on task-specific training to approximate the multi-class anomaly-free distribution and to construct class-agnostic memories, leaving open whether a \emph{task-specific training-free} formulation can surpass these specialized architectures.

\paragraph{Few-Shot Multi-Class Unsupervised Anomaly Detection}
Pushing MUAD into an even more challenging regime, recent works study few-shot MUAD, where a single one-for-all detector must be adapted to multiple categories or domains given only a handful of normal exemplars per class. IIPAD~\cite{lvone} combines vision--language models with instance-induced prompt learning: instead of learning fixed prompts per category, it trains a class-shared prompt generator that produces instance-specific textual prompts guided by generic descriptions of normality and abnormality, together with a visual memory bank for retrieving similar features under the one-for-all paradigm. UniVAD~\cite{gu2025univad} further proposes a unified model for few-shot visual anomaly detection, leveraging a frozen foundation encoder and a unified feature memory to transfer across industrial, logical, and medical domains without domain-specific re-training. Complementary to these unified one-for-all designs, few-shot AD methods such as WinCLIP~\cite{jeong2023winclip} and PromptAD~\cite{li2024promptad} also leverage vision--language foundations and prompt learning, but are typically instantiated as separate detectors per dataset or category rather than a single scalable MUAD model.

\section{Proofs of Fidelity--Stability Analysis}
\label{sec:proof}

We collect here the detailed derivations underlying the discussion in
Sec.~\ref{sec:theory}. Throughout, all vector spaces are equipped with
the standard Euclidean inner product. We write $\|\cdot\|$ for the
Euclidean norm of vectors and $\|\cdot\|_{\mathrm{op}}$ for the
associated operator norm of linear maps (i.e., the spectral norm of a
matrix). For a linear map $L$, we denote by $\sigma_{\min}(L)$ and
$\sigma_{\max}(L)$ its smallest and largest singular values,
respectively.

\paragraph{Standing assumptions.}
Recall that the trained reconstruction map in feature space is
\begin{equation}
    R(z) \;\triangleq\; \Psi_{\theta^*}(B(z)), \qquad z \in \mathbb{R}^D,
\end{equation}
where $B:\mathbb{R}^D\to\mathbb{R}^d$ is the bottleneck and
$\Psi_{\theta^*}:\mathbb{R}^d\to\mathbb{R}^D$ is the decoder at the
trained parameters. During training and evaluation we only apply $R$ to
encoder features $z = \Phi(n)$ of image patches $n$, and we will often
write $R(n)$ as shorthand for $R(\Phi(n))$.

For notational simplicity, in what follows we drop the explicit
dependence on $\theta^*$ and write $\Psi := \Psi_{\theta^*}$, so that
$R = \Psi \circ B$ matches the shorthand used in
Sec.~\ref{sec:theory}.

For simplicity of analysis we assume that $B$ and $\Psi$ are
differentiable in a neighbourhood of the normal feature manifold, so
that $R = \Psi\circ B$ is differentiable there and all Jacobian matrices
are well-defined. For typical piecewise-linear architectures (e.g., ReLU
networks) this holds almost everywhere; our arguments are restricted to
points of differentiability on the normal manifold.

For an anomaly-free patch $n\in\mathcal{D}^{\mathrm{patch}}_{\mathrm{norm}}$,
recall that $\mathcal{V}_n\subset\mathbb{R}^D$ denotes a subspace of
benign feature variations around $\Phi(n)$. We view
$\mathcal{V}_n$ as a local linearization (tangent subspace) of the
normal feature manifold at $\Phi(n)$; in particular, it is
finite-dimensional, and in typical constructions its dimension is at
most $d$ so that it can be embedded by the bottleneck without collapsing
all directions.

\paragraph{Reconstruction error and amortized inverse.}
The reconstruction map in feature space is
$R(z) = \Psi(B(z))$, and for a patch $n$ we denote
$R(n) = R(\Phi(n))$. For each anomaly-free patch
$n\in\mathcal{D}^{\mathrm{patch}}_{\mathrm{norm}}$, we define the
\emph{reconstruction error} in feature space, in line with
Sec.~\ref{sec:theory}, as
\begin{equation}
    \varepsilon_{\mathrm{approx}}(n)
    \;\triangleq\;
    R(\Phi(n)) - \Phi(n),
\end{equation}
so that
\begin{equation}
    R(\Phi(n)) = \Phi(n) + \varepsilon_{\mathrm{approx}}(n),
    \qquad n \in \mathcal{D}_{\mathrm{norm}}^{\mathrm{patch}}.
\end{equation}
We call
\begin{equation}
    \alpha(n)\;\triangleq\;\big\|\varepsilon_{\mathrm{approx}}(n)\big\|
\end{equation}
the \emph{reconstruction gap} at patch $n$, and write
\begin{equation}
    \Delta_{\mathrm{app}}
    \;\triangleq\;
    \mathbb{E}_{n \sim
    \mathcal{D}^{\mathrm{patch}}_{\mathrm{norm}}} \alpha(n)
\end{equation}
for the dataset-average reconstruction gap.

Recall from Eq.~\eqref{eqa:residual_score} that the reconstruction-based
anomaly score is
$S_{\mathrm{rec}}(t) = \|\Phi(t) - R(t)\|$. For anomaly-free patches
$n\in\mathcal{D}^{\mathrm{patch}}_{\mathrm{norm}}$, we thus have
\begin{equation}
    \begin{aligned}
        S_{\mathrm{rec}}(n)
        &= \big\|\Phi(n) - R(\Phi(n))\big\| \\
        &= \big\|\Phi(n) - (\Phi(n) + \varepsilon_{\mathrm{approx}}(n))\big\| \\
        &= \big\|-\varepsilon_{\mathrm{approx}}(n)\big\| \\
        &= \big\|\varepsilon_{\mathrm{approx}}(n)\big\| \\
        &= \alpha(n).
    \end{aligned}
\end{equation}
Thus, on anomaly-free patches the reconstruction residuals are driven
entirely by the reconstruction error $\varepsilon_{\mathrm{approx}}$, and
$R$ behaves as a biased, amortized inverse of the bottlenecked encoder
on the normal feature manifold rather than an exact identity. The
magnitude $\alpha(n)$ and the geometry of
$\varepsilon_{\mathrm{approx}}(n)$ across $n$ govern how reconstruction
fidelity trades off against the stability of the residual score
$S_{\mathrm{rec}}(t)$.

\paragraph{Local fidelity and the Jacobian.}
We now justify the Jacobian-based form of
Assumption~\ref{ass:fidelity} used in Sec.~\ref{sec:theory}. For an
anomaly-free patch $n\in\mathcal{D}^{\mathrm{patch}}_{\mathrm{norm}}$,
recall that $\mathcal{V}_n\subset\mathbb{R}^D$ denotes a subspace of
benign feature variations around $\Phi(n)$. Define
\begin{equation}
    G(z) \;\triangleq\; R(z) - z,
\end{equation}
so that $R(z) = z + G(z)$ and, in particular,
$G(\Phi(n)) = R(\Phi(n)) - \Phi(n) = \varepsilon_{\mathrm{approx}}(n)$.

The geometric form of Assumption~\ref{ass:fidelity} states that there
exist $r>0$ and $\eta\in(0,1)$ such that, for all
$\delta\in\mathcal{V}_n$ with $\|\delta\|\le r$,
\begin{equation}
    \|R(\Phi(n)+\delta) - R(\Phi(n)) - \delta\|
    = \|G(\Phi(n)+\delta) - G(\Phi(n))\|
    \;\le\; \eta \|\delta\|.
    \label{eq:local_fid_G}
\end{equation}
In other words, $G$ is locally $\eta$-Lipschitz along benign directions
spanned by $\mathcal{V}_n$, and the reconstruction map $R$ is locally
near-identity along $\mathcal{V}_n$, up to relative error of order
$\eta$.

Assuming $R$ (and hence $G$) is differentiable at $\Phi(n)$, we can
differentiate \eqref{eq:local_fid_G} along directions in $\mathcal{V}_n$.
Let $v\in\mathcal{V}_n$ with $\|v\|=1$ and write $\delta = h v$. Then
\eqref{eq:local_fid_G} gives
\begin{equation}
    \frac{\|G(\Phi(n)+h v) - G(\Phi(n))\|}{|h|}
    \;\le\; \eta
    \qquad \text{for all sufficiently small $h$.}
\end{equation}
Taking the limit $h\to 0$ and using differentiability of $G$ at
$\Phi(n)$ yields
\begin{equation}
    \|J_G(\Phi(n)) v\| \;\le\; \eta
    \qquad \text{for all $v\in\mathcal{V}_n$ with $\|v\|=1$},
\end{equation}
which is equivalent to
\begin{equation}
    \big\|J_G(\Phi(n))\big|_{\mathcal{V}_n}\big\|_{\mathrm{op}}
    \;\le\; \eta.
\end{equation}

Next, note that $R(z) = z + G(z)$ implies
\begin{equation}
    J_R(\Phi(n)) = I + J_G(\Phi(n)),
\end{equation}
where $I$ is the identity on $\mathbb{R}^D$. For any unit vector
$v\in\mathcal{V}_n$ we have
\begin{equation}
    \begin{aligned}
        \big\|J_R(\Phi(n)) v\big\|
        &= \big\|(I + J_G(\Phi(n))) v\big\| \\
        &\ge \|v\| - \|J_G(\Phi(n)) v\|
        \qquad\text{(triangle inequality)} \\
        &\ge 1 - \eta,
    \end{aligned}
\end{equation}
using $\|v\|=1$ and $\|J_G(\Phi(n)) v\|\le \eta$.
Taking the minimum over all unit vectors $v\in\mathcal{V}_n$ gives
\begin{equation}
    \sigma_{\min}\big(J_R(\Phi(n))\big|_{\mathcal{V}_n}\big)
    = \min_{\substack{v\in\mathcal{V}_n \\ \|v\|=1}}
      \big\|J_R(\Phi(n)) v\big\|
    \;\ge\; 1 - \eta.
\end{equation}
Similarly, by the reverse triangle inequality we obtain the upper bound
\begin{equation}
    \big\|J_R(\Phi(n)) v\big\|
    \le \|v\| + \|J_G(\Phi(n)) v\|
    \le 1 + \eta,
\end{equation}
so that
\begin{equation}
    \sigma_{\max}\big(J_R(\Phi(n))\big|_{\mathcal{V}_n}\big)
    \;\le\; 1 + \eta.
\end{equation}
Thus, under differentiability, the local near-identity condition in
Assumption~\ref{ass:fidelity} is equivalent to the Jacobian-based
inequalities
\begin{equation}
    1 - \eta
    \;\le\;
    \sigma_{\min}\big(J_{R}(\Phi(n))\big|_{\mathcal{V}_n}\big)
    \le
    \sigma_{\max}\big(J_{R}(\Phi(n))\big|_{\mathcal{V}_n}\big)
    \;\le\; 1 + \eta,
\end{equation}
and we will use the lower bound
$\sigma_{\min}\big(J_{R}(\Phi(n))\big|_{\mathcal{V}_n}\big)\ge 1-\eta$
in the amplification analysis.

\paragraph{Chain rule and Jacobian factorization.}
Let $B:\mathbb{R}^D\to\mathbb{R}^d$ denote the bottleneck,
$\Psi:\mathbb{R}^d\to\mathbb{R}^D$ the decoder, and
$R = \Psi\circ B$ the feature reconstruction map. By the multivariate
chain rule, for any $z\in\mathbb{R}^D$,
\begin{equation}
    J_R(z) = J_\Psi\big(B(z)\big)\, J_B(z),
\end{equation}
where $J_B(z)\in\mathbb{R}^{d\times D}$ and
$J_\Psi(B(z))\in\mathbb{R}^{D\times d}$ are the Jacobian matrices of
$B$ and $\Psi$, respectively. In particular, at $z = \Phi(n)$ we obtain
\begin{equation}
    J_R(\Phi(n)) = J_\Psi\big(B(\Phi(n))\big)\, J_B\big(\Phi(n)\big).
\end{equation}

Restricting this identity to the subspace $\mathcal{V}_n$ yields
\begin{equation}
    J_R(\Phi(n))\big|_{\mathcal{V}_n}
    = J_\Psi\big(B(\Phi(n))\big)\,
      \Big(J_B\big(\Phi(n)\big)\big|_{\mathcal{V}_n}\Big),
\end{equation}
viewed as a linear operator from $\mathcal{V}_n$ to $\mathbb{R}^D$.
Concretely, if we fix an orthonormal basis of $\mathcal{V}_n$ and
identify $\mathcal{V}_n$ with $\mathbb{R}^k$ where
$k = \dim(\mathcal{V}_n)$, then $J_B(\Phi(n))|_{\mathcal{V}_n}$ can be
represented as a matrix in $\mathbb{R}^{d\times k}$ and
$J_R(\Phi(n))|_{\mathcal{V}_n}$ as a matrix in $\mathbb{R}^{D\times k}$,
with
\begin{equation}
    J_R(\Phi(n))|_{\mathcal{V}_n}
    =
    J_\Psi\big(B(\Phi(n))\big)\,
    J_B\big(\Phi(n)\big)|_{\mathcal{V}_n}.
\end{equation}
This identification justifies the use of standard singular value
inequalities for matrix products below.

\paragraph{A singular value inequality.}
We use the following standard inequality for singular values of matrix
products.

\begin{lemma}[singular value inequality]\label{lem:sv_ineq}
Let $A\in\mathbb{R}^{p\times q}$ and $B\in\mathbb{R}^{q\times r}$.
Then
\begin{equation}
    \sigma_{\min}(AB) \;\le\;
    \sigma_{\max}(A)\, \sigma_{\min}(B).
\end{equation}
\end{lemma}

\begin{proof}
By definition,
\begin{equation}
    \sigma_{\min}(B) = \min_{\|x\|=1} \|B x\|.
\end{equation}
Thus there exists a unit vector $x_0$ such that
$\|B x_0\| = \sigma_{\min}(B)$. Then
\begin{equation}
    \|A B x_0\|
    \;\le\; \|A\|_{\mathrm{op}} \cdot \|B x_0\|
    = \sigma_{\max}(A)\, \sigma_{\min}(B).
\end{equation}
On the other hand,
$\sigma_{\min}(AB) = \min_{\|x\|=1} \|ABx\| \le \|AB x_0\|$.
Combining the two inequalities yields
\begin{equation}
    \sigma_{\min}(AB)
    \;\le\; \sigma_{\max}(A)\, \sigma_{\min}(B),
\end{equation}
as claimed.
\end{proof}

\paragraph{Proof of Lemma~\ref{lem:amplification}.}
We are now ready to prove the decoder amplification result. For an
anomaly-free patch $n\in\mathcal{D}^{\mathrm{patch}}_{\mathrm{norm}}$,
recall that $R = \Psi\circ B$ and
\begin{equation}
    J_R(\Phi(n))\big|_{\mathcal{V}_n}
    = J_\Psi\big(B(\Phi(n))\big)\,
      \Big(J_B\big(\Phi(n)\big)\big|_{\mathcal{V}_n}\Big).
\end{equation}
Applying Lemma~\ref{lem:sv_ineq} with
\begin{equation}
    A = J_\Psi\big(B(\Phi(n))\big),
    \qquad
    B = J_B\big(\Phi(n)\big)\big|_{\mathcal{V}_n},
\end{equation}
we obtain
\begin{equation}
    \sigma_{\min}\big(J_R(\Phi(n))\big|_{\mathcal{V}_n}\big)
    = \sigma_{\min}(AB)
    \;\le\;
    \sigma_{\max}\big(J_\Psi(B(\Phi(n)))\big)\,
    \sigma_{\min}\big(J_B(\Phi(n))\big|_{\mathcal{V}_n}\big).
\end{equation}
By Assumption~\ref{ass:fidelity} in its Jacobian form,
\begin{equation}
    \sigma_{\min}\big(J_R(\Phi(n))\big|_{\mathcal{V}_n}\big)
    \;\ge\; 1 - \eta.
\end{equation}
Combining the two inequalities gives
\begin{equation}
    1 - \eta
    \;\le\;
    \sigma_{\max}\big(J_\Psi(B(\Phi(n)))\big)\,
    \sigma_{\min}\big(J_B(\Phi(n))\big|_{\mathcal{V}_n}\big).
\end{equation}
Assuming
$\sigma_{\min}\big(J_B(\Phi(n))\big|_{\mathcal{V}_n}\big) > 0$, we can
divide both sides by this quantity to obtain
\begin{equation}
    \sigma_{\max}\big(J_\Psi(B(\Phi(n)))\big)
    \;\ge\;
    \frac{1-\eta}{
    \sigma_{\min}\big(J_B(\Phi(n))\big|_{\mathcal{V}_n}\big)}.
    \label{eq:amplification_bound}
\end{equation}
This is exactly the statement of Lemma~\ref{lem:amplification}.

\begin{remark}[Compatibility with strong compression]
The condition
$\sigma_{\min}\big(J_B(\Phi(n))\big|_{\mathcal{V}_n}\big)>0$ means that
the bottleneck Jacobian $J_B(\Phi(n))$ is injective on $\mathcal{V}_n$:
no benign direction in $\mathcal{V}_n$ is completely collapsed. When
$\sigma_{\min}\big(J_B(\Phi(n))\big|_{\mathcal{V}_n}\big)=0$, the
singular value inequality and the fidelity assumption become
incompatible:
\begin{equation}
    \sigma_{\min}\big(J_R(\Phi(n))\big|_{\mathcal{V}_n}\big)
    \;\le\;
    \sigma_{\max}\big(J_\Psi(B(\Phi(n)))\big)\,
    \sigma_{\min}\big(J_B(\Phi(n))\big|_{\mathcal{V}_n}\big)
    \;=\; 0,
\end{equation}
while Assumption~\ref{ass:fidelity} requires
$\sigma_{\min}(J_R(\Phi(n))|_{\mathcal{V}_n})\ge 1-\eta>0$. Thus, a
bottleneck that completely destroys some benign directions cannot
satisfy the local fidelity requirement. In the regime where
$0<\sigma_{\min}(J_B(\Phi(n))|_{\mathcal{V}_n})\ll 1$, the lower bound
\eqref{eq:amplification_bound} shows that the decoder Jacobian must
exhibit large gain along those directions, realizing the
fidelity--stability amplification effect discussed in
Sec.~\ref{sec:theory}.
\end{remark}

\medskip

In summary, Assumption~\ref{ass:fidelity} enforces local near-identity
reconstruction along benign feature directions $\mathcal{V}_n$, which
forces the decoder Jacobian to have large gain whenever the bottleneck
$B$ is strongly information losing in a local differential sense (i.e.,
when $\sigma_{\min}(J_B(\Phi(n))|_{\mathcal{V}_n})$ is small but
positive). Combined with the fact that residual scores on anomaly-free
patches are governed by the reconstruction error
$\varepsilon_{\mathrm{approx}}$, this amplification behaviour underpins the
fidelity--stability dilemma discussed in Sec.~\ref{sec:theory}.

\section{Proofs of Retrieval-based Analysis}
\label{sec:proof-retrieval}

In this section we provide formal derivations for the claims in
Sec.~\ref{sec:why-retrieval} about the fidelity and stability properties
of retrieval-based anomaly scoring. Throughout, we fix a norm
$\|\cdot\|$ on $\mathbb{R}^D$, and let
$\gamma\subset \mathbb{R}^D$ denote the finite, non-empty set of stored
anomaly-free embeddings (the memory bank). The retrieval-based anomaly
score is the distance-to-memory function
\begin{equation}
    d_\gamma(z)
    \;\triangleq\;
    \min_{u\in\gamma} \|z-u\|,
    \qquad
    S_{\mathrm{ret}}(z) \;\triangleq\; d_\gamma(z),
    \qquad z\in\mathbb{R}^D.
\end{equation}
In particular, for each stored normal embedding $z_i\in\gamma$ we
have $S_{\mathrm{ret}}(z_i) = 0$, since $z_i$ is its own nearest
neighbour.

\paragraph{Empirical saturation on anomaly-free embeddings.}
To formalize the statement that retrieval cannot be strictly improved
upon on the anomaly-free embeddings by any nonnegative scoring rule, we
introduce a simple empirical functional. For any anomaly score
$S:\mathbb{R}^D\to [0,\infty)$, define its empirical normal-score on
$\gamma$ as
\begin{equation}
    \mathcal{E}(S)
    \;\triangleq\;
    \frac{1}{|\gamma|}
    \sum_{z\in\gamma} S(z).
\end{equation}
This quantity can be viewed as a basic notion of empirical
approximation error on the anomaly-free embeddings, or equivalently as
the empirical risk for the pointwise loss $\ell(S(z),0) = S(z)$ with
ground-truth normal label $0$.

\begin{proposition}[Empirical saturation of retrieval]
\label{prop:retrieval-optimal}
Let $\gamma\subset\mathbb{R}^D$ be finite and non-empty, and let
$S:\mathbb{R}^D\to[0,\infty)$ be any nonnegative anomaly score.
Then
\begin{equation}
    \mathcal{E}(S_{\mathrm{ret}})
    = 0
    \;\le\;
    \mathcal{E}(S),
\end{equation}
with equality $\mathcal{E}(S) = 0$ if and only if
$S(z) = 0$ for all $z\in\gamma$.
In particular, no such $S$ can achieve strictly smaller empirical
approximation error on $\gamma$ than $S_{\mathrm{ret}}$, and for every
$z\in\gamma$ we have
\begin{equation}
    S_{\mathrm{ret}}(z)
    \;=\; 0
    \;\le\;
    S(z).
\end{equation}
\end{proposition}

\begin{proof}
By definition of $S_{\mathrm{ret}}$ and the fact that
$z\in\gamma$ is an admissible neighbour for itself, we have
$S_{\mathrm{ret}}(z)=0$ for every $z\in\gamma$, and hence
\begin{equation}
    \mathcal{E}(S_{\mathrm{ret}})
    = \frac{1}{|\gamma|}
      \sum_{z\in\gamma} S_{\mathrm{ret}}(z)
    = 0.
\end{equation}
For any other nonnegative score $S$, each summand in $\mathcal{E}(S)$
satisfies $S(z)\ge 0$, so
\begin{equation}
    \mathcal{E}(S)
    = \frac{1}{|\gamma|}
      \sum_{z\in\gamma} S(z)
    \;\ge\; 0.
\end{equation}
If $\mathcal{E}(S) = 0$, then each summand must satisfy $S(z)=0$, which
gives the “if and only if” statement. The pointwise inequality
$S_{\mathrm{ret}}(z)\le S(z)$ on $\gamma$ follows immediately from
$S_{\mathrm{ret}}(z)=0$ and nonnegativity of $S(z)$.
\end{proof}

Proposition~\ref{prop:retrieval-optimal} formalizes the first fidelity
property used in Sec.~\ref{sec:why-retrieval}: retrieval assigns zero
score to every stored anomaly-free embedding. Thus, on the empirical
memory support, the distance-to-memory score attains the smallest
possible normal-score among all nonnegative anomaly scoring rules.
In contrast, a learned bottlenecked reconstruction map may assign
non-zero residuals to stored normal embeddings when its inverse is
imperfect.

\paragraph{Pointwise maximality within stable scores.}
We next formalize the stable score class used in
Sec.~\ref{sec:why-retrieval}. Let
\begin{equation}
\mathcal{F}_\gamma :=
\Big\{
S:\mathbb{R}^D\to[0,\infty)
\ \big|\
S(u)=0,\ \forall u\in\gamma,\ 
|S(z)-S(z')|\le \|z-z'\|,\ \forall z,z'
\Big\}.
\end{equation}
That is, $\mathcal{F}_\gamma$ contains nonnegative, $1$-Lipschitz
feature-space anomaly scores that vanish on the stored anomaly-free
memory. The following proposition shows that $S_{\mathrm{ret}}$ is the
largest member of this class.

\begin{proposition}[Pointwise maximality of retrieval]
\label{prop:canonical-distance}
For any $S\in\mathcal{F}_\gamma$ and every $x\in\mathbb{R}^D$,
\begin{equation}
    S(x)
    \;\le\;
    d_\gamma(x)
    \;=\;
    S_{\mathrm{ret}}(x).
\end{equation}
In particular, among all nonnegative $1$-Lipschitz anomaly scores that
vanish on $\gamma$, the retrieval score $S_{\mathrm{ret}}$ is
pointwise maximal.
\end{proposition}

\begin{proof}
Fix $x\in\mathbb{R}^D$ and choose
$a_x\in\arg\min_{u\in\gamma}\|x-u\|$, so that
$d_\gamma(x)=\|x-a_x\|$.
Since $a_x\in\gamma$ and $S$ vanishes on $\gamma$, we have
$S(a_x)=0$.
By the $1$-Lipschitz property of $S$,
\begin{equation}
    |S(x)-S(a_x)|
    \;\le\;
    \|x-a_x\|.
\end{equation}
Using $S(a_x)=0$ and $S(x)\ge 0$, this gives
\begin{equation}
    S(x)
    =
    S(x)-S(a_x)
    \;\le\;
    \|x-a_x\|
    =
    d_\gamma(x)
    =
    S_{\mathrm{ret}}(x).
\end{equation}
This proves the claim.
\end{proof}

Proposition~\ref{prop:canonical-distance} formalizes the pointwise
maximality statement used in Sec.~\ref{sec:why-retrieval}: once the
stored anomaly-free memory $\gamma$ is fixed, the distance-to-memory
score is the largest nonnegative $1$-Lipschitz score that still assigns
zero score to all stored normal embeddings.


\paragraph{Expected anomaly--normal separation.}
The pointwise maximality above also gives a simple guarantee for the
expected score gap between anomalous and normal features. Let $\pi$ be a
distribution over anomalous--normal feature pairs $(A,N)$, and define
\begin{equation}
    J_\pi(S)
    \;\triangleq\;
    \mathbb{E}_{(A,N)\sim\pi}\big[S(A)-S(N)\big].
\end{equation}

\begin{theorem}[Near-optimal expected separation]
\label{thm:expected-separation}
Let $\pi$ be a distribution over anomalous--normal feature pairs $(A,N)$
such that
\begin{equation}
    \mathbb{E}_{(A,N)\sim\pi}\big[d_\gamma(A)+d_\gamma(N)\big] < \infty.
\end{equation}
For $S_{\mathrm{ret}}=d_\gamma$ and the stable score class
$\mathcal{F}_\gamma$ defined above,
\begin{equation}
    \sup_{S\in\mathcal{F}_\gamma} J_\pi(S)
    -
    J_\pi(S_{\mathrm{ret}})
    \;\le\;
    \mathbb{E}_{(A,N)\sim\pi}\big[d_\gamma(N)\big].
\end{equation}
In particular, if $d_\gamma(N)=0$ almost surely, then
\begin{equation}
    J_\pi(S_{\mathrm{ret}})
    =
    \sup_{S\in\mathcal{F}_\gamma} J_\pi(S).
\end{equation}
\end{theorem}

\begin{proof}
For any $S\in\mathcal{F}_\gamma$, Proposition~\ref{prop:canonical-distance}
gives $S(A)\le d_\gamma(A)$, while nonnegativity gives $S(N)\ge 0$.
Therefore,
\begin{equation}
    J_\pi(S)
    =
    \mathbb{E}_\pi[S(A)-S(N)]
    \le
    \mathbb{E}_\pi[d_\gamma(A)].
\end{equation}
For retrieval,
\begin{equation}
    J_\pi(S_{\mathrm{ret}})
    =
    \mathbb{E}_\pi[d_\gamma(A)-d_\gamma(N)].
\end{equation}
Thus, for every $S\in\mathcal{F}_\gamma$,
\begin{equation}
    J_\pi(S)-J_\pi(S_{\mathrm{ret}})
    \le
    \mathbb{E}_\pi[d_\gamma(N)].
\end{equation}
Taking the supremum over $S\in\mathcal{F}_\gamma$ proves the first claim.

If $d_\gamma(N)=0$ almost surely, then the right-hand side is zero, and hence
\begin{equation}
    \sup_{S\in\mathcal{F}_\gamma} J_\pi(S)
    \le
    J_\pi(S_{\mathrm{ret}}).
\end{equation}
Moreover, $S_{\mathrm{ret}}=d_\gamma$ belongs to $\mathcal{F}_\gamma$:
it is nonnegative, vanishes on $\gamma$, and is $1$-Lipschitz by
Lemma~\ref{lem:nonexpansive}. Therefore,
\begin{equation}
    \sup_{S\in\mathcal{F}_\gamma} J_\pi(S)
    \ge
    J_\pi(S_{\mathrm{ret}}).
\end{equation}
Combining the two inequalities yields
$J_\pi(S_{\mathrm{ret}})
=
\sup_{S\in\mathcal{F}_\gamma} J_\pi(S)$.
\end{proof}

Theorem~\ref{thm:expected-separation} is the formal version of the
near-optimality statement in Sec.~\ref{sec:why-retrieval}. It shows that
the gap between retrieval and the best score in the same stable score
class is controlled by the distance from unseen normal features to the
stored anomaly-free memory. Therefore, when normal test features are
well covered by the memory, retrieval is near-optimal within
$\mathcal{F}_\gamma$ for expected anomaly--normal separation.

\paragraph{Non-expansiveness of the retrieval score.}
We now give a complete proof of Lemma~\ref{lem:nonexpansive}, which
states that $S_{\mathrm{ret}}$ is $1$-Lipschitz in feature space and
therefore cannot amplify benign perturbations in encoder features.

\begin{lemma}[Non-expansiveness, restated]
For any $u,v \in \mathbb{R}^D$,
\begin{equation}
    \big|S_{\mathrm{ret}}(u) - S_{\mathrm{ret}}(v)\big|
    \;\le\; \|u - v\|.
\end{equation}
In particular, $S_{\mathrm{ret}}$ is a $1$-Lipschitz map from
$(\mathbb{R}^D,\|\cdot\|)$ to $(\mathbb{R},|\cdot|)$.
\end{lemma}

\begin{proof}
Fix $u,v\in\mathbb{R}^D$. Let
$a_v \in \arg\min_{a\in\gamma} \|v-a\|$ be a nearest neighbour of
$v$ in $\gamma$, so that
\begin{equation}
    S_{\mathrm{ret}}(v) = \|v - a_v\|.
\end{equation}
By definition of $S_{\mathrm{ret}}(u)$ and the triangle inequality,
\begin{equation}
    S_{\mathrm{ret}}(u)
    \;=\; \min_{a\in\gamma} \|u-a\|
    \;\le\; \|u - a_v\|
    \;\le\; \|u-v\| + \|v - a_v\|
    \;=\; \|u-v\| + S_{\mathrm{ret}}(v).
\end{equation}
Rearranging gives
\begin{equation}
    S_{\mathrm{ret}}(u) - S_{\mathrm{ret}}(v)
    \;\le\; \|u-v\|.
\end{equation}
Exchanging the roles of $u$ and $v$ and repeating the same argument
yields
\begin{equation}
    S_{\mathrm{ret}}(v) - S_{\mathrm{ret}}(u)
    \;\le\; \|u-v\|.
\end{equation}
Combining the two inequalities gives
\begin{equation}
    |S_{\mathrm{ret}}(u) - S_{\mathrm{ret}}(v)|
    \;\le\; \|u-v\|,
\end{equation}
which is exactly the desired non-expansiveness property.
\end{proof}

\paragraph{An alternative Lipschitz viewpoint.}
For completeness, we note that the non-expansiveness of
$S_{\mathrm{ret}}$ can also be seen directly from its distance-to-memory
form. For each anchor $a\in\gamma$, the map
\begin{equation}
    f_a(z) \;\triangleq\; \|z-a\|
\end{equation}
is $1$-Lipschitz in $z$ by the triangle inequality. Since $\gamma$ is
finite and
\begin{equation}
    S_{\mathrm{ret}}(z)
    = \min_{a\in\gamma} f_a(z),
\end{equation}
$S_{\mathrm{ret}}$ is the pointwise minimum of a finite family of
$1$-Lipschitz functions. Therefore, for any $u,v\in\mathbb{R}^D$,
the same nearest-neighbour argument as above yields
\begin{equation}
    |S_{\mathrm{ret}}(u)-S_{\mathrm{ret}}(v)|
    \le \|u-v\|,
\end{equation}
which again implies Lemma~\ref{lem:nonexpansive}.

\medskip


Taken together, Proposition~\ref{prop:retrieval-optimal},
Proposition~\ref{prop:canonical-distance},
Theorem~\ref{thm:expected-separation}, and
Lemma~\ref{lem:nonexpansive} formalize the fidelity and stability
properties of retrieval-based scoring highlighted in
Sec.~\ref{sec:why-retrieval}. Retrieval assigns zero score to the stored
anomaly-free embeddings, is pointwise maximal within the stable score
class $\mathcal{F}_\gamma$, and is near-optimal for expected
anomaly--normal separation when unseen normal features are close to the
memory support. At the same time, its distance-to-memory form is
non-expansive in feature space, which prevents score amplification under
benign feature perturbations. These properties support the use of RAD as
a way to preserve empirical fidelity while controlling score stability.

\section{Additional Few-shot Results}
\label{sec:few-shot}

Table~\ref{tab:mvtec_only} complements Fig.~\ref{fig:fewshot_scaling} by reporting absolute pixel-level AUROC/AUPRO under the same multi-class few-shot MVTec-AD protocol. Across all shot budgets, RAD achieves the best pixel performance: compared to the strongest few-shot baseline IIPAD, it yields consistently higher AUROC and about 2--3 points gain in AUPRO for 1-, 2-, and 4-shot settings, while maintaining even larger margins over both classical UAD baselines adapted to the few-shot regime (SPADE, PatchCore) and recent few-shot AD methods (WinCLIP, PromptAD) that are explicitly designed and trained for this scenario. Among training-free approaches, RAD slightly improves over UniVAD in AUROC at 1-shot and additionally provides substantially higher AUPRO, despite using a much smaller backbone. These results corroborate that a training-free RAD is competitive with, and often superior to, specialized few-shot AD pipelines even in the strict multi-class few-shot setting.

\begin{table}[H]
  \centering
  \caption{Quantitative comparisons of pixel-level anomaly detection in AUROC and AUPRO on MVTec-AD under the multi-class few-shot setting. 
           The \best{best} and \secondbest{second best} pixel performances are in red and blue colors, respectively.
           $\dag$ denotes methods tailored for few-shot AD.}
  \label{tab:mvtec_only}
  \setlength{\tabcolsep}{4pt}
  \renewcommand{\arraystretch}{1.1}
  \begin{tabular}{l l c ccc}
    \toprule
    \multirow{2}{*}{Method} & \multirow{2}{*}{Backbone(\#Params)} &
    \multirow{2}{*}{Training-free} &
    \multicolumn{3}{c}{MVTec-AD} \\
    & & & 1-shot & 2-shot & 4-shot \\
    \midrule
    SPADE \tiny \textit{Arxiv'20}          & ResNet-50 (69M)          & $\times$ &
      60.4/53.1        & 61.2/54.7           & 62.8/55.6 \\
    WinCLIP$^\dag$ \tiny \textit{CVPR'23}         & CLIP-B (188M)            & $\times$ &
      92.4/83.5        & 92.4/83.9               & 92.9/84.4 \\
    PromptAD$^\dag$ \tiny \textit{CVPR'24}    & CLIP-L (300M)  & $\times$ &
      91.8/83.6 & 92.2/84.3 & 92.4/84.6 \\
    IIPAD$^\dag$  \tiny \textit{ICLR'25}        & CLIP-L (300M)       & $\times$ &
      96.4/\secondbest{89.8}        & \secondbest{96.7}/\secondbest{90.3}           & \secondbest{97.0}/\secondbest{91.2} \\
    \midrule
    PatchCore \tiny \textit{CVPR'22}     & ResNet-50 (69M)  & \checkmark &
      83.9/72.7 & 89.6/74.2          & 92.6/80.8 \\
    UniVAD$^\dag$  \tiny \textit{CVPR'25}        & DINOv2-G (1.1B) + CLIP-L (300M)          & \checkmark &
          \secondbest{96.5}/- & -/-           & -/- \\
    \rowcolor{cyan!10}
    \textbf{RAD (Ours)}  & DINOv3-B (86M)            & \checkmark &
      \best{96.7}/\best{91.9} &
      \best{97.3}/\best{93.0} &
      \best{97.7}/\best{94.0} \\
    \bottomrule
  \end{tabular}
\end{table}

\section{Impact of Encoder Quality and Resolution}
\label{sec:encoder}

\begin{table*}[t]
\centering
\caption{Comparison between pre-trained ViT foundations, conducted on MVTec-AD (\%). All models are ViT-Base. The patch size of DINOv2 is $14^2$; others are $16^2$. $R512^2$-$C448^2$ represents first resizing images to $512 \times 512$, then center cropping to $448 \times 448$. ``IN-1K top-1'' is ImageNet-1K top-1 accuracy at $224^2$ resolution from the original encoder papers (DeiT: supervised; others: self-supervised pre-training followed by standard fine-tuning, except DINOv2 which is linear probing).}
\label{tab:pretrain_backbone}
\resizebox{\textwidth}{!}{
\setlength{\aboverulesep}{0pt}
\setlength{\belowrulesep}{0pt}
\begin{tabular}{lcccccccccc}
\toprule
\multirow{2}{*}{\shortstack[l]{Pre-Train\\Backbone}} &
\multirow{2}{*}{Type} &
\multirow{2}{*}{\shortstack{Image\\Size}} &
\multicolumn{3}{c}{Image-level} &
\multicolumn{4}{c}{Pixel-level} &
\multirow{2}{*}{\shortstack{IN-1K\\top-1}} \\
\cmidrule(lr){4-6} \cmidrule(lr){7-10}
& & & AUROC & AP & $F_{1}$-max & AUROC & AP & $F_{1}$-max & AUPRO & \\
\midrule
DeiT   & Supervised & R$512^2$-C$448^2$ & 98.0 & 99.2 & 97.7 & 97.0 & 70.1 & 68.7 & 91.3 & 81.8 \\
MAE    & MIM        & R$512^2$-C$448^2$ & 88.8 & 95.2 & 90.7 & 93.4 & 54.5 & 53.2 & 82.3 & 83.6 \\
D-iGPT & MIM        & R$512^2$-C$448^2$ & 99.1 & 99.6 & 98.5 & 97.9 & 67.7 & 67.5 & 92.7 & 86.2 \\
MoCoV3 & CL         & R$512^2$-C$448^2$ & 99.1 & 99.7 & 98.1 & 97.9 & 72.9 & 71.1 & 93.3 & 83.2 \\
DINO   & CL         & R$512^2$-C$448^2$ & 99.1 & 99.7 & 98.4 & 97.9 & 73.8 & 71.0 & 94.0 & 82.8 \\
iBOT   & CL+MIM     & R$512^2$-C$448^2$ & 99.2 & 99.7 & 98.6 & 98.4 & 74.1 & 72.3 & 94.5 & 84.0 \\
DINOv2 & CL+MIM     & R$512^2$-C$448^2$ & 99.7 & 99.9 & 99.5 & 98.1 & 70.0 & 68.7 & 95.6 & 84.5 \\
\rowcolor{cyan!10}
DINOv3 & CL+MIM     & R$512^2$-C$448^2$ & 99.6 & 99.8 & 99.0 & 98.5 & 75.6 & 71.3 & 94.9 & -- \\
\midrule
DeiT   & Supervised & R$256^2$-C$224^2$ & 95.1 & 98.1 & 96.9 & 96.4 & 63.8 & 63.5 & 87.7 & 81.8 \\
MAE    & MIM        & R$256^2$-C$224^2$ & 91.1 & 96.2 & 92.4 & 95.8 & 59.4 & 58.1 & 84.3 & 83.6 \\
BEiT   & MIM        & R$256^2$-C$224^2$ & 90.3 & 96.0 & 93.7 & 88.2 & 45.4 & 49.6 & 74.7 & 83.2 \\
BEiTv2 & MIM        & R$256^2$-C$224^2$ & 96.5 & 98.4 & 97.2 & 96.3 & 63.2 & 62.8 & 89.5 & 85.5 \\
D-iGPT & MIM        & R$256^2$-C$224^2$ & 99.2 & 99.6 & 98.4 & 97.6 & 61.9 & 63.3 & 91.6 & 86.2 \\
MoCoV3 & CL         & R$256^2$-C$224^2$ & 97.2 & 99.0 & 97.1 & 96.6 & 65.7 & 65.0 & 88.6 & 83.2 \\
DINO   & CL         & R$256^2$-C$224^2$ & 97.9 & 99.3 & 98.0 & 97.1 & 67.5 & 66.1 & 91.0 & 82.8 \\
iBOT   & CL+MIM     & R$256^2$-C$224^2$ & 98.2 & 99.3 & 98.0 & 97.5 & 67.5 & 66.6 & 91.3 & 84.0 \\
DINOv2 & CL+MIM     & R$256^2$-C$224^2$ & 99.2 & 99.7 & 98.4 & 97.8 & 64.8 & 64.9 & 93.2 & 84.5 \\
\rowcolor{cyan!10}
DINOv3 & CL+MIM     & R$256^2$-C$224^2$ & 99.5 & 99.8 & 98.8 & 98.2 & 69.2 & 66.5 & 94.0 & -- \\
\bottomrule
\end{tabular}
}
\end{table*}

\paragraph{Additional analysis.}
Table~\ref{tab:pretrain_backbone} also reports the ImageNet-1K top-1 accuracy (IN-1K top-1) for each foundation, taken from the original encoder papers at $224^2$ resolution (DeiT: supervised; others: self-supervised pre-training followed by standard fine-tuning, except DINOv2 which is evaluated with a linear probe). Using IN-1K top-1 as a proxy for generic representation strength, we observe a clear positive trend with RAD's localization quality. Focusing on the R$256^2$-C$224^2$ setting and the primary localization metric P-AUPRO, the encoders with available IN-1K numbers yield a Pearson correlation of of approximately $0.35$ and a Spearman rank correlation of $\rho \approx 0.60$ between IN-1K top-1 and P-AUPRO. If we exclude the BEiT family (BEiT/BEiTv2), which forms a clear outlier group with relatively strong ImageNet classification but notably weaker anomaly localization, the Pearson correlation further increases to approximately $0.47$ and the Spearman rank correlation to $\rho \approx 0.71$. These statistics quantitatively support the qualitative hierarchy discussed in the main text: RAD generally benefits from stronger ImageNet representations, while tokenizer-based MIM objectives such as BEiT do not automatically translate into superior anomaly maps.

\paragraph{Effect of resolution.}
Table~\ref{tab:pretrain_backbone} also allows us to isolate the impact of input resolution on RAD while keeping the encoder fixed. When increasing the resolution from R$256^2$-C$224^2$ to R$512^2$-C$448^2$, P-AUPRO improves for almost all foundations: the average gain across backbones is about $2.1\uparrow$ points, with the largest improvement observed for MoCoV3 ($4.7\uparrow$). The only exception is MAE, whose P-AUPRO slightly decreases at the higher resolution, suggesting that its patch-wise reconstruction objective is already well aligned with the lower-resolution inputs. Importantly, the relative ranking of encoders is largely preserved across resolutions: contrastive and hybrid encoders (DINOv3, DINOv2, iBOT, DINO, D-iGPT, MoCoV3) consistently occupy the top of the hierarchy, followed by supervised DeiT and finally the MIM-only encoders. Overall, these appendix results confirm that RAD's anomaly localization performance is primarily governed by the encoder's ImageNet representation quality, while increasing the input resolution provides a mostly uniform boost without altering this dependence.

\section{Results and Analysis on the 3D-ADAM Dataset}
\label{sec:adam}

Although 3D-ADAM provides both RGB and 3D signals, our evaluation uses only the RGB modality as a preliminary test for RAD under larger geometric, viewpoint, and appearance variations.
We apply the official background filtering so that pixel-level evaluation focuses on the foreground objects rather than irrelevant background regions.
All methods are evaluated under the multi-class standard setting, and Dinomaly is reported with a stronger DINOv3 backbone for a fairer comparison.

As shown in Table~\ref{tab:3dadam}, RAD achieves the best performance on all seven metrics.
Compared with PatchCore, RAD improves by \textbf{10.5$\uparrow$/6.6$\uparrow$/4.5$\uparrow$} on \textbf{I-AUROC/I-AP/I-$F_1$-max} and by \textbf{3.2$\uparrow$/18.6$\uparrow$/17.3$\uparrow$/15.4$\uparrow$} on \textbf{P-AUROC/P-AP/P-$F_1$-max/P-AUPRO}.
Compared with Dinomaly using the DINOv3 backbone, RAD further improves by \textbf{6.5$\uparrow$/3.3$\uparrow$/2.8$\uparrow$} on image-level metrics and by \textbf{0.1$\uparrow$/2.3$\uparrow$/2.4$\uparrow$/0.5$\uparrow$} on pixel-level metrics.
These results indicate that RAD remains effective under the larger 3D-induced variations of 3D-ADAM.

\begin{table}[t]
  \centering
  \caption{Quantitative comparisons on the official background-filtered 3D-ADAM benchmark under the multi-class standard setting.
  All methods use only the RGB modality. Dinomaly is reported with a stronger DINOv3 backbone.
  The \best{best} and \secondbest{second best} performances are in red and blue colors, respectively.}
  \label{tab:3dadam}

  \begingroup
  \footnotesize
  \setlength{\tabcolsep}{2.5pt}
  \renewcommand{\arraystretch}{0.95}

  \resizebox{0.88\linewidth}{!}{%
  \begin{tabular}{lccccccc}
    \toprule
    \multirow{2}{*}{Method} &
    \multicolumn{3}{c}{Image-level} &
    \multicolumn{4}{c}{Pixel-level} \\
    \cmidrule(lr){2-4} \cmidrule(lr){5-8}
    & I-AUROC & I-AP & I-$F_1$-max
    & P-AUROC & P-AP & P-$F_1$-max & P-AUPRO \\
    \midrule
    PatchCore {\scriptsize\textit{CVPR'22}}
    & 84.4 & 91.2 & 91.9
    & 95.7 & 23.5 & 29.1 & 78.8 \\
    Dinomaly (DINOv3) {\scriptsize\textit{CVPR'25}}
    & \secondbest{88.4} & \secondbest{94.5} & \secondbest{93.6}
    & \secondbest{98.8} & \secondbest{39.8} & \secondbest{44.0} & \secondbest{93.7} \\
    \midrule
    \rowcolor{cyan!10}
    \textbf{RAD (Ours)}
    & \best{94.9} & \best{97.8} & \best{96.4}
    & \best{98.9} & \best{42.1} & \best{46.4} & \best{94.2} \\
    \bottomrule
  \end{tabular}%
  }

  \endgroup
  \vspace{-2mm}
\end{table}

\paragraph{Analysis of pixel-level $F_1$.}
Although RAD achieves a high P-AUROC of \textbf{98.9} and improves P-$F_1$-max to \textbf{46.4}, pixel-level $F_1$ remains much lower than image-level $F_1$.
This is mainly due to the extreme pixel-level imbalance in 3D-ADAM: anomalous pixels account for only \textbf{0.61\%} of all evaluated pixels.
Under this imbalance, even a small high-score tail among normal pixels can produce many false positives in absolute count, reducing precision and therefore suppressing $F_1$.

To further examine this effect, we analyze the score distributions of normal and anomalous pixels.
As shown in Tab.~\ref{tab:3dadam_score_dist}, anomalous pixels have a much larger mean score than normal pixels, with \textbf{0.1156} versus \textbf{0.0096}.
Moreover, the average normal 95th percentile score is \textbf{0.0445}, which remains below the anomalous 5th percentile score of \textbf{0.0628}.
At 95\% TPR, the pixel-level FPR is \textbf{4.6\%}.
These results suggest that the remaining gap in pixel-level $F_1$ is mainly caused by extreme class imbalance and residual score-tail overlap, rather than collapsed pixel-level discrimination.

\begin{table}[h]
  \centering
  \caption{Pixel-score distribution analysis on 3D-ADAM. Despite the low anomaly-pixel ratio, normal and anomalous pixels remain well separated in score distribution.}
  \label{tab:3dadam_score_dist}
  \setlength{\tabcolsep}{6pt}
  \renewcommand{\arraystretch}{1.1}
  \begin{tabular}{lcc}
    \toprule
    Statistic & Normal pixels & Anomalous pixels \\
    \midrule
    Mean score & 0.0096 & 0.1156 \\
    Percentile statistic & 95th: 0.0445 & 5th: 0.0628 \\
    \bottomrule
  \end{tabular}
\end{table}

\begin{figure}[h]
    \centering
    \includegraphics[width=1\linewidth]{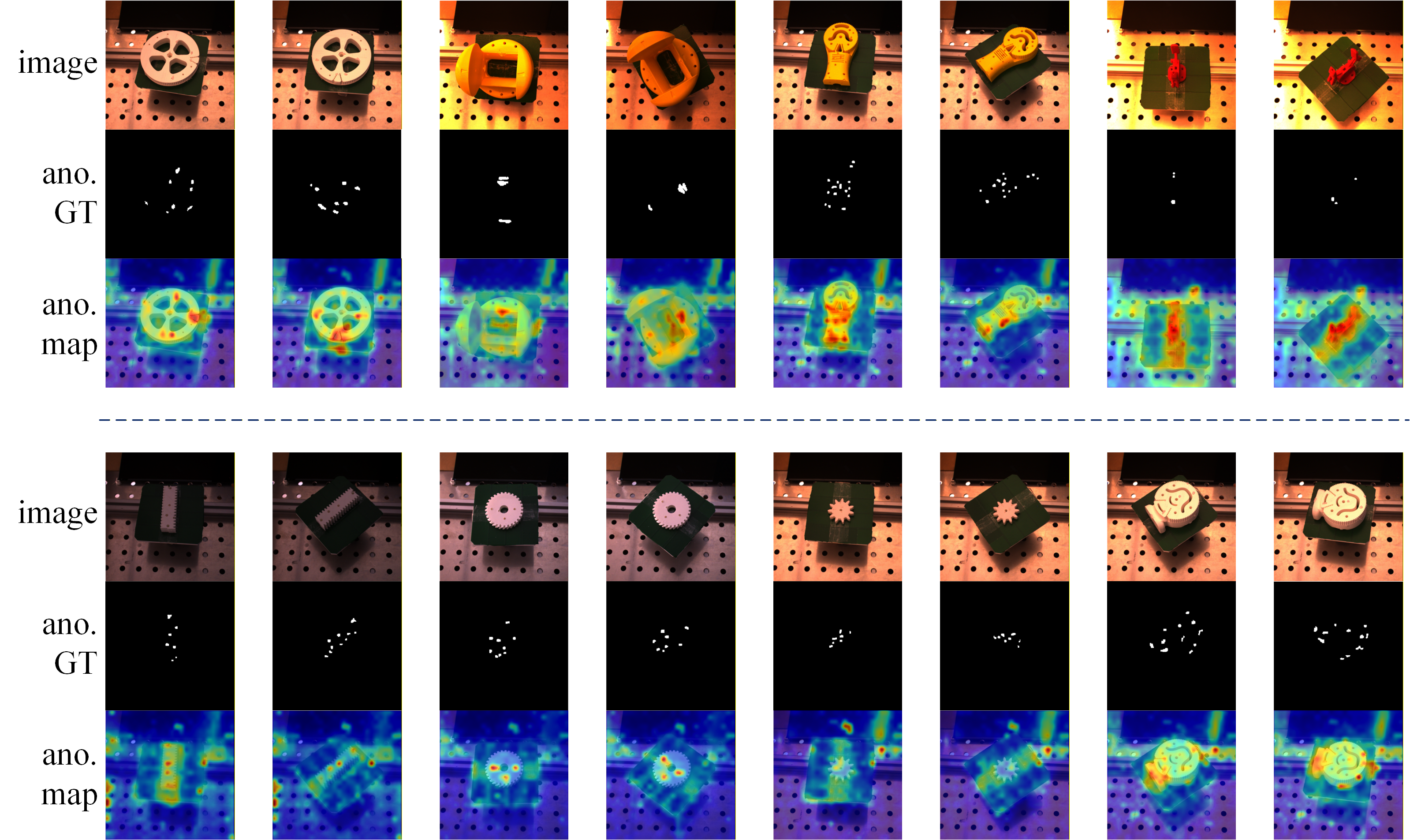}
    \caption{Anomaly map visualizations on 3D-ADAM using only the RGB modality. All samples are randomly chosen.}
    \label{fig:Figure_adam}
\end{figure}

Figure~\ref{fig:Figure_adam} presents qualitative anomaly maps on 3D-ADAM.
Consistent with the quantitative results, RAD can localize small missing components and subtle surface defects under large viewpoint and appearance changes.
At the same time, occasional activations still appear on visually complex background or supporting-board regions.
These false positives are partly caused by RGB appearance changes induced by viewpoint and illumination variations, which may make some normal regions appear far from the stored anomaly-free memory.

Overall, the 3D-ADAM results show that RAD remains effective under geometric and appearance variations, while the score-distribution analysis explains why pixel-level $F_1$ remains challenging despite high pixel AUROC.
Since the current evaluation uses only RGB, a promising future direction is to exploit the full 3D signals available in 3D-ADAM, such as depth, point clouds, or multi-view consistency, to better separate objects from background and suppress spurious responses on geometrically stable regions.

\section{Ablation Results}
\label{sec:ablation}

\begin{table}[h]
\centering
\caption{Ablations of RAD elements on MVTec-AD(\%).}
\label{tab:ablation}
\begin{tabular}{c c c r r}
\toprule
Multi-layer & Global Context Retrieval & Spatial Conditioning & P-AP & P-$F_1$-max \\
\midrule
          & \checkmark &           & 70.3 & 66.2 \\
          & \checkmark & \checkmark & 71.1 & 66.8 \\
\checkmark &          &            & 74.0 & 70.0 \\
\checkmark & \checkmark &           & 75.0 & 70.8 \\
\rowcolor{cyan!10}
\checkmark & \checkmark & \checkmark & \textbf{75.6} & \textbf{71.3} \\
\bottomrule
\end{tabular}
\end{table}

We conduct ablations on the three key design elements of RAD, namely multi-layer memory, global context retrieval, and
spatial conditioning, on MVTec-AD. Pixel-level results are summarized in Table~\ref{tab:ablation}. We focus on
P-AP and P-$F_1$-max since they directly reflect localization quality and other metrics reach saturation.

\paragraph{Effect of multi-layer memory.}
Introducing multi-layer memory yields a clear improvement. P-AP increases by $3.3\uparrow$ and P-$F_1$-max increases
by $3.4\uparrow$. This shows that aggregating features from multiple encoder layers substantially improves the robustness
of patch matching and strengthens localization of both subtle and large defects.

\paragraph{Complementarity of multi-layer memory and global retrieval.}
Starting from the multi-layer variant, adding global context retrieval brings further gains. P-AP increases by
$1.0\uparrow$ and P-$F_1$-max increases by $0.8\uparrow$. Even after exploiting rich multi-layer features, restricting
patch retrieval to a small set of globally compatible reference images still brings clear improvements, which shows that
the global retrieval stage effectively suppresses interference from semantically mismatched categories in the shared
memory bank.

\paragraph{Effect of spatial conditioning.}
Spatial conditioning provides consistent and non-trivial gains in different configurations. On the single-layer
configuration with global retrieval, adding spatial conditioning improves P-AP by $0.8\uparrow$ and P-$F_1$-max
by $0.6\uparrow$. On the strongest ablated variant that already combines multi-layer memory and global retrieval,
introducing spatial conditioning further increases P-AP by $0.6\uparrow$ and P-$F_1$-max by $0.5\uparrow$. These gains,
obtained on top of a strong task-specific training-free baseline, highlight the importance of local spatial coherence for avoiding
mismatches between different functional parts and for sharpening anomaly maps.

Overall, the full RAD configuration that combines all three components achieves the best localization performance.
Compared with the baseline that uses only global retrieval without multi-layer memory or spatial conditioning,
the full model achieves $5.3\uparrow$ P-AP and $5.1\uparrow$ P-$F_1$-max, and each component contributes complementary
improvements rather than redundant effects.

\section{Same-backbone Comparison with Dinomaly}
\label{sec:same-backbone}

To separate the contribution of retrieval from backbone strength, we compare RAD with Dinomaly under the same DINOv3-B backbone. As shown in Tab.~\ref{tab:same_backbone}, RAD still outperforms Dinomaly on both MVTec-AD~\cite{bergmann2019mvtec} and MVTec-AD2~\cite{heckler2026mvtec}. In particular, RAD improves P-AP by 2.7 on MVTec-AD and by 6.0 on MVTec-AD2, indicating that the retrieval design contributes beyond the use of a stronger frozen encoder.

\begin{table}[h]
  \centering
  \caption{Same-backbone comparison with Dinomaly using DINOv3-B.}
  \label{tab:same_backbone}
  \setlength{\tabcolsep}{4pt}
  \renewcommand{\arraystretch}{1.05}
  \begin{tabular}{lcccccc}
    \toprule
    \multirow{2}{*}{Method} &
    \multicolumn{3}{c}{MVTec-AD} &
    \multicolumn{3}{c}{MVTec-AD2} \\
    \cmidrule(lr){2-4} \cmidrule(lr){5-7}
    & P-AP & P-$F_1$ & P-AUPRO
    & P-AP & P-$F_1$ & P-AUPRO \\
    \midrule
    Dinomaly (DINOv3-B)
    & 72.9 & 70.3 & 94.7
    & 28.3 & 34.9 & 55.0 \\
    \rowcolor{cyan!10}
    RAD (DINOv3-B)
    & \textbf{75.6} & \textbf{71.3} & \textbf{94.9}
    & \textbf{34.3} & \textbf{39.6} & \textbf{57.2} \\
    \bottomrule
  \end{tabular}
\end{table}

\section{Hyperparameter Sensitivity Analysis}
\label{sec:sensitivity}

We further study the sensitivity of RAD to the number of globally retrieved references $K$, the spatial neighborhood radius $\rho$, and the layer-fusion weights $w_\ell$. As shown in Tab.~\ref{tab:mvtec_sensitivity}, RAD is generally insensitive to these hyperparameters. When sweeping $K$, the maximum variations on I-AUROC/I-$F_1$/P-AP/P-$F_1$ are only 1.0/1.3/2.7/2.0. When sweeping $\rho$, the variations are further limited to 0.7/0.7/0.7/0.8. For multi-layer fusion weights, excluding single-layer ablations, the maximum variations are 0.2/0.4/2.0/1.3. These small changes show that RAD is robust to the default choices of $K$, $\rho$, and $w_\ell$.

\begin{table}[h]
\centering
\caption{Sensitivity sweeps on MVTec-AD. The radius $\rho$ is defined on the discrete patch grid and therefore only takes integer values.}
\begin{tabular}{lcccc}
\toprule
Setting & I-AUROC & I-F1 & P-AP & P-F1 \\
\midrule
\multicolumn{5}{l}{\emph{Top-$K$ sweep} ($\rho=1$, $w_l=[1,1,1,1]$)} \\
$K=5$ & 98.6 & 97.7 & 72.9 & 69.3 \\
$K=10$ & 99.1 & 98.1 & 74.1 & 70.1 \\
$K=20$ & 99.3 & 98.4 & 74.8 & 70.7 \\
$K=50$ & 99.4 & 98.7 & 75.3 & 71.1 \\
$K=100$ & 99.5 & 98.8 & 75.5 & 71.2 \\
$\mathbf{K=150}$ & 99.6 & 99.0 & 75.6 & 71.3 \\
$K=200$ & 99.5 & 98.9 & 75.6 & 71.3 \\
\midrule
\multicolumn{5}{l}{\emph{Radius sweep} ($K=150$, $w_l=[1,1,1,1]$)} \\
$\rho=0$ & 98.9 & 98.3 & 74.9 & 70.5 \\
$\mathbf{\rho=1}$ & 99.6 & 99.0 & 75.6 & 71.3 \\
$\rho=2$ & 99.5 & 98.9 & 75.5 & 71.2 \\
$\rho=3$ & 99.5 & 98.9 & 75.5 & 71.2 \\
$\rho=4$ & 99.4 & 98.9 & 75.4 & 71.2 \\
$\rho=6$ & 99.3 & 98.8 & 75.3 & 71.1 \\
$\rho=8$ & 99.2 & 98.7 & 75.2 & 70.9 \\
$\rho=10$ & 98.9 & 98.4 & 74.9 & 70.6 \\
\midrule
\multicolumn{5}{l}{\emph{Layer-fusion sweep} ($K=150$, $\rho=1$)} \\
$w_l=[1,0,0,0]$ & 96.6 & 96.5 & 65.5 & 64.1 \\
$w_l=[0,1,0,0]$ & 99.2 & 98.4 & 71.0 & 68.6 \\
$w_l=[0,0,1,0]$ & 99.4 & 98.8 & 72.3 & 68.9 \\
$w_l=[0,0,0,1]$ & 99.5 & 99.0 & 71.2 & 66.9 \\
$\mathbf{w_l=[1,1,1,1]}$ & 99.6 & 99.0 & 75.6 & 71.3 \\
$w_l=[4,1,1,1]$ & 99.4 & 98.6 & 76.0 & 71.7 \\
$w_l=[1,4,1,1]$ & 99.4 & 98.7 & 74.4 & 70.8 \\
$w_l=[1,1,4,1]$ & 99.5 & 98.8 & 74.2 & 70.4 \\
$w_l=[1,1,1,4]$ & 99.6 & 99.0 & 75.2 & 70.5 \\
$w_l=[4,4,1,1]$ & 99.4 & 98.6 & 74.7 & 71.1 \\
$w_l=[1,4,4,1]$ & 99.5 & 98.8 & 74.0 & 70.5 \\
$w_l=[1,1,4,4]$ & 99.5 & 98.9 & 75.2 & 70.7 \\
$w_l=[1,2,3,4]$ & 99.5 & 98.9 & 75.4 & 70.9 \\
$w_l=[4,3,2,1]$ & 99.4 & 98.6 & 74.9 & 71.2 \\
\bottomrule
\end{tabular}
\label{tab:mvtec_sensitivity}
\end{table}

\section{Unseen Normal Generalization}
\label{sec:unseen-normal}

To assess whether retrieval-based scoring over-penalizes normal test samples that are not stored in the memory bank, we evaluate RAD on the official unseen-normal test subset of MVTec-AD. As shown in Tab.~\ref{tab:unseen_normal}, RAD produces a mean image-level score of \textbf{0.0596} for unseen normal samples and \textbf{0.2069} for anomalous samples, while Dinomaly produces \textbf{0.1724} and \textbf{0.2837}, respectively. At 95\% TPR, RAD obtains a lower false-positive rate of \textbf{3.5\%}, compared with \textbf{9.9\%} for Dinomaly. These results show that RAD does not simply assign high anomaly scores to all samples absent from the memory bank; unseen normal variations can still receive low retrieval scores when they remain close to the anomaly-free support.

\begin{table}[h]
  \centering
  \caption{Unseen-normal generalization on the official MVTec-AD unseen-normal test subset. RAD better separates unseen normal samples from anomalous samples and yields a lower false-positive rate at 95\% TPR.}
  \label{tab:unseen_normal}
  \setlength{\tabcolsep}{5pt}
  \renewcommand{\arraystretch}{1.05}
  \begin{tabular}{lccc}
    \toprule
    Method & Normal score $\downarrow$ & Anomaly score & FPR@95TPR $\downarrow$ \\
    \midrule
    Dinomaly & 0.1724 & 0.2837 & 9.9\% \\
    \rowcolor{cyan!10}
    RAD & \textbf{0.0596} & 0.2069 & \textbf{3.5\%} \\
    \bottomrule
  \end{tabular}
\end{table}

\section{Additional Incremental-class Scaling Results on VisA}
\label{sec:visa_incremental_scaling}

To further examine whether the incremental-class scaling behavior transfers beyond MVTec-AD, we conduct an additional experiment on VisA under a challenging target-class scaling protocol.
Specifically, we keep \textit{chewinggum}, \textit{cashew}, and \textit{candle} fixed in the anomaly-free memory, and gradually scale the target class \textit{macaroni1}.
For a fair comparison, Dinomaly is also reported with the stronger DINOv3 backbone.

As shown in Tab.~\ref{tab:visa_scaling_macaroni1}, RAD consistently outperforms Dinomaly across all memory sizes.
In the extremely low-data regime, with only $0.5\%$ of target-class anomaly-free samples, RAD improves I-AUROC by \textbf{19.1$\uparrow$}.
As more anomaly-free samples are added, both methods improve, while RAD maintains stronger localization performance.
At the full-data setting, RAD still improves over Dinomaly by \textbf{5.7$\uparrow$} on P-AP and \textbf{3.0$\uparrow$} on P-$F_1$.
These results support the conclusion that RAD can incorporate new class evidence by expanding the memory bank, without re-training model parameters.

\begin{table}[h]
\centering
\small
\setlength{\tabcolsep}{4pt}
\renewcommand{\arraystretch}{1.05}
\caption{Transfer results of the incremental-class scaling experiment on VisA. We keep \textit{chewinggum}, \textit{cashew}, and \textit{candle} fixed and gradually scale the target class \textit{macaroni1}.}
\label{tab:visa_scaling_macaroni1}
\begin{tabular}{lccc|ccc}
\toprule
\multirow{2}{*}{Ratio (\%)} & \multicolumn{3}{c|}{Dinomaly (DINOv3)} & \multicolumn{3}{c}{RAD} \\
 & I-AUROC & P-AP & P-$F_1$ & I-AUROC & P-AP & P-$F_1$ \\
\midrule
0.5   & 68.6 & 16.5 & 26.9 & \textbf{87.7} & \textbf{20.2} & \textbf{30.2} \\
5.0   & 90.9 & 26.4 & 34.7 & \textbf{92.6} & \textbf{28.6} & \textbf{35.2} \\
20.0  & 93.1 & 28.8 & 34.9 & \textbf{93.9} & \textbf{31.7} & \textbf{36.8} \\
100.0 & 95.9 & 31.7 & 37.2 & \textbf{96.5} & \textbf{37.4} & \textbf{40.2} \\
\bottomrule
\end{tabular}
\end{table}

\section{Per-category Quantitative Results}
\label{sec:pre-category}
For future research, we report the per-class results of MVTec-AD~\cite{bergmann2019mvtec}, VisA~\cite{zou2022spot}, and Real-IAD~\cite{wang2024real}.
The results of image-level anomaly detection and pixel-level anomaly localization on MVTec-AD are presented in
Table~\ref{tab:mvtec_i_per_class} and Table~\ref{tab:mvtec_p_per_class}, respectively.
The results of image-level anomaly detection and pixel-level anomaly localization on VisA are presented in
Table~\ref{tab:visa_i_per_class} and Table~\ref{tab:visa_p_per_class}, respectively.
The results of image-level anomaly detection and pixel-level anomaly localization on Real-IAD are presented in
Table~\ref{tab:real_iad_i_per_class} and Table~\ref{tab:real_iad_p_per_class}, respectively.

\begin{table*}[h]
\centering
\caption{Per-class performance on \textbf{MVTec-AD} dataset for multi-class anomaly detection with AUROC/AP/$F_{1}$-max metrics.}
\label{tab:mvtec_i_per_class}
\resizebox{\textwidth}{!}{
\setlength{\tabcolsep}{3pt}
\setlength{\aboverulesep}{0pt}
\setlength{\belowrulesep}{0pt}
\begin{tabular}{cccccccccc>{\columncolor{oursblue}}c}
\toprule
\multicolumn{2}{c}{Method $\rightarrow$} 
& RD4AD & UniAD & SimpleNet & DeSTSeg & DiAD & MambaAD & OmiAD & Dinomaly & RAD \\
\cmidrule(lr){1-2}
\multicolumn{2}{c}{Category $\downarrow$} 
& CVPR'22 & NIPS'22 & CVPR'23 & CVPR'23 & AAAI'24 & NIPS'24 & ICML'25 & CVPR'25 & Ours \\
\cmidrule(lr){1-10}
\multirow{10}{*}{\rotatebox{270}{Objects}} 
& Bottle     & 99.6/99.9/98.4 & 99.7/\textbf{100.}/\textbf{100.} & \textbf{100.}/\textbf{100.}/\textbf{100.} & 98.7/99.6/96.8 & 99.7/96.5/91.8 & \textbf{100.}/\textbf{100.}/\textbf{100.} & \textbf{100.}/\textbf{100.}/\textbf{100.} & \textbf{100.}/\textbf{100.}/\textbf{100.} & \textbf{100.}/\textbf{100.}/\textbf{100.} \\

& \cellcolor{rowgray} Cable & \cellcolor{rowgray} 84.1/89.5/82.5 & \cellcolor{rowgray} 95.2/95.9/88.0 & \cellcolor{rowgray} 97.5/98.5/94.7 & \cellcolor{rowgray} 89.5/94.6/85.9 & \cellcolor{rowgray} 94.8/98.8/95.2 & \cellcolor{rowgray} 98.8/99.2/95.7 & \cellcolor{rowgray} 98.4/99.4/95.6 & \cellcolor{rowgray} \textbf{100.}/\textbf{100.}/\textbf{100.} & 99.8/99.9/98.4 \\

& Capsule    & 94.1/96.9/96.9 & 86.9/97.8/94.4 & 90.7/97.9/93.5 & 82.8/95.9/92.6 & 89.0/97.5/95.5 & 94.4/98.7/94.9 & 94.7/99.3/96.8 & 97.9/99.5/97.7 & \textbf{98.5}/\textbf{99.7}/\textbf{99.1} \\

& \cellcolor{rowgray} Hazelnut & \cellcolor{rowgray} 60.8/69.8/86.4 & \cellcolor{rowgray} 99.8/\textbf{100.}/99.3 & \cellcolor{rowgray} 99.9/99.9/99.3 & \cellcolor{rowgray} 98.8/99.2/98.6 & \cellcolor{rowgray} 99.5/99.7/97.3 & \cellcolor{rowgray} \textbf{100.}/\textbf{100.}/\textbf{100.} & \cellcolor{rowgray} \textbf{100.}/\textbf{100.}/\textbf{100.} & \cellcolor{rowgray} \textbf{100.}/\textbf{100.}/\textbf{100.} & \textbf{100.}/\textbf{100.}/\textbf{100.} \\

& Metal Nut  & \textbf{100.}/\textbf{100.}/99.5 & 99.2/99.9/99.5 & 96.9/99.3/96.1 & 92.9/98.4/92.2 & 99.1/96.0/91.6 & 99.9/\textbf{100.}/99.5 & 99.4/99.9/98.9 & \textbf{100.}/\textbf{100.}/\textbf{100.} & \textbf{100.}/\textbf{100.}/\textbf{100.} \\

& \cellcolor{rowgray} Pill & \cellcolor{rowgray} 97.5/99.6/96.8 & \cellcolor{rowgray} 93.7/98.7/95.7 & \cellcolor{rowgray} 88.2/97.7/92.5 & \cellcolor{rowgray} 77.1/94.4/91.7 & \cellcolor{rowgray} 95.7/98.5/94.5 & \cellcolor{rowgray} 97.0/99.5/96.2 & \cellcolor{rowgray} 94.2/99.2/95.4 & \cellcolor{rowgray} \textbf{99.1}/\textbf{99.9}/\textbf{98.3} & 98.7/99.8/98.2 \\

& Screw      & 97.7/99.3/95.8 & 87.5/96.5/89.0 & 76.7/90.6/87.7 & 69.9/88.4/85.4 & 90.7/\textbf{99.7}/\textbf{97.9} & 94.7/97.9/94.0 & 96.9/98.8/96.3 & \textbf{98.4}/99.5/96.1 & 97.2/99.1/95.3 \\

& \cellcolor{rowgray} Toothbrush & \cellcolor{rowgray} 97.2/99.0/94.7 & \cellcolor{rowgray} 94.2/97.4/95.2 & \cellcolor{rowgray} 89.7/95.7/92.3 & \cellcolor{rowgray} 71.7/89.3/84.5 & \cellcolor{rowgray} 99.7/99.9/99.2 & \cellcolor{rowgray} 98.3/99.3/98.4 & \cellcolor{rowgray} 99.7/\textbf{100.}/\textbf{100.} & \cellcolor{rowgray} \textbf{100.}/\textbf{100.}/\textbf{100.} & \textbf{100.}/\textbf{100.}/\textbf{100.} \\

& Transistor & 94.2/99.8/90.0 & 99.8/98.0/93.8 & 99.2/98.7/97.6 & 78.2/79.5/68.8 & 99.8/99.6/97.4 & \textbf{100.}/\textbf{100.}/\textbf{100.} & 99.9/99.9/98.8 & 99.0/98.0/96.4 & 99.5/99.3/96.3 \\

& \cellcolor{rowgray} Zipper & \cellcolor{rowgray} 99.5/99.9/99.2 & \cellcolor{rowgray} 95.8/99.5/97.1 & \cellcolor{rowgray} 99.0/99.7/98.3 & \cellcolor{rowgray} 88.4/96.3/93.1 & \cellcolor{rowgray} 95.1/99.1/94.4 & \cellcolor{rowgray} 99.3/99.8/97.5 & \cellcolor{rowgray} 99.8/\textbf{100.}/99.6 & \cellcolor{rowgray} \textbf{100.}/\textbf{100.}/\textbf{100.} & 99.7/99.9/99.2 \\

\cmidrule(lr){1-10}
\multirow{5}{*}{\rotatebox{270}{Textures}} 
& Carpet & 98.5/99.6/97.2 & 99.8/99.9/\textbf{99.4} & 95.7/98.7/93.2 & 95.9/98.8/94.9 & 99.4/99.9/98.3 & 99.8/99.9/\textbf{99.4} & 99.6/\textbf{100.}/\textbf{99.4} & 99.8/\textbf{100.}/98.9 & \textbf{99.9}/\textbf{100.}/\textbf{99.4} \\

& \cellcolor{rowgray} Grid & \cellcolor{rowgray} 98.0/99.4/96.5 & \cellcolor{rowgray} 98.2/99.5/97.3 & \cellcolor{rowgray} 97.6/99.2/96.4 & \cellcolor{rowgray} 97.9/99.2/96.6 & \cellcolor{rowgray} 98.5/99.8/97.7 & \cellcolor{rowgray} \textbf{100.}/\textbf{100.}/\textbf{100.} & \cellcolor{rowgray} 99.8/99.9/99.1 & \cellcolor{rowgray} 99.9/\textbf{100.}/99.1 & \textbf{100.}/\textbf{100.}/\textbf{100.} \\

& Leather & \textbf{100.}/\textbf{100.}/\textbf{100.} & \textbf{100.}/\textbf{100.}/\textbf{100.} & \textbf{100.}/\textbf{100.}/\textbf{100.} & 99.2/99.8/98.9 & 99.8/99.7/97.6 & \textbf{100.}/\textbf{100.}/\textbf{100.} & \textbf{100.}/\textbf{100.}/\textbf{100.} & \textbf{100.}/\textbf{100.}/\textbf{100.} & \textbf{100.}/\textbf{100.}/\textbf{100.} \\

& \cellcolor{rowgray} Tile & \cellcolor{rowgray} 98.3/99.3/96.4 & \cellcolor{rowgray} 99.3/99.8/98.2 & \cellcolor{rowgray} 99.3/99.8/98.8 & \cellcolor{rowgray} 97.0/98.9/95.3 & \cellcolor{rowgray} 96.8/99.9/98.4 & \cellcolor{rowgray} 98.2/99.3/95.4 & \cellcolor{rowgray} \textbf{100.}/99.9/98.8 & \cellcolor{rowgray} \textbf{100.}/\textbf{100.}/\textbf{100.} & 99.8/99.9/98.8 \\

& Wood       & 99.2/99.8/98.3 & 98.6/99.6/96.6 & 98.4/99.5/96.7 & \textbf{99.9}/\textbf{100.}/99.2 & 99.7/\textbf{100.}/\textbf{100.} & 98.8/99.6/96.6 & 99.0/99.8/98.3 & 99.8/99.9/99.2 & 99.4/99.8/98.3 \\

\midrule
\multicolumn{2}{c}{\cellcolor{rowgray} Mean} & \cellcolor{rowgray} 94.6/96.5/95.2 & \cellcolor{rowgray} 96.5/98.8/96.2 & \cellcolor{rowgray} 95.3/98.4/95.8 & \cellcolor{rowgray} 89.2/95.5/91.6 & \cellcolor{rowgray} 97.2/99.0/96.5 & \cellcolor{rowgray} 98.6/99.6/97.8 & \cellcolor{rowgray} 98.8/99.7/98.5 & \cellcolor{rowgray} \textbf{99.6}/\textbf{100.}/\textbf{99.0} & 99.5/\textbf{100.}/98.9 \\
\bottomrule
\end{tabular}
}
\end{table*}

\begin{table*}[p]
\centering
\caption{Per-class performance on \textbf{MVTec-AD} dataset for multi-class anomaly localization with AUROC/AP/$F_{1}$-max/AUPRO metrics.}
\label{tab:mvtec_p_per_class}
\resizebox{\textwidth}{!}{
\setlength{\tabcolsep}{3pt}
\setlength{\aboverulesep}{0pt}
\setlength{\belowrulesep}{0pt}
\begin{tabular}{cccccccccc>{\columncolor{oursblue}}c}
\toprule
\multicolumn{2}{c}{Method $\rightarrow$} 
& RD4AD & UniAD & SimpleNet & DeSTSeg & DiAD & MambaAD & OmiAD & Dinomaly & RAD \\
\cmidrule(lr){1-2}
\multicolumn{2}{c}{Category $\downarrow$} 
& CVPR'22 & NIPS'22 & CVPR'23 & CVPR'23 & AAAI'24 & NIPS'24 & ICML'25 & CVPR'25 & Ours \\
\cmidrule(lr){1-10}
\multirow{10}{*}{\rotatebox{270}{Objects}} 
& Bottle     & 97.8/68.2/67.6/94.0 & 98.1/66.0/69.2/93.1 & 97.2/53.8/62.4/89.0 & 93.3/61.7/56.0/67.5 & 98.4/52.2/54.8/86.6 & 98.8/79.7/76.7/95.2 & 98.6/74.9/73.8/95.6 & \textbf{99.2}/\textbf{88.6}/\textbf{84.2}/\textbf{96.6} & 98.9/88.5/80.8/\textbf{96.6} \\

& \cellcolor{rowgray} Cable & \cellcolor{rowgray} 85.1/26.3/33.6/75.1 & \cellcolor{rowgray} 97.3/39.9/45.2/86.1 & \cellcolor{rowgray} 96.7/42.4/51.2/85.4 & \cellcolor{rowgray} 89.3/37.5/40.5/49.4 & \cellcolor{rowgray} 96.8/50.1/57.8/80.5 & \cellcolor{rowgray} 95.8/42.2/48.1/90.3 & \cellcolor{rowgray} 98.3/60.5/62.9/91.8 & \cellcolor{rowgray} \textbf{98.6}/72.0/\textbf{74.3}/\textbf{94.2} & 98.3/\textbf{79.5}/72.9/93.8 \\

& Capsule    & 98.8/43.4/50.0/94.8 & 98.5/42.7/46.5/92.1 & 98.5/35.4/44.3/84.5 & 95.8/47.9/48.9/62.1 & 97.1/42.0/45.3/87.2 & 98.4/43.9/47.7/92.6 & 98.9/48.1/52.0/92.5 & 98.7/61.4/\textbf{60.3}/97.2 & \textbf{99.0}/\textbf{65.3}/60.1/\textbf{97.6} \\

& \cellcolor{rowgray} Hazelnut & \cellcolor{rowgray} 97.9/36.2/51.6/92.7 & \cellcolor{rowgray} 98.1/55.2/56.8/94.1 & \cellcolor{rowgray} 98.4/44.6/51.4/87.4 & \cellcolor{rowgray} 98.2/65.8/61.6/84.5 & \cellcolor{rowgray} 98.3/79.2/80.4/91.5 & \cellcolor{rowgray} 99.0/63.6/64.4/95.7 & \cellcolor{rowgray} 98.6/59.6/59.9/94.3 & \cellcolor{rowgray} 99.4/82.2/76.4/\textbf{97.0} & \textbf{99.6}/\textbf{89.9}/\textbf{84.5}/\textbf{97.0} \\

& Metal Nut  & 94.8/55.5/66.4/91.9 & 62.7/14.6/29.2/81.8 & 98.0/\textbf{83.1}/79.4/85.2 & 84.2/42.0/22.8/53.0 & 97.3/30.0/38.3/90.6 & 96.7/74.5/79.1/93.7 & 96.5/66.6/75.6/90.3 & 96.9/78.6/86.7/94.9 & \textbf{98.5}/\textbf{89.0}/\textbf{89.2}/\textbf{96.0} \\

& \cellcolor{rowgray} Pill & \cellcolor{rowgray} 97.5/63.4/65.2/95.8 & \cellcolor{rowgray} 95.0/44.0/53.9/95.3 & \cellcolor{rowgray} 96.5/72.4/67.7/81.9 & \cellcolor{rowgray} 96.2/61.7/41.8/27.9 & \cellcolor{rowgray} 95.7/46.0/51.4/89.0 & \cellcolor{rowgray} 97.4/64.0/66.5/95.7 & \cellcolor{rowgray} 96.6/56.8/60.7/95.9 & \cellcolor{rowgray} \textbf{97.8}/\textbf{76.4}/\textbf{71.6}/\textbf{97.3} & 94.9/64.4/64.2/96.6 \\

& Screw      & 99.4/40.2/44.6/96.8 & 98.3/28.7/37.6/95.2 & 96.5/15.9/23.2/84.0 & 93.8/19.9/25.3/47.3 & 97.9/\textbf{60.6}/59.6/95.0 & 99.5/49.8/50.9/97.1 & 99.5/38.7/43.5/97.2 & \textbf{99.6}/60.2/\textbf{59.6}/\textbf{98.3} & 99.2/58.4/56.3/96.4 \\

& \cellcolor{rowgray} Toothbrush & \cellcolor{rowgray} 99.0/53.6/58.8/92.0 & \cellcolor{rowgray} 98.4/34.9/45.7/87.9 & \cellcolor{rowgray} 98.4/46.9/52.5/87.4 & \cellcolor{rowgray} 96.2/52.9/58.8/30.9 & \cellcolor{rowgray} 99.0/\textbf{78.7}/72.8/95.0 & \cellcolor{rowgray} 99.0/48.5/59.2/91.7 & \cellcolor{rowgray} 98.7/40.5/56.2/91.1 & \cellcolor{rowgray} 98.9/51.5/62.6/95.3 & \textbf{99.4}/73.0/\textbf{72.9}/\textbf{97.1} \\

& Transistor & 85.9/42.3/45.2/74.7 & \textbf{97.9}/59.5/\textbf{64.6}/93.5 & 95.8/58.2/56.0/83.2 & 73.6/38.4/39.2/43.9 & 95.1/15.6/31.7/90.0 & 96.5/69.4/67.1/87.0 & 98.4/73.4/72.5/\textbf{96.1} & 93.2/59.9/58.5/77.0 & 97.5/\textbf{80.8}/72.6/92.5 \\

& \cellcolor{rowgray} Zipper & \cellcolor{rowgray} 98.5/53.9/60.3/94.1 & \cellcolor{rowgray} 96.8/40.1/41.9/92.6 & \cellcolor{rowgray} 97.9/53.4/54.6/90.7 & \cellcolor{rowgray} 97.3/64.7/59.2/66.9 & \cellcolor{rowgray} 96.2/60.7/60.0/91.6 & \cellcolor{rowgray} 98.4/60.4/61.7/94.3 & \cellcolor{rowgray} 98.6/52.7/59.3/95.6 & \cellcolor{rowgray} \textbf{99.2}/\textbf{79.5}/\textbf{75.4}/\textbf{97.2} & 98.9/76.4/69.6/95.9 \\

\cmidrule(lr){1-10}
\multirow{5}{*}{\rotatebox{270}{Textures}} 
& Carpet     & 99.0/58.5/60.4/95.1 & 98.5/49.9/51.1/94.4 & 97.4/38.7/43.2/90.6 & 93.6/59.9/58.9/89.3 & 98.6/42.2/46.4/90.6 & 99.2/60.0/63.3/96.7 & 98.5/52.9/54.8/94.6 & 99.3/68.7/71.1/97.6 & \textbf{99.6}/\textbf{83.5}/\textbf{78.0}/\textbf{98.5} \\

& \cellcolor{rowgray} Grid & \cellcolor{rowgray} 96.5/23.0/28.4/97.0 & \cellcolor{rowgray} 63.1/10.7/11.9/92.9 & \cellcolor{rowgray} 96.8/20.5/27.6/88.6 & \cellcolor{rowgray} 97.0/42.1/46.9/86.8 & \cellcolor{rowgray} 96.6/\textbf{66.0}/\textbf{64.1}/94.0 & \cellcolor{rowgray} 99.2/47.4/47.7/97.0 & \cellcolor{rowgray} 98.5/35.4/37.1/95.5 & \cellcolor{rowgray} \textbf{99.4}/55.3/57.7/\textbf{97.2} & 99.3/60.1/59.4/95.8 \\

& Leather    & 99.3/38.0/45.1/97.4 & 98.8/32.9/34.4/96.8 & 98.7/28.5/32.9/92.7 & 99.5/\textbf{71.5}/\textbf{66.5}/91.1 & \textbf{99.8}/56.1/62.3/91.3 & 99.4/50.3/53.3/98.7 & 98.9/36.3/39.4/96.9 & 99.4/52.2/55.0/97.6 & 99.4/63.4/59.4/\textbf{98.9} \\

& \cellcolor{rowgray} Tile & \cellcolor{rowgray} 95.3/48.5/60.5/85.8 & \cellcolor{rowgray} 91.8/42.1/50.6/78.4 & \cellcolor{rowgray} 95.7/60.5/59.9/90.6 & \cellcolor{rowgray} 93.0/71.0/66.2/87.1 & \cellcolor{rowgray} 92.4/65.7/64.1/\textbf{90.7} & \cellcolor{rowgray} 93.8/45.1/54.8/80.0 & \cellcolor{rowgray} 92.7/47.5/54.7/82.2 & \cellcolor{rowgray} \textbf{98.1}/80.1/75.7/90.5 & 97.5/\textbf{82.6}/\textbf{77.1}/89.5 \\

& Wood       & 95.3/47.8/51.0/90.0 & 93.2/37.2/41.5/86.7 & 91.4/34.8/39.7/76.3 & 95.9/77.3/71.3/83.4 & 93.3/43.3/43.5/\textbf{97.5} & 94.4/46.2/48.2/91.2 & 94.2/44.5/48.0/88.5 & \textbf{97.6}/72.8/68.4/94.0 & 97.0/\textbf{78.9}/\textbf{72.6}/94.8 \\
\midrule
\multicolumn{2}{c}{\cellcolor{rowgray} Mean} & \cellcolor{rowgray} 96.1/48.6/53.8/91.1 & \cellcolor{rowgray} 96.8/43.4/49.5/90.7 & \cellcolor{rowgray} 96.9/45.9/49.7/86.5 & \cellcolor{rowgray} 93.1/54.3/50.9/64.8 & \cellcolor{rowgray} 96.8/52.6/55.5/90.7 & \cellcolor{rowgray} 97.7/56.3/59.2/93.1 & \cellcolor{rowgray} 97.7/52.6/56.7/93.2 & \cellcolor{rowgray} 98.4/69.3/69.2/94.8 & \textbf{98.5}/\textbf{75.6}/\textbf{71.3}/\textbf{95.8} \\
\bottomrule
\end{tabular}
}
\end{table*}

\begin{table*}[p]
\centering
\caption{Per-class performance on \textbf{VisA} dataset for multi-class anomaly detection with AUROC/AP/$F_{1}$-max metrics.}
\label{tab:visa_i_per_class}
\resizebox{\textwidth}{!}{
\setlength{\tabcolsep}{3pt}
\setlength{\aboverulesep}{0pt}
\setlength{\belowrulesep}{0pt}
\begin{tabular}{ccccccccc>{\columncolor{oursblue}}c}
\toprule
Method $\rightarrow$ & RD4AD & UniAD & SimpleNet & DeSTSeg & DiAD & MambaAD & OmiAD & Dinomaly & RAD \\
\cmidrule(lr){1-1}
Category $\downarrow$ & CVPR'22 & NIPS'22 & CVPR'23 & CVPR'23 & AAAI'24 & NIPS'24 & ICML'25 & CVPR'25 & Ours \\
\midrule
pcb1        & 96.2/95.5/91.9 & 92.8/92.7/87.8 & 91.6/91.9/86.0 & 87.6/83.1/83.7 & 88.1/88.7/80.7 & 95.4/93.0/91.6 & 97.8/97.4/95.1 & \textbf{99.1}/\textbf{99.1}/\textbf{96.6} & 98.1/97.2/\textbf{96.6} \\
\cellcolor{rowgray} pcb2 & \cellcolor{rowgray} 97.8/97.8/94.2 & \cellcolor{rowgray} 87.8/87.7/83.1 & \cellcolor{rowgray} 92.4/93.3/84.5 & \cellcolor{rowgray} 86.5/85.8/82.6 & \cellcolor{rowgray} 91.4/91.4/84.7 & \cellcolor{rowgray} 94.2/93.7/89.3 & \cellcolor{rowgray} 97.8/98.5/94.1 & \cellcolor{rowgray} \textbf{99.3}/\textbf{99.2}/\textbf{97.0} & 96.4/94.5/94.2 \\
pcb3        & 96.4/96.2/91.0 & 78.6/78.6/76.1 & 89.1/91.1/82.6 & 93.7/95.1/87.0 & 86.2/87.6/77.6 & 93.7/94.1/86.7 & 96.7/95.1/87.6 & 98.9/98.9/\textbf{96.1} & \textbf{99.5}/\textbf{99.5}/96.0 \\
\cellcolor{rowgray} pcb4 & \cellcolor{rowgray} 99.9/99.9/\textbf{99.0} & \cellcolor{rowgray} 98.8/98.8/94.3 & \cellcolor{rowgray} 97.0/97.0/93.5 & \cellcolor{rowgray} 97.8/97.8/92.7 & \cellcolor{rowgray} 99.6/99.5/97.0 & \cellcolor{rowgray} 99.9/99.9/98.5 & \cellcolor{rowgray} \textbf{100.}/\textbf{100.}/\textbf{99.0} & \cellcolor{rowgray} 99.8/99.8/98.0 & 99.6/99.6/98.5 \\
\cmidrule(lr){1-9}
macaroni1   & 75.9/ 1.5/76.8 & 79.9/79.8/72.7 & 85.9/82.5/73.1 & 76.6/69.0/71.0 & 85.7/85.2/78.8 & 91.6/89.8/81.6 & 97.3/97.5/92.8 & \textbf{98.0}/\textbf{97.6}/\textbf{94.2} & 96.9/97.0/92.2 \\
\cellcolor{rowgray} macaroni2 & \cellcolor{rowgray} 88.3/84.5/83.8 & \cellcolor{rowgray} 71.6/71.6/69.9 & \cellcolor{rowgray} 68.3/54.3/59.7 & \cellcolor{rowgray} 68.9/62.1/67.7 & \cellcolor{rowgray} 62.5/57.4/69.6 & \cellcolor{rowgray} 81.6/78.0/73.8 & \cellcolor{rowgray} 85.1/83.3/79.5 & \cellcolor{rowgray} \textbf{95.9}/\textbf{95.7}/\textbf{90.7} & 84.3/82.6/80.4 \\
capsules    & 82.2/90.4/81.3 & 55.6/55.6/76.9 & 74.1/82.8/74.6 & 87.1/93.0/84.2 & 58.2/69.0/78.5 & 91.8/95.0/88.8 & 85.7/89.0/78.8 & \textbf{98.6}/\textbf{99.0}/\textbf{97.1} & 97.6/98.5/95.6 \\
\cellcolor{rowgray} candle & \cellcolor{rowgray} 92.3/92.9/86.0 & \cellcolor{rowgray} 94.1/94.0/86.1 & \cellcolor{rowgray} 84.1/73.3/76.6 & \cellcolor{rowgray} 94.9/94.8/89.2 & \cellcolor{rowgray} 92.8/92.0/87.6 & \cellcolor{rowgray} 96.8/96.9/90.1 & \cellcolor{rowgray} 97.4/98.6/93.4 & \cellcolor{rowgray} 98.7/98.8/95.1 & \textbf{99.5}/\textbf{99.5}/\textbf{96.5} \\
\midrule
cashew      & 92.0/95.8/90.7 & 92.8/92.8/91.4 & 88.0/91.3/84.7 & 92.0/96.1/88.1 & 91.5/95.7/89.7 & 94.5/97.3/91.1 & 93.3/96.4/90.9 & 98.7/99.4/97.0 & \textbf{99.1}/\textbf{99.6}/\textbf{97.5} \\
\cellcolor{rowgray} chewing gum & \cellcolor{rowgray} 94.9/97.5/92.1 & \cellcolor{rowgray} 96.3/96.2/95.2 & \cellcolor{rowgray} 96.4/98.2/93.8 & \cellcolor{rowgray} 95.8/98.3/94.7 & \cellcolor{rowgray} 99.1/99.5/95.9 & \cellcolor{rowgray} 97.7/98.9/94.2 & \cellcolor{rowgray} 99.2/99.8/97.5 & \cellcolor{rowgray} \textbf{99.8}/\textbf{99.9}/\textbf{99.0} & 99.2/99.6/98.0 \\
fryum       & 95.3/97.9/91.5 & 83.0/83.0/85.0 & 88.4/93.0/83.3 & 92.1/96.1/89.5 & 89.8/95.0/87.2 & 95.2/97.7/90.5 & 94.0/96.5/88.5 & 98.8/99.4/96.5 & \textbf{99.5}/\textbf{99.8}/\textbf{99.0} \\
\cellcolor{rowgray} pipe fryum & \cellcolor{rowgray} 97.9/98.9/96.5 & \cellcolor{rowgray} 94.7/94.7/93.9 & \cellcolor{rowgray} 90.8/95.5/88.6 & \cellcolor{rowgray} 94.1/97.1/91.9 & \cellcolor{rowgray} 96.2/98.1/93.7 & \cellcolor{rowgray} 98.7/99.3/97.0 & \cellcolor{rowgray} 98.9/99.4/97.0 & \cellcolor{rowgray} 99.2/99.7/97.0 & \textbf{99.8}/\textbf{99.9}/\textbf{98.5} \\
\cmidrule(lr){1-9}
Mean        & 92.4/92.4/89.6 & 85.5/85.5/84.4 & 87.2/87.0/81.8 & 88.9/89.0/85.2 & 86.8/88.3/85.1 & 94.3/94.5/89.4 & 95.3/96.0/91.2 & \textbf{98.7}/\textbf{98.9}/\textbf{96.2} & 97.5/97.3/95.2 \\
\bottomrule
\end{tabular}
}
\end{table*}

\begin{table*}[p]
\centering
\caption{Per-class performance on \textbf{VisA} dataset for multi-class anomaly localization with AUROC/AP/$F_{1}$-max/AUPRO metrics.}
\label{tab:visa_p_per_class}
\resizebox{\textwidth}{!}{
\setlength{\tabcolsep}{3pt}
\setlength{\aboverulesep}{0pt}
\setlength{\belowrulesep}{0pt}
\begin{tabular}{ccccccccc>{\columncolor{oursblue}}c}
\toprule
Method $\rightarrow$ & RD4AD & UniAD & SimpleNet & DeSTSeg & DiAD & MambaAD & OmiAD & Dinomaly & RAD \\
\cmidrule(lr){1-1}
Category $\downarrow$ & CVPR'22 & NIPS'22 & CVPR'23 & CVPR'23 & AAAI'24 & NIPS'24 & ICML'25 & CVPR'25 & Ours \\
\midrule
pcb1        & 99.4/66.2/62.4/95.8 & 93.3/ 3.9/ 8.3/64.1 & 99.2/86.1/78.8/83.6 & 95.8/46.4/49.0/83.2 & 98.7/49.6/52.8/80.2 & \textbf{99.8}/77.1/72.4/92.8 & 99.7/70.8/66.2/93.0 & 99.5/87.9/80.5/95.1 & 99.7/\textbf{90.8}/\textbf{83.4}/\textbf{96.0} \\
\cellcolor{rowgray} pcb2 & \cellcolor{rowgray} 98.0/22.3/30.0/90.8 & \cellcolor{rowgray} 93.9/ 4.2/ 9.2/66.9 & \cellcolor{rowgray} 96.6/ 8.9/18.6/85.7 & \cellcolor{rowgray} 97.3/14.6/28.2/79.9 & \cellcolor{rowgray} 95.2/ 7.5/16.7/67.0 & \cellcolor{rowgray} \textbf{98.9}/13.3/23.4/89.6 & \cellcolor{rowgray} \textbf{98.9}/16.6/22.2/87.1 & \cellcolor{rowgray} 98.0/\textbf{47.0}/\textbf{49.8}/\textbf{91.3} & \textbf{98.9}/26.5/42.3/89.5 \\
pcb3        & 97.9/26.2/35.2/93.9 & 97.3/13.8/21.9/70.6 & 97.2/31.0/36.1/85.1 & 97.7/28.1/33.4/62.4 & 96.7/ 8.0/18.8/68.9 & \textbf{99.1}/18.3/27.4/89.1 & \textbf{99.1}/29.7/30.4/87.2 & 98.4/\textbf{41.7}/\textbf{45.3}/\textbf{94.6} & 99.0/30.9/38.5/92.8 \\
\cellcolor{rowgray} pcb4 & \cellcolor{rowgray} 97.8/31.4/37.0/88.7 & \cellcolor{rowgray} 94.9/14.7/22.9/72.3 & \cellcolor{rowgray} 93.9/23.9/32.9/61.1 & \cellcolor{rowgray} 95.8/\textbf{53.0}/\textbf{53.2}/76.9 & \cellcolor{rowgray} 97.0/17.6/27.2/85.0 & \cellcolor{rowgray} 98.6/47.0/46.9/87.6 & \cellcolor{rowgray} 98.1/42.0/44.1/86.4 & \cellcolor{rowgray} \textbf{98.7}/50.5/53.1/\textbf{94.4} & 98.6/51.8/51.7/93.1 \\
\midrule
macaroni1   & 99.4/ 2.9/ 6.9/95.3 & 97.4/ 3.7/ 9.7/84.0 & 98.9/ 3.5/ 8.4/92.0 & 99.1/ 5.8/13.4/62.4 & 94.1/10.2/16.7/68.5 & 99.5/17.5/27.6/95.2 & \textbf{99.7}/20.1/29.8/96.3 & 99.6/33.5/40.6/96.4 & \textbf{99.7}/\textbf{38.1}/\textbf{40.8}/\textbf{96.9} \\
\cellcolor{rowgray} macaroni2 & \cellcolor{rowgray} \textbf{99.7}/13.2/21.8/97.4 & \cellcolor{rowgray} 95.2/ 0.9/ 4.3/76.6 & \cellcolor{rowgray} 93.2/ 0.6/ 3.9/77.8 & \cellcolor{rowgray} 98.5/ 6.3/14.4/70.0 & \cellcolor{rowgray} 93.6/ 0.9/ 2.8/73.1 & \cellcolor{rowgray} 99.5/ 9.2/16.1/96.2 & \cellcolor{rowgray} 99.4/08.5/15.4/94.0 & \cellcolor{rowgray} \textbf{99.7}/24.7/\textbf{36.1}/\textbf{98.7} & 99.2/\textbf{26.8}/35.2/95.9 \\
capsules    & 99.4/60.4/60.8/93.1 & 88.7/ 3.0/ 7.4/43.7 & 97.1/52.9/53.3/73.7 & 96.9/33.2/ 9.1/76.7 & 97.3/10.0/21.0/77.9 & 99.1/61.3/59.8/91.8 & 99.4/62.7/59.6/90.8 & \textbf{99.6}/\textbf{65.0}/\textbf{66.6}/\textbf{97.4} & 98.9/58.2/59.6/97.0 \\
\cellcolor{rowgray} candle & \cellcolor{rowgray} 99.1/25.3/35.8/94.9 & \cellcolor{rowgray} 98.5/17.6/27.9/91.6 & \cellcolor{rowgray} 97.6/ 8.4/16.5/87.6 & \cellcolor{rowgray} 98.7/39.9/45.8/69.0 & \cellcolor{rowgray} 97.3/12.8/22.8/89.4 & \cellcolor{rowgray} 99.0/23.2/32.4/95.5 & \cellcolor{rowgray} 99.4/25.7/35.1/97.1 & \cellcolor{rowgray} 99.4/43.0/47.9/95.4 & \textbf{99.5}/\textbf{53.6}/\textbf{55.8}/\textbf{96.2} \\
\midrule
cashew      & 91.7/44.2/49.7/86.2 & 98.6/51.7/58.3/87.9 & 98.9/68.9/66.0/84.1 & 87.9/47.6/52.1/66.3 & 90.9/53.1/60.9/61.8 & 94.3/46.8/51.4/87.8 & 97.2/46.3/55.3/83.8 & 97.1/64.5/62.4/94.0 & \textbf{99.7}/\textbf{81.7}/\textbf{76.9}/\textbf{97.7} \\
\cellcolor{rowgray} chewing gum & \cellcolor{rowgray} 98.7/59.9/61.7/76.9 & \cellcolor{rowgray} 98.8/54.9/56.1/81.3 & \cellcolor{rowgray} 97.9/26.8/29.8/78.3 & \cellcolor{rowgray} 98.8/\textbf{86.9}/\textbf{81.0}/68.3 & \cellcolor{rowgray} 94.7/11.9/25.8/59.5 & \cellcolor{rowgray} 98.1/57.5/59.9/79.7 & \cellcolor{rowgray} 98.8/58.0/55.6/73.1 & \cellcolor{rowgray} 99.1/65.0/67.7/88.1 & \textbf{99.3}/82.3/76.6/\textbf{88.8} \\
fryum       & 97.0/47.6/51.5/93.4 & 95.9/34.0/40.6/76.2 & 93.0/39.1/45.4/85.1 & 88.1/35.2/38.5/47.7 & 97.6/\textbf{58.6}/\textbf{60.1}/81.3 & 96.9/47.8/51.9/91.6 & \textbf{97.7}/47.7/55.2/87.2 & 96.6/51.6/53.4/93.5 & 97.4/53.7/57.6/\textbf{94.0} \\
\cellcolor{rowgray} pipe fryum & \cellcolor{rowgray} 99.1/56.8/58.8/95.4 & \cellcolor{rowgray} 98.9/50.2/57.7/91.5 & \cellcolor{rowgray} 98.5/65.6/63.4/83.0 & \cellcolor{rowgray} 98.9/\textbf{78.8}/\textbf{72.7}/45.9 & \cellcolor{rowgray} \textbf{99.4}/72.7/69.9/89.9 & \cellcolor{rowgray} 99.1/53.5/58.5/95.1 & \cellcolor{rowgray} 99.2/56.2/60.1/94.5 & \cellcolor{rowgray} 99.2/64.3/65.1/95.2 & \textbf{99.4}/69.8/70.9/\textbf{96.8} \\
\midrule
Mean        & 98.1/38.0/42.6/91.8 & 95.9/21.0/27.0/75.6 & 96.8/34.7/37.8/81.4 & 96.1/39.6/43.4/67.4 & 96.0/26.1/33.0/75.2 & 98.5/39.4/44.0/91.0 & 98.9/40.4/44.1/89.2 & 98.7/53.2/55.7/94.5 & \textbf{99.1}/\textbf{55.3}/\textbf{57.5}/\textbf{94.6} \\
\bottomrule
\end{tabular}
}
\end{table*}

\begin{table*}[p]
\centering
\caption{Per-class performance on \textbf{Real-IAD} dataset for multi-class anomaly detection with AUROC/AP/$F_{1}$-max metrics.}
\label{tab:real_iad_i_per_class}
\resizebox{\textwidth}{!}{
\setlength{\tabcolsep}{3pt}
\setlength{\aboverulesep}{0pt}
\setlength{\belowrulesep}{0pt}
\begin{tabular}{ccccccccc>{\columncolor{oursblue}}c}
\toprule
Method $\rightarrow$ & RD4AD & UniAD & SimpleNet & DeSTSeg & DiAD & MambaAD & OmiAD & Dinomaly & RAD \\
\cmidrule(lr){1-1}
Category $\downarrow$ & CVPR'22 & NIPS'22 & CVPR'23 & CVPR'23 & AAAI'24 & NIPS'24 & ICML'25 & CVPR'25 & Ours \\
\midrule
audiojack        & 76.2/63.2/60.8 & 81.4/76.6/64.9 & 58.4/44.2/50.9 & 81.1/72.6/64.5 & 76.5/54.3/65.7 & 84.2/76.5/67.4 & 84.5/82.0/\textbf{77.8} & 86.8/82.4/72.2 & \textbf{87.9}/\textbf{83.6}/72.5 \\
\cellcolor{rowgray} bottle cap & \cellcolor{rowgray} 89.5/86.3/81.0 & \cellcolor{rowgray} 92.5/91.7/81.7 & \cellcolor{rowgray} 54.1/47.6/60.3 & \cellcolor{rowgray} 78.1/74.6/68.1 & \cellcolor{rowgray} 91.6/\textbf{94.0}/\textbf{87.9} & \cellcolor{rowgray} 92.8/92.0/82.1 & \cellcolor{rowgray} 93.7/92.7/85.2 & \cellcolor{rowgray} 89.9/86.7/81.2 & \textbf{93.8}/92.8/82.7 \\
button battery   & 73.3/78.9/76.1 & 75.9/81.6/76.3 & 52.5/60.5/72.4 & 86.7/89.2/83.5 & 80.5/71.3/70.6 & 79.8/85.3/77.8 & 84.9/90.1/80.1 & 86.6/88.9/82.1 & \textbf{89.5}/\textbf{91.3}/\textbf{85.2} \\
\cellcolor{rowgray} end cap & \cellcolor{rowgray} 79.8/84.0/77.8 & \cellcolor{rowgray} 80.9/86.1/78.0 & \cellcolor{rowgray} 51.6/60.8/72.9 & \cellcolor{rowgray} 77.9/81.1/77.1 & \cellcolor{rowgray} 85.1/83.4/\textbf{84.8} & \cellcolor{rowgray} 78.0/82.8/77.2 & \cellcolor{rowgray} 79.4/80.4/80.8 & \cellcolor{rowgray} 87.0/87.5/83.4 & \textbf{88.5}/\textbf{88.3}/84.7 \\
eraser           & 90.0/88.7/79.7 & 90.3/89.2/80.2 & 46.4/39.1/55.8 & 84.6/82.9/71.8 & 80.0/80.0/77.3 & 87.5/86.2/76.1 & 89.5/90.2/84.2 & 90.3/87.6/78.6 & \textbf{95.5}/\textbf{94.2}/\textbf{85.7} \\
\cellcolor{rowgray} fire hood & \cellcolor{rowgray} 78.3/70.1/64.5 & \cellcolor{rowgray} 80.6/74.8/66.4 & \cellcolor{rowgray} 58.1/41.9/54.4 & \cellcolor{rowgray} 81.7/72.4/67.7 & \cellcolor{rowgray} 83.3/81.7/80.5 & \cellcolor{rowgray} 79.3/72.5/64.8 & \cellcolor{rowgray} \textbf{94.1}/\textbf{87.6}/\textbf{83.3} & \cellcolor{rowgray} 83.8/76.2/69.5 & 89.6/82.5/76.7 \\
mint             & 65.8/63.1/64.8 & 67.0/66.6/64.6 & 52.4/50.3/63.7 & 58.4/55.8/63.7 & 76.7/76.7/\textbf{76.0} & 70.1/70.8/65.5 & 66.0/77.7/75.0 & 73.1/72.0/67.7 & \textbf{83.0}/\textbf{83.3}/74.7 \\
\cellcolor{rowgray} mounts & \cellcolor{rowgray} 88.6/79.9/74.8 & \cellcolor{rowgray} 87.6/77.3/77.2 & \cellcolor{rowgray} 58.7/48.1/52.4 & \cellcolor{rowgray} 74.7/56.5/63.1 & \cellcolor{rowgray} 75.3/74.5/82.5 & \cellcolor{rowgray} 86.8/78.0/73.5 & \cellcolor{rowgray} \textbf{95.2}/\textbf{92.3}/\textbf{85.9} & \cellcolor{rowgray} 90.4/84.2/78.0 & 87.4/74.3/77.6 \\
pcb              & 79.5/85.8/79.7 & 81.0/88.2/79.1 & 54.5/66.0/75.5 & 82.0/88.7/79.6 & 86.0/85.1/85.4 & 89.1/93.7/84.0 & 92.2/95.7/87.3 & 92.0/95.3/87.0 & \textbf{94.4}/\textbf{96.6}/\textbf{89.6} \\
\cellcolor{rowgray} phone battery & \cellcolor{rowgray} 87.5/83.3/77.1 & \cellcolor{rowgray} 83.6/80.0/71.6 & \cellcolor{rowgray} 51.6/43.8/58.0 & \cellcolor{rowgray} 83.3/81.8/72.1 & \cellcolor{rowgray} 82.3/77.7/75.9 & \cellcolor{rowgray} 90.2/88.9/80.5 & \cellcolor{rowgray} 92.6/93.0/84.5 & \cellcolor{rowgray} 92.9/91.6/82.5 & \textbf{95.7}/\textbf{94.3}/\textbf{87.6} \\
plastic nut      & 80.3/68.0/64.4 & 80.0/69.2/63.7 & 59.2/40.3/51.8 & 83.1/75.4/66.5 & 71.9/58.2/65.6 & 87.1/80.7/70.7 & 84.2/67.5/62.2 & 88.3/81.8/74.7 & \textbf{93.4}/\textbf{89.1}/\textbf{81.0} \\
\cellcolor{rowgray} plastic plug & \cellcolor{rowgray} 81.9/74.3/68.8 & \cellcolor{rowgray} 81.4/75.9/67.6 & \cellcolor{rowgray} 48.2/38.4/54.6 & \cellcolor{rowgray} 71.7/63.1/60.0 & \cellcolor{rowgray} 88.7/89.2/\textbf{90.9} & \cellcolor{rowgray} 85.7/82.2/72.6 & \cellcolor{rowgray} \textbf{94.1}/\textbf{93.2}/86.6 & \cellcolor{rowgray} 90.5/86.4/78.6 & 91.9/88.8/79.3 \\
porcelain doll   & 86.3/76.3/71.5 & 85.1/75.2/69.3 & 66.3/54.5/52.1 & 78.7/66.2/64.3 & 72.6/66.8/65.2 & 88.0/82.2/74.1 & 86.1/84.5/76.3 & 85.1/73.3/69.6 & \textbf{93.0}/\textbf{88.4}/\textbf{79.6} \\
\cellcolor{rowgray} regulator & \cellcolor{rowgray} 66.9/48.8/47.7 & \cellcolor{rowgray} 56.9/41.5/44.5 & \cellcolor{rowgray} 50.5/29.0/43.9 & \cellcolor{rowgray} 79.2/63.5/56.9 & \cellcolor{rowgray} 72.1/71.4/\textbf{78.2} & \cellcolor{rowgray} 69.7/58.7/50.4 & \cellcolor{rowgray} \textbf{89.5}/69.3/67.2 & \cellcolor{rowgray} 85.2/\textbf{78.9}/69.8 & 84.4/73.0/64.1 \\
rolled strip base & 97.5/98.7/94.7 & 98.7/99.3/96.5 & 59.0/75.7/79.8 & 96.5/98.2/93.0 & 68.4/55.9/56.8 & 98.0/99.0/95.0 & \textbf{99.8}/\textbf{99.9}/\textbf{98.9} & 99.2/99.6/97.1 & 99.4/99.7/97.8 \\
\cellcolor{rowgray} sim card set & \cellcolor{rowgray} 91.6/91.8/84.8 & \cellcolor{rowgray} 89.7/90.3/83.2 & \cellcolor{rowgray} 63.1/69.7/70.8 & \cellcolor{rowgray} 95.5/96.2/89.2 & \cellcolor{rowgray} 72.6/53.7/61.5 & \cellcolor{rowgray} 94.4/95.1/87.2 & 95.9/97.3/91.6 & 95.8/96.3/88.8 & \textbf{97.8}/\textbf{98.1}/\textbf{93.3} \\
switch           & 84.3/87.2/77.9 & 85.5/88.6/78.4 & 62.2/66.8/68.6 & 90.1/92.8/83.1 & 73.4/49.4/61.2 & 91.7/94.0/85.4 & 94.8/96.9/91.5 & \textbf{97.8}/98.1/\textbf{93.3} & 97.7/\textbf{98.2}/93.1 \\
\cellcolor{rowgray} tape & \cellcolor{rowgray} 96.0/95.1/87.6 & \cellcolor{rowgray} 97.2/96.2/89.4 & \cellcolor{rowgray} 49.9/41.1/54.5 & \cellcolor{rowgray} 94.5/93.4/85.9 & \cellcolor{rowgray} 73.9/57.8/66.1 & \cellcolor{rowgray} 96.8/95.9/89.3 & 98.0/98.0/\textbf{93.7} & 96.9/95.0/88.8 & \textbf{98.4}/\textbf{97.5}/92.4 \\
terminalblock    & 89.4/89.7/83.1 & 87.5/89.1/81.0 & 59.8/64.7/68.8 & 83.1/86.2/76.6 & 62.1/36.4/47.8 & 96.1/96.8/90.0 & \textbf{98.4}/\textbf{99.0}/\textbf{96.1} & 96.7/97.4/91.1 & 97.1/97.7/91.3 \\
\cellcolor{rowgray} toothbrush & \cellcolor{rowgray} 82.0/83.8/77.2 & \cellcolor{rowgray} 78.4/80.1/75.6 & \cellcolor{rowgray} 65.9/70.0/70.1 & \cellcolor{rowgray} 83.7/85.3/79.0 & \cellcolor{rowgray} \textbf{91.2}/93.7/\textbf{90.9} & \cellcolor{rowgray} 85.1/86.2/80.3 & 90.7/\textbf{94.8}/86.9 & 90.4/91.9/83.0 & 87.4/88.8/82.0 \\
toy              & 69.4/74.2/75.9 & 68.4/75.1/74.8 & 57.8/64.4/73.4 & 70.3/74.8/75.4 & 66.2/57.3/59.8 & 83.0/87.5/79.6 & 89.5/93.2/\textbf{87.6} & 85.6/89.1/81.9 & \textbf{90.2}/\textbf{93.2}/84.9 \\
\cellcolor{rowgray} toy brick & \cellcolor{rowgray} 63.6/56.1/59.0 & \cellcolor{rowgray} 77.0/71.1/66.2 & \cellcolor{rowgray} 58.3/49.7/58.2 & \cellcolor{rowgray} 73.2/68.7/63.3 & \cellcolor{rowgray} 68.4/45.3/55.9 & \cellcolor{rowgray} 70.5/63.7/61.6 & \cellcolor{rowgray} \textbf{85.0}/\textbf{85.5}/\textbf{74.8} & \cellcolor{rowgray} 72.3/65.1/63.4 & 81.0/78.3/69.0 \\
transistor1      & 91.0/94.0/85.1 & 93.7/95.9/88.9 & 62.2/69.2/72.1 & 90.2/92.1/84.6 & 73.1/63.1/62.7 & 94.4/96.0/89.0 & 96.2/98.3/92.9 & \textbf{97.4}/\textbf{98.2}/\textbf{93.1} & 96.9/97.7/92.4 \\
\cellcolor{rowgray} u block & \cellcolor{rowgray} 89.5/85.0/74.2 & \cellcolor{rowgray} 88.8/84.2/75.5 & \cellcolor{rowgray} 62.4/48.4/51.8 & \cellcolor{rowgray} 80.1/73.9/64.3 & \cellcolor{rowgray} 75.2/68.4/67.9 & \cellcolor{rowgray} 89.7/85.7/\textbf{75.3} & 90.1/83.4/74.5 & 89.9/84.0/75.2 & \textbf{93.2}/\textbf{90.0}/80.6 \\
usb              & 84.9/84.3/75.1 & 78.7/79.4/69.1 & 57.0/55.3/62.9 & 87.8/88.0/78.3 & 58.9/37.4/45.7 & 92.0/92.2/84.5 & \textbf{95.5}/\textbf{94.1}/\textbf{90.2} & 92.0/91.6/83.3 & 94.6/94.0/86.5 \\
\cellcolor{rowgray} usb adaptor & \cellcolor{rowgray} 71.1/61.4/62.2 & \cellcolor{rowgray} 76.8/71.3/64.9 & \cellcolor{rowgray} 47.5/38.4/56.5 & \cellcolor{rowgray} 80.1/74.9/67.4 & \cellcolor{rowgray} 76.9/60.2/67.2 & \cellcolor{rowgray} 79.4/76.0/66.3 & 82.6/82.4/72.6 & 81.5/74.5/69.4 & \textbf{88.2}/\textbf{84.6}/\textbf{77.0} \\
vcpill           & 85.1/80.3/72.4 & 87.1/84.0/74.7 & 59.0/48.7/56.4 & 83.8/81.5/69.9 & 64.1/40.4/56.2 & 88.3/87.7/77.4 & 91.4/90.7/82.8 & 92.0/91.2/82.0 & \textbf{94.3}/\textbf{93.4}/\textbf{84.5} \\
\cellcolor{rowgray} wooden beads & \cellcolor{rowgray} 81.2/78.9/70.9 & \cellcolor{rowgray} 78.4/77.2/67.8 & \cellcolor{rowgray} 55.1/52.0/60.2 & \cellcolor{rowgray} 82.4/78.5/73.0 & \cellcolor{rowgray} 62.1/56.4/65.9 & \cellcolor{rowgray} 82.5/81.7/71.8 & 77.0/83.5/74.9 & 87.3/85.8/77.4 & \textbf{91.8}/\textbf{91.1}/\textbf{82.1} \\
woodstick        & 76.9/61.2/58.1 & 80.8/72.6/63.6 & 58.2/35.6/45.2 & 80.4/69.2/60.3 & 74.1/\textbf{66.0}/\textbf{62.1} & 80.4/69.0/63.4 & \textbf{92.3}/65.9/60.5 & 84.0/73.3/65.6 & 84.8/\textbf{74.3}/67.4 \\
\cellcolor{rowgray} zipper & \cellcolor{rowgray} 95.3/97.2/91.2 & \cellcolor{rowgray} 98.2/98.9/95.3 & \cellcolor{rowgray} 77.2/86.7/77.6 & \cellcolor{rowgray} 96.9/98.1/93.5 & \cellcolor{rowgray} 86.0/87.0/84.0 & 99.2/99.6/96.9 & \textbf{99.8}/\textbf{99.9}/\textbf{99.0} & 99.1/99.5/96.5 & 99.1/99.5/95.9 \\
\midrule
Mean             & 82.4/79.0/73.9 & 83.0/80.9/74.3 & 57.2/53.4/61.5 & 82.3/79.2/73.2 & 75.6/66.4/69.9 & 86.3/84.6/77.0 & 90.1/88.6/82.8 & 89.3/86.8/80.2 & \textbf{92.0}/\textbf{89.9}/\textbf{83.0} \\
\bottomrule
\end{tabular}
}
\end{table*}

\begin{table*}[p]
\centering
\caption{Per-class performance on \textbf{Real-IAD} dataset for multi-class anomaly localization with AUROC/AP/$F_{1}$-max/AUPRO metrics.}
\label{tab:real_iad_p_per_class}
\resizebox{\textwidth}{!}{
\setlength{\tabcolsep}{3pt}
\setlength{\aboverulesep}{0pt}
\setlength{\belowrulesep}{0pt}
\begin{tabular}{ccccccccc>{\columncolor{oursblue}}c}
\toprule
Method $\rightarrow$ & RD4AD & UniAD & SimpleNet & DeSTSeg & DiAD & MambaAD & OmiAD & Dinomaly & RAD \\
\cmidrule(lr){1-1}
Category $\downarrow$ & CVPR'22 & NIPS'22 & CVPR'23 & CVPR'23 & AAAI'24 & NIPS'24 & ICML'25 & CVPR'25 & Ours \\
\midrule
audiojack        & 96.6/12.8/22.1/79.6 & 97.6/20.0/31.0/83.7 & 74.4/ 0.9/ 4.8/38.0 & 95.5/25.4/31.9/52.6 & 91.6/ 1.0/ 3.9/63.3 & 97.7/21.6/29.5/83.9 & 99.0/47.1/51.4/91.5 & 98.7/48.1/54.5/91.7 & \textbf{99.2}/\textbf{57.6}/\textbf{58.8}/\textbf{94.7} \\
\cellcolor{rowgray} bottle cap & \cellcolor{rowgray} 99.5/18.9/29.9/95.7 & \cellcolor{rowgray} 99.5/19.4/29.6/96.0 & \cellcolor{rowgray} 85.3/ 2.3/ 5.7/45.1 & \cellcolor{rowgray} 94.5/25.3/31.1/25.3 & \cellcolor{rowgray} 94.6/ 4.9/11.4/73.0 & \cellcolor{rowgray} \textbf{99.7}/30.6/34.6/97.2 & \cellcolor{rowgray} 99.4/23.4/29.0/95.2 & \cellcolor{rowgray} \textbf{99.7}/32.4/36.7/\textbf{98.1} & \textbf{99.7}/\textbf{44.5}/\textbf{45.5}/98.0 \\
button battery   & 97.6/33.8/37.8/86.5 & 96.7/28.5/34.4/77.5 & 75.9/ 3.2/ 6.6/40.5 & 98.3/63.9/\textbf{60.4}/36.9 & 84.1/ 1.4/ 5.3/66.9 & 98.1/46.7/49.5/86.2 & \textbf{99.2}/61.0/\textbf{60.4}/91.9 & 99.1/46.9/56.7/92.9 & \textbf{99.2}/\textbf{65.8}/60.0/\textbf{95.1} \\
\cellcolor{rowgray} end cap & \cellcolor{rowgray} 96.7/12.5/22.5/89.2 & \cellcolor{rowgray} 95.8/ 8.8/17.4/85.4 & \cellcolor{rowgray} 63.1/ 0.5/ 2.8/25.7 & \cellcolor{rowgray} 89.6/14.4/22.7/29.5 & \cellcolor{rowgray} 81.3/ 2.0/ 6.9/38.2 & \cellcolor{rowgray} 97.0/12.0/19.6/89.4 & \cellcolor{rowgray} 97.2/09.2/14.4/91.3 & \cellcolor{rowgray} 99.1/26.2/32.9/96.0 & \textbf{99.3}/\textbf{32.2}/\textbf{39.4}/\textbf{97.2} \\
eraser           & 99.5/30.8/36.7/96.0 & 99.3/24.4/30.9/94.1 & 80.6/ 2.7/ 7.1/42.8 & 95.8/52.7/53.9/46.7 & 91.1/ 7.7/15.4/67.5 & 99.2/30.2/38.3/93.7 & 99.3/39.5/44.2/93.3 & 99.5/39.6/43.3/96.4 & \textbf{99.8}/\textbf{56.1}/\textbf{55.7}/\textbf{98.8} \\
\cellcolor{rowgray} fire hood & \cellcolor{rowgray} 98.9/27.7/35.2/87.9 & \cellcolor{rowgray} 98.6/23.4/32.2/85.3 & \cellcolor{rowgray} 70.5/ 0.3/ 2.2/25.3 & \cellcolor{rowgray} 97.3/27.1/35.3/34.7 & \cellcolor{rowgray} 91.8/ 3.2/ 9.2/66.7 & \cellcolor{rowgray} 98.7/25.1/31.3/86.3 & \cellcolor{rowgray} 98.9/43.3/48.6/94.6 & \cellcolor{rowgray} 99.3/38.4/42.7/93.0 & \textbf{99.5}/\textbf{50.6}/\textbf{50.6}/\textbf{96.6} \\
mint             & 95.0/11.7/23.0/72.3 & 94.4/ 7.7/18.1/62.3 & 79.9/ 0.9/ 3.6/43.3 & 84.1/10.3/22.4/ 9.9 & 91.1/ 5.7/11.6/64.2 & 96.5/15.9/27.0/72.6 & 96.6/29.6/38.7/67.4 & 96.9/22.0/32.5/77.6 & \textbf{98.1}/\textbf{35.9}/\textbf{40.9}/\textbf{90.6} \\
\cellcolor{rowgray} mounts & \cellcolor{rowgray} 99.3/30.6/37.1/94.9 & \cellcolor{rowgray} 99.4/28.0/32.8/95.2 & \cellcolor{rowgray} 80.5/ 2.2/ 6.8/46.1 & \cellcolor{rowgray} 94.2/30.0/41.3/43.3 & \cellcolor{rowgray} 84.3/ 0.4/ 1.1/48.8 & \cellcolor{rowgray} 99.2/31.4/35.4/93.5 & \cellcolor{rowgray} \textbf{99.7}/39.2/40.0/\textbf{98.6} & \cellcolor{rowgray} 99.4/39.9/44.3/95.6 & 99.4/\textbf{45.8}/\textbf{46.1}/96.6 \\
pcb              & 97.5/15.8/24.3/88.3 & 97.0/18.5/28.1/81.6 & 78.0/ 1.4/ 4.3/41.3 & 97.2/37.1/40.4/48.8 & 92.0/ 3.7/ 7.4/66.5 & 99.2/46.3/50.4/93.1 & 99.0/48.8/51.2/92.1 & 99.3/55.0/56.3/95.7 & \textbf{100.}/\textbf{68.7}/\textbf{64.8}/\textbf{96.9} \\
\cellcolor{rowgray} phone battery & \cellcolor{rowgray} 77.3/22.6/31.7/94.5 & \cellcolor{rowgray} 85.5/11.2/21.6/88.5 & \cellcolor{rowgray} 43.4/ 0.1/ 0.9/11.8 & \cellcolor{rowgray} 79.5/25.6/33.8/39.5 & \cellcolor{rowgray} 96.8/ 5.3/11.4/85.4 & \cellcolor{rowgray} 99.4/36.3/41.3/95.3 & \cellcolor{rowgray} 99.2/41.1/44.5/94.0 & \cellcolor{rowgray} 99.7/51.6/54.2/96.8 & \textbf{99.8}/\textbf{60.8}/\textbf{57.4}/\textbf{98.6} \\
plastic nut      & 98.8/21.1/29.6/91.0 & 98.4/20.6/27.1/88.9 & 77.4/ 0.6/ 3.6/41.5 & 96.5/44.8/45.7/38.4 & 81.1/ 0.4/ 3.4/38.6 & 99.4/33.1/37.3/96.1 & 98.3/27.2/31.1/90.1 & 99.7/41.0/45.0/97.4 & \textbf{99.8}/\textbf{54.4}/\textbf{53.0}/\textbf{98.7} \\
\cellcolor{rowgray} plastic plug & \cellcolor{rowgray} 99.1/20.5/28.4/94.9 & \cellcolor{rowgray} 98.6/17.4/26.1/90.3 & \cellcolor{rowgray} 78.6/ 0.7/ 1.9/38.8 & \cellcolor{rowgray} 91.9/20.1/27.3/21.0 & \cellcolor{rowgray} 92.9/ 8.7/15.0/66.1 & \cellcolor{rowgray} 99.0/24.2/31.7/91.5 & \cellcolor{rowgray} \textbf{99.5}/37.1/41.3/97.1 & \cellcolor{rowgray} 99.4/31.7/37.2/95.6 & \textbf{99.5}/\textbf{42.0}/\textbf{43.4}/\textbf{98.0} \\
porcelain doll   & 99.2/24.8/34.6/95.7 & 98.7/14.1/24.5/93.2 & 81.8/ 2.0/ 6.4/47.0 & 93.1/35.9/40.3/24.8 & 93.1/ 1.4/ 4.8/70.4 & 99.2/31.3/36.6/95.4 & 98.8/18.3/26.3/93.8 & 99.3/27.9/33.9/96.0 & \textbf{100.}/\textbf{46.9}/\textbf{48.9}/\textbf{98.5} \\
\cellcolor{rowgray} regulator & \cellcolor{rowgray} 98.0/ 7.8/16.1/88.6 & \cellcolor{rowgray} 95.5/ 9.1/17.4/76.1 & \cellcolor{rowgray} 76.6/ 0.1/ 0.6/38.1 & \cellcolor{rowgray} 88.8/18.9/23.6/17.5 & \cellcolor{rowgray} 84.2/ 0.4/ 1.5/44.4 & \cellcolor{rowgray} 97.6/20.6/29.8/87.0 & \cellcolor{rowgray} \textbf{99.7}/37.4/42.2/\textbf{98.6} & \cellcolor{rowgray} 99.3/42.2/48.9/95.5 & 99.4/\textbf{49.0}/\textbf{52.7}/96.0 \\
rolled strip base& 99.7/31.4/39.9/98.4 & 99.6/20.7/32.2/97.8 & 80.5/ 1.7/ 5.1/52.1 & 99.2/48.7/50.1/55.5 & 87.7/ 0.6/ 3.2/63.4 & 99.7/37.4/42.5/98.8 & 99.7/32.4/42.5/98.9 & 99.7/41.6/45.5/98.5 & \textbf{99.9}/\textbf{59.5}/\textbf{61.7}/\textbf{99.4} \\
\cellcolor{rowgray} sim card set & \cellcolor{rowgray} 98.5/40.2/44.2/89.5 & \cellcolor{rowgray} 97.9/31.6/39.8/85.0 & \cellcolor{rowgray} 71.0/ 6.8/14.3/30.8 & \cellcolor{rowgray} 99.1/65.5/62.1/73.9 & \cellcolor{rowgray} 89.9/ 1.7/ 5.8/60.4 & \cellcolor{rowgray} 98.8/51.1/50.6/89.4 & \cellcolor{rowgray} 99.3/48.9/50.1/95.4 & \cellcolor{rowgray} 99.0/52.1/52.9/90.9 & \textbf{100.}/\textbf{70.7}/\textbf{65.0}/\textbf{97.6} \\
switch           & 94.4/18.9/26.6/90.9 & 98.1/33.8/40.6/90.7 & 71.7/ 3.7/ 9.3/44.2 & 97.4/57.6/55.6/44.7 & 90.5/ 1.4/ 5.3/64.2 & 98.2/39.9/45.4/92.9 & \textbf{99.5}/\textbf{63.6}/63.4/95.8 & 96.7/62.3/\textbf{63.6}/\textbf{95.9} & 99.4/57.6/58.2/97.7 \\
\cellcolor{rowgray} tape & \cellcolor{rowgray} 99.7/42.4/47.8/98.4 & \cellcolor{rowgray} 99.7/29.2/36.9/97.5 & \cellcolor{rowgray} 77.5/ 1.2/ 3.9/41.4 & \cellcolor{rowgray} 99.0/61.7/57.6/48.2 & \cellcolor{rowgray} 81.7/ 0.4/ 2.7/47.3 & \cellcolor{rowgray} \textbf{99.8}/47.1/48.2/98.0 & \cellcolor{rowgray} 99.7/29.8/36.4/\textbf{98.8} & \cellcolor{rowgray} \textbf{99.8}/54.0/55.8/\textbf{98.8} & \textbf{99.8}/\textbf{66.1}/\textbf{61.7}/\textbf{98.8} \\
terminalblock    & 99.5/27.4/35.8/97.6 & 99.2/23.1/30.5/94.4 & 87.0/ 0.8/ 3.6/54.8 & 96.6/40.6/44.1/34.8 & 75.5/ 0.1/ 1.1/38.5 & 99.8/35.3/39.7/98.2 & 99.8/48.3/51.0/98.9 & 99.8/48.0/50.7/98.8 & \textbf{99.9}/\textbf{61.2}/\textbf{59.6}/\textbf{99.2} \\
\cellcolor{rowgray} toothbrush & \cellcolor{rowgray} 96.9/26.1/34.2/88.7 & \cellcolor{rowgray} 95.7/16.4/25.3/84.3 & \cellcolor{rowgray} 84.7/ 7.2/14.8/52.6 & \cellcolor{rowgray} 94.3/30.0/37.3/42.8 & \cellcolor{rowgray} 82.0/ 1.9/ 6.6/54.5 & \cellcolor{rowgray} 97.5/27.8/36.7/91.4 & \cellcolor{rowgray} \textbf{98.8}/\textbf{39.8}/\textbf{47.9}/\textbf{93.8} & \cellcolor{rowgray} 96.9/38.3/43.9/90.4 & 97.0/39.2/44.4/90.2 \\
toy              & 95.2/ 5.1/12.8/82.3 & 93.4/ 4.6/12.4/70.5 & 67.7/ 0.1/ 0.4/25.0 & 86.3/ 8.1/15.9/16.4 & 82.1/ 1.1/ 4.2/50.3 & 96.0/16.4/25.8/86.3 & \textbf{97.8}/19.8/25.5/90.8 & 94.9/22.5/32.1/\textbf{91.0} & 92.9/\textbf{31.9}/\textbf{39.0}/90.0 \\
\cellcolor{rowgray} toy brick & \cellcolor{rowgray} 96.4/16.0/24.6/75.3 & \cellcolor{rowgray} 97.4/17.1/27.6/81.3 & \cellcolor{rowgray} 86.5/ 5.2/11.1/56.3 & \cellcolor{rowgray} 94.7/24.6/30.8/45.5 & \cellcolor{rowgray} 93.5/ 3.1/ 8.1/66.4 & \cellcolor{rowgray} 96.6/18.0/25.8/74.7 & \cellcolor{rowgray} \textbf{98.6}/44.3/48.7/\textbf{89.6} & \cellcolor{rowgray} 96.8/27.9/34.0/76.6 & 97.3/\textbf{46.5}/\textbf{50.7}/87.7 \\
transistor1      & 99.1/29.6/35.5/95.1 & 98.9/25.6/33.2/94.3 & 71.7/ 5.1/11.3/35.3 & 97.3/43.8/44.5/45.4 & 88.6/ 7.2/15.3/58.1 & 99.4/39.4/40.0/96.5 & 99.1/40.2/43.6/95.9 & \textbf{99.6}/53.5/53.3/\textbf{97.8} & 99.5/\textbf{58.0}/\textbf{55.6}/97.3 \\
\cellcolor{rowgray} u block & \cellcolor{rowgray} 99.6/40.5/45.2/96.9 & \cellcolor{rowgray} 99.3/22.3/29.6/94.3 & \cellcolor{rowgray} 76.2/ 4.8/12.2/34.0 & \cellcolor{rowgray} 96.9/57.1/55.7/38.5 & \cellcolor{rowgray} 88.8/ 1.6/ 5.4/54.2 & \cellcolor{rowgray} 99.5/37.8/46.1/95.4 & \cellcolor{rowgray} 99.5/24.2/35.6/97.8 & \cellcolor{rowgray} 99.5/41.8/45.6/95.7 & \textbf{99.7}/\textbf{62.6}/\textbf{63.0}/\textbf{97.9} \\
usb              & 98.1/26.4/35.2/91.0 & 97.9/20.6/31.7/85.3 & 81.1/ 1.5/ 4.9/52.4 & 98.4/42.2/47.7/57.1 & 78.0/ 1.0/ 3.1/28.0 & 99.2/39.1/44.4/95.2 & \textbf{99.6}/43.4/48.3/97.0 & 99.2/45.0/48.7/97.5 & 99.5/\textbf{52.0}/\textbf{53.2}/\textbf{97.9} \\
\cellcolor{rowgray} usb adaptor & \cellcolor{rowgray} 94.5/ 9.8/17.9/73.1 & \cellcolor{rowgray} 96.6/10.5/19.0/78.4 & \cellcolor{rowgray} 67.9/ 0.2/ 1.3/28.9 & \cellcolor{rowgray} 94.9/25.5/34.9/36.4 & \cellcolor{rowgray} 94.0/ 2.3/ 6.6/75.5 & \cellcolor{rowgray} 97.3/15.3/22.6/82.5 & \cellcolor{rowgray} 96.8/18.1/27.3/84.2 & \cellcolor{rowgray} 98.7/23.7/32.7/91.0 & \textbf{100.}/\textbf{42.8}/\textbf{44.8}/\textbf{96.0} \\
vcpill           & 98.3/43.1/48.6/88.7 & \textbf{99.1}/40.7/43.0/91.3 & 68.2/ 1.1/ 3.3/22.0 & 97.1/64.7/62.3/42.3 & 90.2/ 1.3/ 5.2/60.8 & 98.7/50.2/54.5/89.3 & 99.0/58.4/61.2/92.6 & \textbf{99.1}/66.4/66.7/93.7 & \textbf{99.1}/\textbf{73.1}/\textbf{70.3}/\textbf{94.1} \\
\cellcolor{rowgray} wooden beads & \cellcolor{rowgray} 98.0/27.1/34.7/85.7 & \cellcolor{rowgray} 97.6/16.5/23.6/84.6 & \cellcolor{rowgray} 68.1/ 2.4/ 6.0/28.3 & \cellcolor{rowgray} 94.7/38.9/42.9/39.4 & \cellcolor{rowgray} 85.0/ 1.1/ 4.7/45.6 & \cellcolor{rowgray} 98.0/32.6/39.8/84.5 & \cellcolor{rowgray} 97.3/26.2/31.4/83.1 & \cellcolor{rowgray} \textbf{99.1}/45.8/50.1/90.5 & \textbf{99.1}/\textbf{59.0}/\textbf{58.3}/\textbf{95.1} \\
woodstick        & 97.8/30.7/38.4/85.0 & 94.0/36.2/44.3/77.2 & 76.1/ 1.4/ 6.0/32.0 & 97.9/60.3/60.0/51.0 & 90.9/ 2.6/ 8.0/60.7 & 97.7/40.1/44.9/82.7 & 98.4/48.5/51.9/\textbf{93.1} & 99.0/50.9/52.1/90.5 & \textbf{98.9}/\textbf{62.7}/\textbf{61.0}/91.7 \\
\cellcolor{rowgray} zipper & \cellcolor{rowgray} 99.1/44.7/50.2/96.3 & \cellcolor{rowgray} 98.4/32.5/36.1/95.1 & \cellcolor{rowgray} 89.9/23.3/31.2/55.5 & \cellcolor{rowgray} 98.2/35.3/39.0/78.5 & \cellcolor{rowgray} 90.2/12.5/18.8/53.5 & \cellcolor{rowgray} 99.3/58.2/61.3/97.6 & \cellcolor{rowgray} 99.0/43.8/49.7/97.4 & \cellcolor{rowgray} \textbf{99.3}/\textbf{67.2}/\textbf{66.5}/\textbf{97.8} & 99.1/58.4/61.0/97.1 \\
\midrule
Mean             & 97.3/25.0/32.7/89.6 & 97.3/21.1/29.2/86.7 & 75.7/ 2.8/ 6.5/39.0 & 94.6/37.9/41.7/40.6 & 88.0/ 2.9/ 7.1/58.1 & 98.5/33.0/38.7/90.5 & 98.9/37.7/42.6/93.1 & 98.8/42.8/47.1/93.9 & \textbf{99.1}/\textbf{53.7}/\textbf{54.2}/\textbf{96.1} \\
\bottomrule
\end{tabular}
}
\end{table*}

\section{Qualitative Visualizations}
\label{sec:visualization}
We visualize the output anomaly maps of RAD on MVTec-AD~\cite{bergmann2019mvtec}, VisA~\cite{zou2022spot}, and Real-IAD~\cite{wang2024real}, as shown in Figure~\ref{fig:Figure_A1}, Figure~\ref{fig:Figure_A2}, and Figure~\ref{fig:Figure_A3}. It is noted that all visualized samples are randomly chosen without artificial selection.

\begin{figure}[h]
    \centering
    \includegraphics[width=1\linewidth]{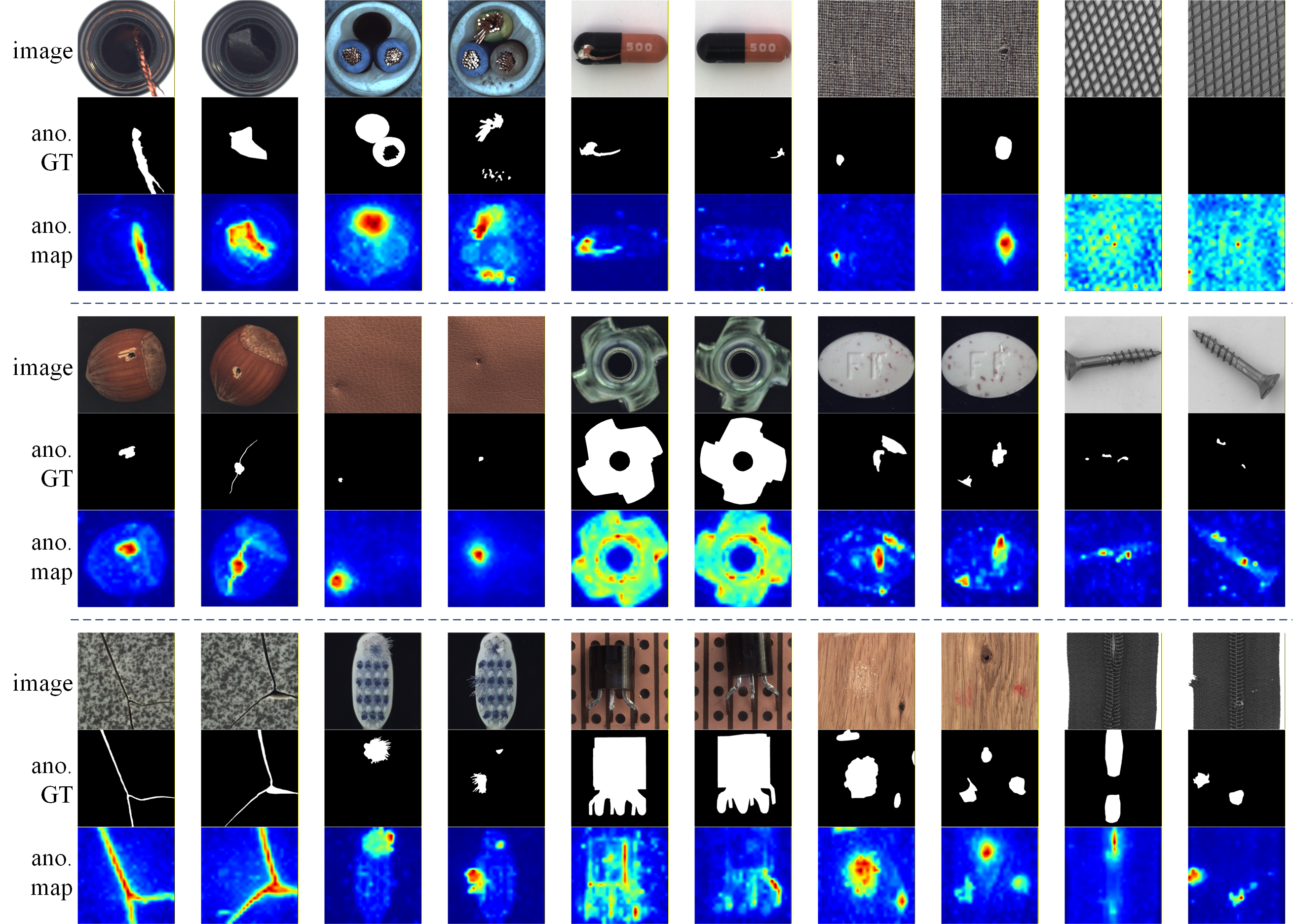}
    \caption{Anomaly maps visualization on MVTec-AD. All samples are randomly chosen.}
    \label{fig:Figure_A1}
\end{figure}
\clearpage

\begin{figure}[p]
    \centering
    \includegraphics[width=1\linewidth]{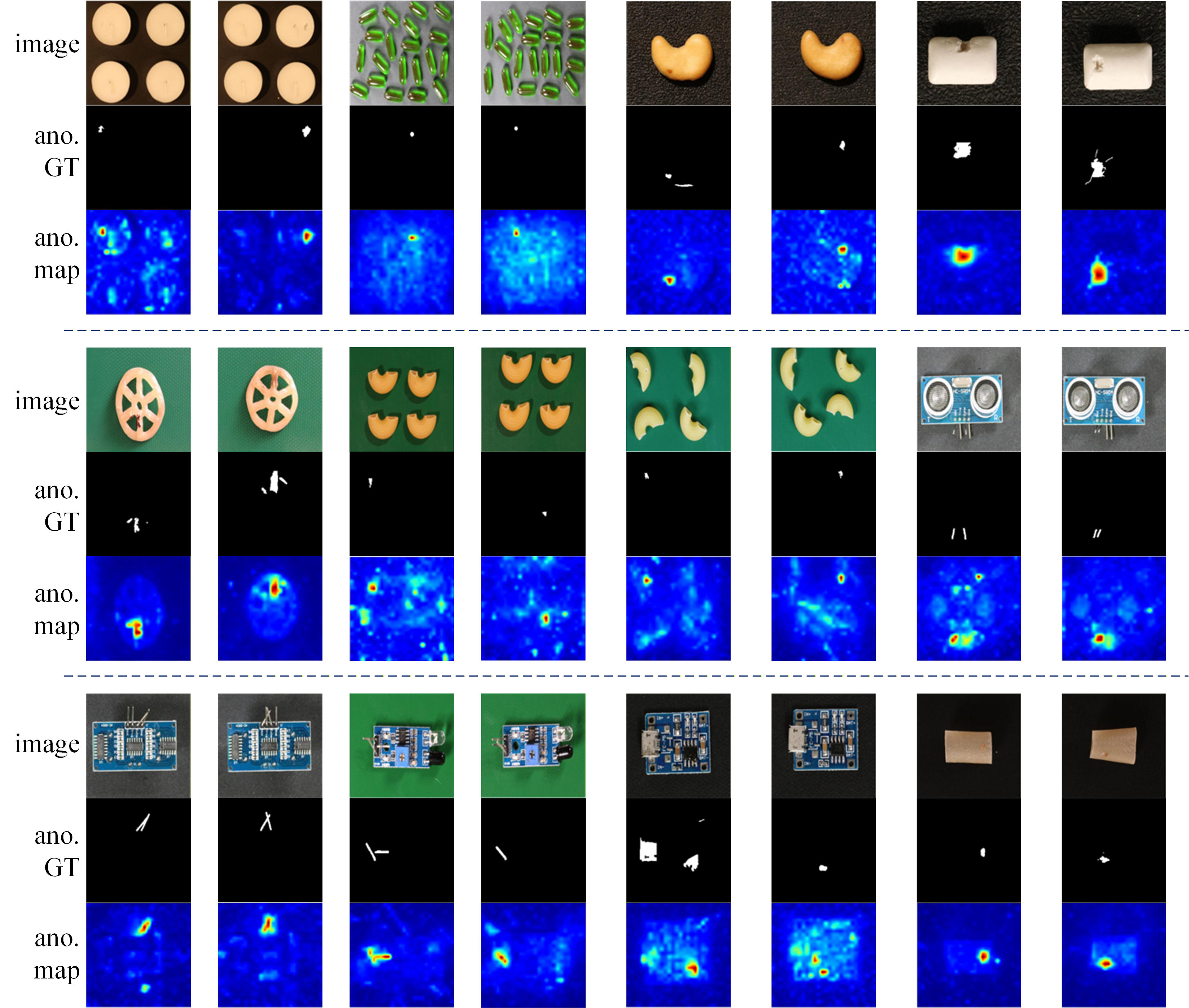}
    \caption{Anomaly maps visualization on VisA. All samples are randomly chosen.}
    \label{fig:Figure_A2}
\end{figure}
\clearpage

\begin{figure}[p]
    \centering
    \includegraphics[width=1\linewidth]{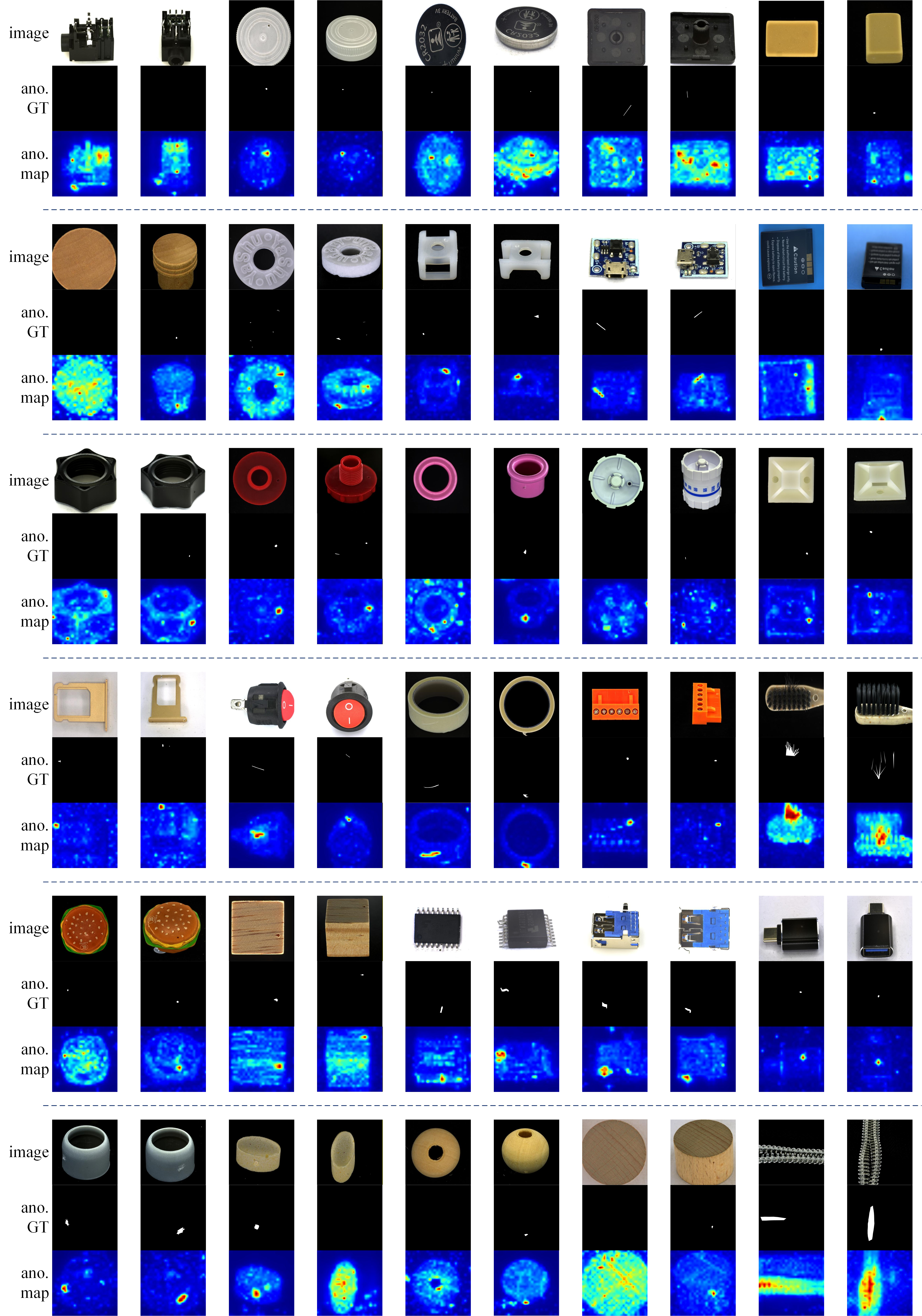}
    \caption{Anomaly maps visualization on Real-IAD. All samples are randomly chosen.}
    \label{fig:Figure_A3}
\end{figure}
\clearpage



\end{document}